\documentclass[11pt]{article}
\usepackage{preamble}
\usepackage{bm}

\newcommand{\RR}{\mathbb{R}}

\newcommand{\ZZ}{\mathbb{Z}}

\newcommand{\EE}{\mathbb{E}}

\newcommand{\tr}{\mathrm{tr}}
\newcommand{\Cov}{\mathsf{Cov}}

\newcommand{\sech}{\mathrm{sech}}

\newcommand{\uMI}{\mathsf{I}_{\sf unif}}
\newcommand{\gMI}{\mathsf{I}_{\sf gauss}}
\newcommand{\tcritu}{t^{\star}_{\sf unif}}
\newcommand{\tcritg}{t^{\star}_{\sf gauss}}

\newcommand{\TV}{\mathsf{TV}}
\newcommand{\Unif}{\mathrm{Unif}}
\newcommand{\err}{\mathrm{err}}
\newcommand{\atanh}{\mathrm{atanh}}

\newcommand{\Ber}{\mathrm{Ber}}

\newcommand{\post}{p^{\sf post}}
\newcommand{\gibbs}{p^{\sf gibbs}}

\newcommand{\law}{\mathrm{law}}

\usepackage{braket}

\usepackage[T1]{fontenc}

\usepackage{times}
\AtBeginDocument{\DeclareFontShape{T1}{ptm}{m}{scit}{<->ssub*ptm/m/sc}{}}

\usepackage{setspace}
\usepackage[framemethod=TikZ]{mdframed}
\usepackage{multicol}
\usepackage[utf8]{inputenc}

\usepackage[scr=rsfs]{mathalpha}

\usepackage{upgreek}
\usepackage{thm-restate}

\usepackage[linesnumbered,ruled,vlined]{algorithm2e}
\SetKwInput{KwInput}{Input}
\SetKwInput{KwOutput}{Output}
\SetKwInput{Promise}{Promise}
\SetKwInput{KwRequire}{Require}
\SetKwInput{KwEnsure}{Ensure}
\SetKwProg{Fn}{Function}{:}{} 

\title{Parallelism, critical windows, and separations among diffusion language models}
\author{
Sitan Chen
\thanks{SEAS, Harvard University. Email: \href{mailto:sitan@seas.harvard.edu}{sitan@seas.harvard.edu}. This work was supported by the Harvard Dean's Competitive Fund for Promising Scholarship and was completed in part during a visit to the Simons Institute for the Theory of Computing.}
\qquad\qquad
Liye Wang
\thanks{Tsinghua University. Email: \href{mailto:ly-wang23@mails.tsinghua.edu.cn}{ly-wang23@mails.tsinghua.edu.cn}.}
}

\hypersetup{
  pdftitle={Parallelism, critical windows, and separations among diffusion language models},
  pdfauthor={Sitan Chen, Liye Wang}
}

\newcommand{\sub}[3]{#1^{#2\leftarrow#3}}
\newcommand{\yia}{\sub{y}{i}{a}}
\newcommand{\unifscore}[3]{s_{#1}(#2)[#3]}
\newcommand{\widehatunifscore}[3]{\widehat{s}_{#1}(#2)[#3]}

\newcommand{\prodrevu}{\overline R^{\sf unif}}
\newcommand{\approxrevu}{\widehat R^{\sf unif}}
\newcommand{\prodrevg}{\overline R^{\sf G}}
\newcommand{\approxrevg}{\widehat R^{\sf G}}
\newcommand{\TC}{\mathsf{TC}}
\newcommand{\DTC}{\mathsf{DTC}}
\newcommand{\match}[2]{\mathsf{match}_{#1}(#2)}

\date{September 13, 2026}

\begin{document}

\maketitle

\begin{abstract}
    A popular selling point of diffusion large language models (dLLMs) is their capacity for \emph{parallelism}: the ability to generate sequences of text far more efficiently than autoregressive models, which require one forward pass per token. Yet among the many competing paradigms for dLLMs, from masked to uniform to Gaussian diffusion, principled understanding of how these different proposals compare in parallelism remains limited. In this work, we initiate a fine-grained comparison of the capacity for parallelism among these three leading approaches and prove the following:
    \begin{itemize}[leftmargin=*]
        \item Uniform and Gaussian diffusion can sample in a number of forward passes which scales with the \emph{dual total correlation} of the underlying distribution, a measure of intrinsic complexity which can be much smaller than the context length. Previously, it was only known how to achieve this using masked diffusion~\cite{chen2025optimal,lavenant2025error}.
        \item For a certain family of random empirical measures, we show that $\widetilde{\Theta}(\sqrt{d})$ forward passes are necessary and sufficient to sample using uniform or Gaussian diffusion, yet there exist approximate score oracles for which $\widetilde{\Omega}(d)$ forward passes are needed for masked diffusion. This establishes the first provable separation in parallelism between the three prevailing dLLM paradigms.
    \end{itemize}
    Contrary to popular intuition that masked diffusions are harder to parallelize because they must commit to token values, the latter separation instead comes from the fact that the critical windows in masked diffusion sampling are asymptotically narrower than those in uniform and Gaussian diffusion sampling.
\end{abstract}

\newpage

\tableofcontents

\newpage

\section{Introduction}

Diffusion large language models (dLLMs) have recently emerged as a powerful alternative to autoregressive large language models (LLMs) for generative modeling over discrete domains~\cite{austin2021structured,lou2023discrete,sahoo2024simple,nie2025large}. Despite the moniker, diffusion language modeling is not one, but a multitude of competing frameworks (Figure~\ref{fig:paradigms}) all built around the same guiding principle of denoising: given a corruption process that converts data into noise, learn how to undo it in order to transform fresh noise into fresh samples. LLMs are themselves a special case of this paradigm, where the corruption process is right-to-left erasure, and a central question in the theory and practice of dLLMs is whether there is a better choice of corruption process.

One of the most widely touted selling points of moving beyond right-to-left erasure is that it unlocks the ability to perform \emph{few-step generation}: Whereas LLMs must decode one token at a time, dLLMs based on alternative corruption processes can in principle generate a sequence of length $d$ in $o(d)$ forward passes.
The cost of doing so, however, remains poorly understood. It is well known that there is \emph{some} statistical price to pay, as such decoding strategies fundamentally incur \emph{discretization error} coming from approximating certain posterior distributions using only their marginals. But the extent to which this price can be rendered negligible depends heavily on how the sampling algorithm at inference time is tuned.

Various empirical works have suggested that this price also depends heavily on the choice of corruption process itself, arguing that few-step generation is easier to achieve for some frameworks than for others~\cite{sahoo2025duality,lee2026flow,hu2026elf}. For example, it is commonly claimed that corruption processes involving independent erasures, which give rise to the popular paradigm of \emph{masked diffusion models}, are less conducive to few-step generation. The informal reasoning is that by default, masked diffusion models have to \emph{commit} to token values in discrete steps of the generation process, meaning that mistakes arising from committing multiple tokens in parallel cannot subsequently be corrected~\cite{schiff2024simple,sahoo2025duality,sahoo2026scaling}. Other popular choices of corruption process, e.g., involving random re-assignments of token values (\emph{uniform diffusion models}~\cite{lou2023discrete,sahoo2025duality,diffusiongemma2026}) or Brownian motion in a latent space (\emph{Gaussian diffusion models}~\cite{li2022diffusionlm,dieleman2022continuous,roos2026categorical,lee2026flow,hu2026elf,dieleman2026continuous}), are claimed not to suffer from this as they can progressively revise their output over the course of sampling (see Figure~\ref{fig:paradigms} for a depiction of these different dLLM paradigms).

\begin{figure}[h!]
  \centering
  \input{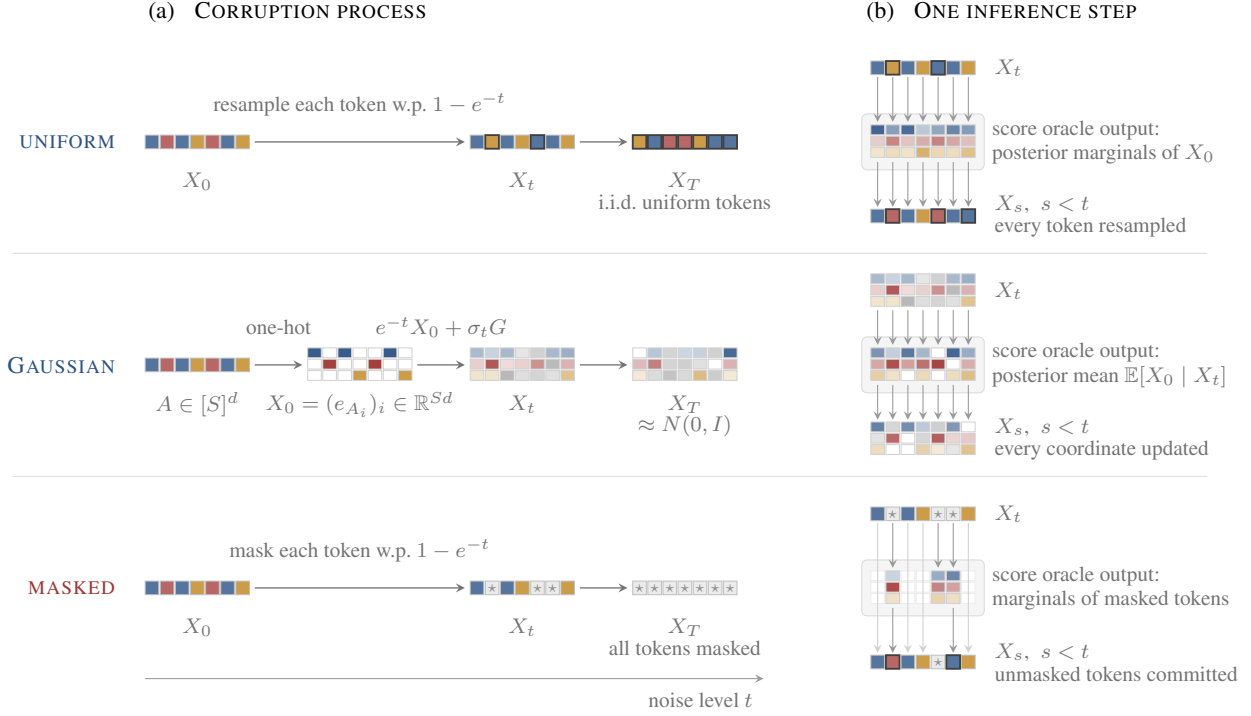}
  \caption{The three dLLM paradigms considered in this work.}
  \label{fig:paradigms}
\end{figure}

Yet these intuitions are largely heuristic, and the foundations of few-step generation under any of these paradigms remain nascent. A flurry of recent theoretical works~\cite{li2026breaking,chen2025optimal,dmitriev2026efficient,wainwright2026data,lavenant2025error,zhao2026adaptation} have suggested that with the right inference schedule, masked diffusion models can achieve few-step generation, where the number of steps scales with an intrinsic complexity measure associated to the data distribution. In contrast, no such results were known for other frameworks like uniform and Gaussian diffusion.  In this work, we thus ask:
\begin{center}
    \emph{How does the choice of corruption process affect the few-step generation capabilities of dLLMs?}\vspace{-0.5em}
\end{center}

\paragraph{Framework: score oracle queries.} When comparing different dLLM frameworks, one basic figure of merit is the number of forward passes of the model needed to generate an accurate sample. More formally, a single forward pass of the model provides us with the following information. Given a clean sample $X_0$, the corruption process degrades it into a noisy sample $X_t$, and a single forward pass provides an approximation to the coordinatewise posterior marginals $\mathrm{law}((X_0)_i \mid X_t = z)$ simultaneously for \emph{all coordinates $i$}, for any noise level $t$ and conditioning $z$ of one's choice.
As these marginals are equivalent, up to affine transformation, to the so-called \emph{annealed score functions} of the data distribution, in this work we refer to a single forward pass of the model as a query to the \emph{(approximate) score oracle}.
We define these oracles more formally in Sections~\ref{sec:uniform-diffusion}--\ref{sec:masked-diffusion}. When proving \emph{positive results} about few-step generation, we will analyze specific algorithms that leverage these oracles; when proving \emph{no-go results}, we will rule out \emph{all algorithms} that leverage these oracles.

\subsection{Result 1: Scaling with intrinsic complexity} 

First, we prove that both uniform and Gaussian diffusion can match the scaling established in prior work for masked diffusion. More specifically, we prove the following. Given a categorical distribution $q$ over sequences $X = (X_1,\ldots,X_d)$ in a product space $\Sigma^d$, its \emph{dual total correlation} $\DTC(q)$ is, roughly speaking, the total entropy in the distribution which is not explained by the \emph{local} entropies $H(X_i \mid X_{-i})$. Prior work showed that, up to logarithmic factors, masked diffusion models can sample from $q$ in a number of score oracle queries scaling linearly in $\DTC(q)$, which can in general be much smaller than the dimension $d$.

We begin by showing that such a result is also possible for uniform diffusion.

\begin{theorem}[Informal, see Theorem~\ref{thm:uniform-dtc-sampler}]\label{thm:informal-uniform-dtc}
    Let $q$ be an arbitrary distribution over $\Sigma^d$. There is an algorithm that makes $\widetilde{O}(\DTC(q)/\epsilon)$ queries, where $\widetilde{O}(\cdot)$ hides logarithmic factors in $|\Sigma|$, $d$, and $\epsilon$, to an accurate estimate of the \emph{uniform} diffusion score for $q$ and outputs a sample from a distribution $\widehat{q}$ for which $\KL{q}{\widehat{q}}\le \epsilon$.
\end{theorem}

\noindent By a similar mechanism albeit with a more involved argument, we also show:

\begin{theorem}[Informal, see Theorem~\ref{thm:gaussian-dtc-sampler}]\label{thm:informal-gauss-dtc}
    There is a sampler with the same guarantee as in Theorem~\ref{thm:informal-uniform-dtc} which uses $\widetilde{O}(\DTC(q)/\epsilon)$ queries to an accurate estimate of the \emph{Gaussian} diffusion score for $q$.
\end{theorem}

\noindent The proof strategy for these results is a marked departure from the one for masked diffusion in~\cite{chen2025optimal}. Crucial to that work was an \emph{exact} characterization for the error incurred by the standard algorithm for masked diffusion sampling; with this characterization in hand, it was enough to substitute an appropriate choice of step size schedule into that expression to obtain the desired rate of $\widetilde{O}(\DTC(q)/\epsilon)$.

In contrast, for Theorems~\ref{thm:informal-uniform-dtc} and~\ref{thm:informal-gauss-dtc}, in lieu of an exact characterization, we prove that the error incurred in a single step of the sampler from ignoring conditional correlations can be controlled in terms of the increase in dual total correlation when going from one noise level to a lower noise level. This is proven using a certain \emph{reverse data processing inequality (DPI)} which shows that the distance between two distributions cannot contract too quickly if they are only slightly noised. With this estimate, we can then telescope the DTC increases together and conclude a query complexity that scales linearly in $\DTC(q)$. This argument is illustrated in Figure~\ref{fig:dtc-telescope}.

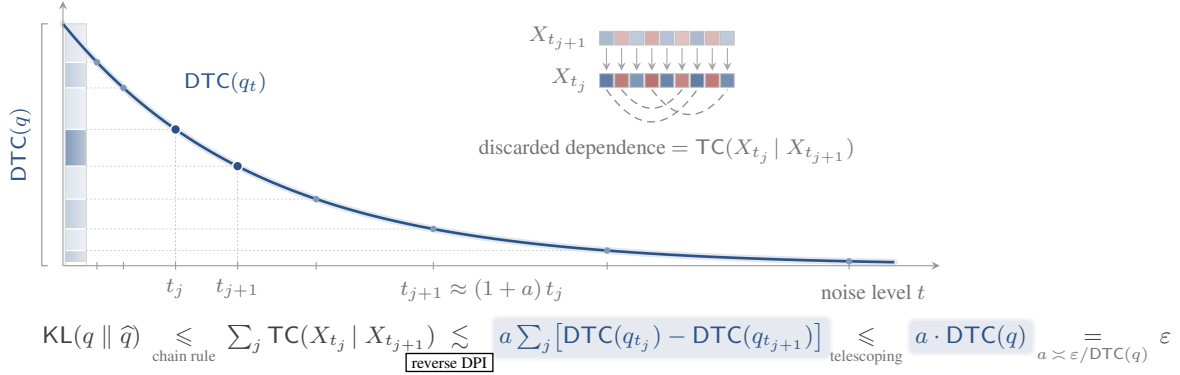
\begin{figure}[h!]
  \centering
  \definecolor{coordblue}{RGB}{43,82,134}
\definecolor{blockred}{RGB}{158,53,51}
\definecolor{winamber}{RGB}{201,151,63}
\pgfdeclarelayer{bg}
\pgfsetlayers{bg,main}
\begin{tikzpicture}[
  x=1cm, y=1cm,
  axis/.style={black!40, line width=0.4pt},
  note/.style={black!55, font=\scriptsize, inner sep=1pt},
]
\foreach \ya/\yb/\tl/\tr in {
    2.691/3.200/16/8, 2.352/2.691/30/15, 1.803/2.352/16/8, 1.317/1.803/58/32,
    0.881/1.317/16/8, 0.486/0.881/30/15, 0.201/0.486/16/8, 0.0586/0.201/30/15}
  \shade[left color=coordblue!\tl, right color=coordblue!\tr] (1.03,\ya) rectangle (1.31,\yb);
\foreach \yj in {2.691, 2.352, 1.803, 1.317, 0.881, 0.486, 0.201}
  \draw[white, line width=0.5pt] (1.03,\yj) -- (1.31,\yj);
\draw[black!15, line width=0.3pt] (1.03,0) rectangle (1.31,3.20);
\foreach \xj/\yj in {1.45/2.691, 1.80/2.352, 2.49/1.803, 3.31/1.317,
                     4.35/0.881, 5.90/0.486, 8.20/0.201}
  \draw[black!20, line width=0.25pt, dash pattern=on 0.4pt off 1.1pt, line cap=round] (1.31,\yj) -- (\xj,\yj);
\draw[black!25, line width=0.25pt, dash pattern=on 0.4pt off 1.1pt, line cap=round] (2.49,0) -- (2.49,1.803);
\draw[black!25, line width=0.25pt, dash pattern=on 0.4pt off 1.1pt, line cap=round] (3.31,0) -- (3.31,1.317);
\draw[coordblue!14, line width=2.6pt, line cap=round, smooth, samples=120, domain=1.0:12.0]
  plot (\x,{3.2*exp(-(\x-1)/2.6)});
\draw[coordblue, line width=1.0pt, smooth, samples=120, domain=1.0:12.0]
  plot (\x,{3.2*exp(-(\x-1)/2.6)});
\node[coordblue, font=\scriptsize, inner sep=1pt] at (3.15,2.42) {$\DTC(q_t)$};
\foreach \xj/\yj in {1.45/2.691, 1.80/2.352, 4.35/0.881, 5.90/0.486, 8.20/0.201, 11.40/0.0586}
  \fill[coordblue!60] (\xj,\yj) circle (0.042);
\fill[white] (2.49,1.803) circle (0.075);
\fill[white] (3.31,1.317) circle (0.075);
\fill[coordblue] (2.49,1.803) circle (0.055);
\fill[coordblue] (3.31,1.317) circle (0.055);
\draw[axis, -stealth] (1.0,0) -- (12.60,0);
\draw[axis, -stealth] (1.0,0) -- (1.0,3.50);
\node[note, anchor=east] at (12.45,-0.34) {noise level $t$};
\foreach \xj in {1.45,1.80,2.49,3.31,4.35,5.90,8.20,11.40}
  \draw[black!40, line width=0.4pt] (\xj,-0.05) -- (\xj,0.05);
\node[note] at (2.49,-0.30) {$t_j$};
\node[note] at (3.31,-0.30) {$t_{j+1}$};
\draw[black!45, line width=0.4pt] (0.80,0) -- (0.72,0) -- (0.72,3.20) -- (0.80,3.20);
\node[coordblue, font=\scriptsize, inner sep=1pt, rotate=90] at (0.44,1.60) {$\DTC(q)$};
\node[note, anchor=east] at (8.00,3.01) {$X_{t_{j+1}}$};
\foreach \i/\c in {0/coordblue!40,1/blockred!35,2/coordblue!30,3/blockred!40,
                   4/coordblue!35,5/blockred!30,6/coordblue!40,7/blockred!35,8/coordblue!30}
  \draw[black!25, line width=0.3pt, fill=\c] (8.10+0.20*\i,2.92) rectangle (8.28+0.20*\i,3.10);
\foreach \i in {0,...,8}
  \draw[black!35, line width=0.4pt, -stealth] (8.19+0.20*\i,2.88) -- (8.19+0.20*\i,2.58);
\node[note, anchor=east] at (8.00,2.45) {$X_{t_j}$};
\foreach \i/\c in {0/coordblue!75,1/blockred!65,2/coordblue!60,3/blockred!70,
                   4/coordblue!70,5/blockred!60,6/coordblue!75,7/blockred!65,8/coordblue!65}
  \draw[black!25, line width=0.3pt, fill=\c] (8.10+0.20*\i,2.36) rectangle (8.28+0.20*\i,2.54);
\draw[black!45, densely dashed, line width=0.5pt]
  (8.39,2.33) .. controls (8.66,2.00) and (8.92,2.00) .. (9.19,2.33);
\draw[black!45, densely dashed, line width=0.5pt]
  (8.79,2.33) .. controls (9.12,1.89) and (9.46,1.89) .. (9.79,2.33);
\draw[black!45, densely dashed, line width=0.5pt]
  (8.19,2.33) .. controls (8.59,1.76) and (8.99,1.76) .. (9.39,2.33);
\node[note] at (8.99,1.52) {discarded dependence $=\TC(X_{t_j}\!\mid X_{t_{j+1}})$};
\node[note] at (6.55,-0.32) {$t_{j+1}\approx(1+a)\,t_j$};
\node[black!75, font=\footnotesize, inner sep=0pt, anchor=base west] (c1) at (0.72,-1.02)
  {$\KL{q}{\widehat q}
    \,\underset{\textcolor{black!55}{\text{chain rule}}}{\le}\,
    \sum_j \TC(X_{t_j}\!\mid X_{t_{j+1}})\;\;$};
\node[black!75, font=\footnotesize, inner sep=0pt, anchor=base west] (c2) at (c1.base east)
  {$\lesssim$};
\node[text=black, font=\tiny, draw=black, line width=0.5pt,
  inner xsep=2.2pt, inner ysep=1.4pt, anchor=north]
  (cdpi) at ([xshift=-4pt,yshift=-1.5pt]c2.south) {reverse DPI};
\node[coordblue, font=\footnotesize, inner xsep=1.8pt, inner ysep=1.4pt, anchor=base west] (c3)
  at ([xshift=9pt]c2.base east) {$a\sum_j \bigl[\DTC(q_{t_j})-\DTC(q_{t_{j+1}})\bigr]$};
\node[black!75, font=\footnotesize, inner sep=0pt, anchor=base west] (c4) at (c3.base east)
  {$\underset{\textcolor{black!55}{\text{telescoping}}}{\le}$};
\node[coordblue, font=\footnotesize, inner xsep=1.8pt, inner ysep=1.4pt, anchor=base west] (c5)
  at ([xshift=2.5pt]c4.base east) {$a\cdot\DTC(q)$};
\node[black!75, font=\footnotesize, inner sep=0pt, anchor=base west] (c6) at ([xshift=1.5pt]c5.base east)
  {$\underset{\textcolor{black!55}{a\,\asymp\,\epsilon/\DTC(q)}}{=}\,\epsilon$};
\begin{pgfonlayer}{bg}
  \foreach \n in {c3,c5}{
    \fill[coordblue!4, rounded corners=2.8pt]
      ([xshift=-1.2pt,yshift=-1.2pt]\n.south west) rectangle ([xshift=1.2pt,yshift=1.2pt]\n.north east);
    \fill[coordblue!8, rounded corners=2.5pt]
      ([xshift=-0.6pt,yshift=-0.6pt]\n.south west) rectangle ([xshift=0.6pt,yshift=0.6pt]\n.north east);
    \fill[coordblue!12, rounded corners=2.2pt]
      (\n.south west) rectangle (\n.north east);
  }
\end{pgfonlayer}
\end{tikzpicture}
  \caption{Illustration of the telescoping argument behind Theorems~\ref{thm:informal-uniform-dtc} and~\ref{thm:informal-gauss-dtc}. The sampler's KL error is a sum over conditional total correlations corresponding to dependencies discarded at each sampling step (inset). The key step is a reverse data-processing inequality which bounds each summand, roughly, by step size times the drop in $\DTC(q_t)$ in that step.}
  \label{fig:dtc-telescope}
\end{figure}

\subsection{Result 2: Separating masked diffusion from other paradigms}

Theorems~\ref{thm:informal-uniform-dtc} and~\ref{thm:informal-gauss-dtc} establish that uniform and Gaussian diffusion can adapt to a certain measure of the intrinsic complexity of the data distribution at least as well as masked diffusion. But they leave open whether or not there is a genuine separation among these paradigms.

In our next set of results, we exhibit such a separation by constructing a simple family of \emph{random empirical measures} for which the query complexity of sampling with a masked diffusion score oracle is strictly higher than with a uniform or Gaussian diffusion oracle. More precisely, distributions in this family are given by sampling $e^{\kappa d}$ many points uniformly at random from the Boolean hypercube, where $0 < \kappa < \log 2$ is an (unknown) absolute constant (see Section~\ref{sec:measure} for a formal definition). Recently, \cite{xun2026query} showed that given an approximate Gaussian diffusion score oracle for such a distribution, at least $\widetilde{\Omega}(\sqrt{d})$ queries are necessary even to determine $\kappa$ to sufficient precision, let alone generate samples close to the distribution. In Appendix~\ref{app:uniform-query-lower}, we prove that this lower bound also applies to uniform diffusion, using similar ideas.

Our main result in this second part of the paper is to show that for these random empirical measures, $\widetilde{O}(\sqrt{d})$ query complexity is actually tight for both uniform and Gaussian diffusion:

\begin{theorem}[Informal, see Theorems~\ref{thm:uniform_random_measure_sampler} and~\ref{thm:gaussian_random_measure_sampler}]\label{thm:informal-sqrtd}
    For empirical measures $q$ supported on $2^{\Theta(d)}$ random points on the Boolean hypercube, there is an algorithm that uses $\widetilde{O}(\sqrt{d}/\epsilon^2)$ queries to any approximate \emph{uniform} diffusion score oracle for $q$ that has score error $1/\mathrm{poly}(d)$ and, with high probability over the randomness of the support of $q$, succeeds in sampling from $q$ to total variation error $\epsilon$. The same result also holds given any approximate \emph{Gaussian} diffusion score oracle that has score error $1/\mathrm{poly}(d)$.
\end{theorem}

\noindent Notably, for this family of distributions, the dual total correlation is $\Theta(d)$, so the scaling achieved in Theorem~\ref{thm:informal-sqrtd} provably goes beyond the dual total correlation scaling implied by Theorems~\ref{thm:informal-uniform-dtc} and~\ref{thm:informal-gauss-dtc}. In contrast, we show that for this same problem instance, $\widetilde{\Omega}(d)$ queries are needed in the case of masked diffusion.

\begin{theorem}[Informal, see Theorem~\ref{thm:masked_random_measure_lower}]\label{thm:informal-d-mask}
    Let $q$ be an empirical measure supported on $2^{\Theta(d)}$ random points on the Boolean hypercube. There is an approximate \emph{masked} diffusion oracle such that any algorithm for sampling making $\widetilde{o}(d)$ queries to this oracle produces samples from a distribution $\widehat{q}$ such that with high probability over the randomness of $q$, $\TV(q, \widehat{q}) \ge 0.99$.
\end{theorem}

\noindent Altogether, these results establish the first provable separation in few-step generation ability among the three prevailing paradigms for diffusion language modeling.

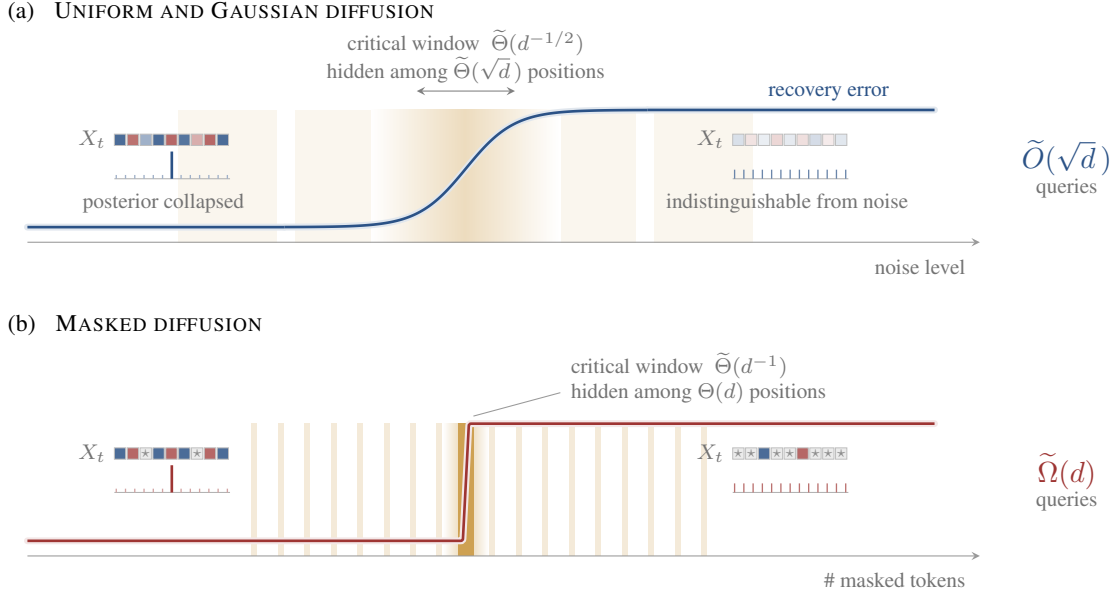
\begin{figure}[h!]
  \centering
  \definecolor{coordblue}{RGB}{43,82,134}
\definecolor{blockred}{RGB}{158,53,51}
\definecolor{winamber}{RGB}{201,151,63}
\begin{tikzpicture}[
  x=1cm, y=1cm,
  axis/.style={black!40, line width=0.4pt},
  curveA/.style={coordblue, line width=1.0pt},
  curveB/.style={blockred, line width=1.0pt},
  note/.style={black!55, font=\scriptsize, inner sep=1pt},
  paneltitle/.style={anchor=west, font=\footnotesize, inner sep=1pt},
]
\begin{scope}
  \foreach \c in {2.95,4.50,7.70,9.25}
    \fill[winamber!9] (\c-0.65,0) rectangle (\c+0.65,1.75);
  \shade[left color=winamber!0, right color=winamber!0, middle color=winamber!34]
    (4.85,0) rectangle (7.35,1.75);
  \node[paneltitle] at (0.0,3.05) {(a)\;\; \textsc{Uniform and Gaussian diffusion}};
  \draw[axis, -stealth] (0.30,0) -- (12.90,0);
  \node[note, anchor=east] at (12.75,-0.32) {noise level};
  \draw[coordblue!14, line width=2.6pt, line cap=round] (0.30,0.20) -- (3.70,0.20);
  \draw[coordblue!14, line width=2.6pt, line cap=round, smooth, samples=80, domain=3.70:8.50]
    plot (\x, {0.20+1.55/(1+exp(-(\x-6.1)/0.32))});
  \draw[coordblue!14, line width=2.6pt, line cap=round] (8.50,1.75) -- (12.30,1.75);
  \draw[curveA] (0.30,0.20) -- (3.70,0.20);
  \draw[curveA, smooth, samples=80, domain=3.70:8.50]
    plot (\x, {0.20+1.55/(1+exp(-(\x-6.1)/0.32))});
  \draw[curveA] (8.50,1.75) -- (12.30,1.75);
  \node[coordblue, font=\scriptsize, inner sep=1pt] at (10.90,1.97) {recovery error};
  \node[note, anchor=east] at (1.38,1.36) {$X_t$};
  \foreach \i/\c in {0/coordblue!85,1/blockred!70,2/coordblue!45,3/coordblue!85,
                      4/blockred!75,5/coordblue!80,6/blockred!40,7/blockred!75,8/coordblue!85}
    \draw[black!25, line width=0.3pt, fill=\c] (1.45+0.17*\i,1.28) rectangle (1.60+0.17*\i,1.43);
  \draw[black!30, line width=0.4pt] (1.45,0.84) -- (2.98,0.84);
  \foreach \k in {0,...,12}
    \draw[coordblue!45, line width=0.5pt] (1.475+0.1225*\k,0.84) -- (1.475+0.1225*\k,0.89);
  \draw[coordblue, line width=1.0pt] (2.21,0.84) -- (2.21,1.20);
  \node[note, anchor=east] at (9.56,1.36) {$X_t$};
  \foreach \i/\c in {0/coordblue!18,1/blockred!14,2/coordblue!10,3/blockred!20,
                      4/coordblue!12,5/blockred!16,6/coordblue!20,7/blockred!10,8/coordblue!14}
    \draw[black!25, line width=0.3pt, fill=\c] (9.63+0.17*\i,1.28) rectangle (9.78+0.17*\i,1.43);
  \draw[black!30, line width=0.4pt] (9.63,0.84) -- (11.16,0.84);
  \foreach \k in {0,...,12}
    \draw[coordblue!70, line width=0.5pt] (9.655+0.1225*\k,0.84) -- (9.655+0.1225*\k,0.96);
  \node[note] at (2.10,0.52) {posterior collapsed};
  \node[note] at (10.35,0.52) {indistinguishable from noise};
  \draw[stealth-stealth, black!50, line width=0.4pt] (5.45,2.00) -- (6.75,2.00);
  \node[note] at (6.10,2.64) {critical window\; $\wt\Theta(d^{-1/2})$};
  \node[note] at (6.10,2.28) {hidden among $\wt\Theta(\sqrt d\,)$ positions};
  \node[coordblue] at (14.05,1.12) {$\wt O(\sqrt d\,)$};
  \node[note] at (14.05,0.72) {queries};
\end{scope}
\begin{scope}[yshift=-4.15cm]
  \foreach \k in {0,...,7,9,10,...,17}
    \fill[winamber!20] (3.265+0.35*\k,0) rectangle (3.335+0.35*\k,1.75);
  \shade[left color=winamber!0, right color=winamber!0, middle color=winamber!45]
    (5.80,0) rectangle (6.40,1.75);
  \fill[winamber!90] (6.00,0) rectangle (6.20,1.75);
  \node[paneltitle] at (0.0,3.05) {(b)\;\; \textsc{Masked diffusion}};
  \draw[axis, -stealth] (0.30,0) -- (12.90,0);
  \node[note, anchor=east] at (12.75,-0.32) {\# masked tokens};
  \draw[blockred!12, line width=2.6pt, line cap=round, rounded corners=0.8pt]
    (0.30,0.20) -- (6.06,0.20) -- (6.14,1.75) -- (12.30,1.75);
  \draw[curveB, rounded corners=0.8pt] (0.30,0.20) -- (6.06,0.20) -- (6.14,1.75) -- (12.30,1.75);
  \node[note, anchor=east] at (1.38,1.36) {$X_t$};
  \foreach \i/\c in {0/coordblue!85,1/blockred!75,2/black!8,3/coordblue!85,
                      4/blockred!75,5/coordblue!85,6/black!8,7/blockred!75,8/coordblue!85}
    \draw[black!25, line width=0.3pt, fill=\c] (1.45+0.17*\i,1.28) rectangle (1.60+0.17*\i,1.43);
  \foreach \i in {2,6}
    \node[black!45, font=\tiny] at (1.525+0.17*\i,1.355) {$\star$};
  \draw[black!30, line width=0.4pt] (1.45,0.84) -- (2.98,0.84);
  \foreach \k in {0,...,12}
    \draw[blockred!45, line width=0.5pt] (1.475+0.1225*\k,0.84) -- (1.475+0.1225*\k,0.89);
  \draw[blockred, line width=1.0pt] (2.21,0.84) -- (2.21,1.20);
  \node[note, anchor=east] at (9.56,1.36) {$X_t$};
  \foreach \i/\c in {0/black!8,1/black!8,2/coordblue!85,3/black!8,4/black!8,
                      5/blockred!75,6/black!8,7/black!8,8/black!8}
    \draw[black!25, line width=0.3pt, fill=\c] (9.63+0.17*\i,1.28) rectangle (9.78+0.17*\i,1.43);
  \foreach \i in {0,1,3,4,6,7,8}
    \node[black!45, font=\tiny] at (9.705+0.17*\i,1.355) {$\star$};
  \draw[black!30, line width=0.4pt] (9.63,0.84) -- (11.16,0.84);
  \foreach \k in {0,...,12}
    \draw[blockred!70, line width=0.5pt] (9.655+0.1225*\k,0.84) -- (9.655+0.1225*\k,0.96);
  \node[note, anchor=west] at (7.45,2.52) {critical window\; $\wt\Theta(d^{-1})$};
  \node[note, anchor=west] at (7.45,2.16) {hidden among $\Theta(d)$ positions};
  \draw[black!40, line width=0.4pt] (7.35,2.16) -- (6.20,1.85);
  \node[blockred] at (14.05,1.12) {$\wt\Omega(d)$};
  \node[note] at (14.05,0.72) {queries};
\end{scope}
\end{tikzpicture}
  \caption{Critical window for masked diffusion is narrower than for uniform and Gaussian diffusion, so more queries are needed to locate it before one can generate a sample.}
  \label{fig:separation}
\end{figure}

We highlight that this separation manifests through a mechanism fundamentally different from the ``commitment'' issue mentioned previously. Instead, it comes from a difference in the sharpness of a certain phase transition arising in the sampling dynamics for masked diffusion versus uniform and Gaussian diffusion. To see this, we note that in both the $\widetilde{\Omega}(\sqrt{d})$ lower bound for uniform and Gaussian diffusion and the $\widetilde{\Omega}(d)$ lower bound for masked diffusion, the key idea is that for random empirical measures, unless one queries the score oracle near the right noise level, the approximation error in the score oracle can be designed to suppress any useful signal about the underlying distribution $q$. Crucially, however, the range of acceptable noise levels for which this is the case is far narrower for masked diffusion than for uniform and Gaussian diffusion: whereas the informative range for masked diffusion occupies an $O(\log(d)/d)$ fraction of the full noise spectrum (Proposition~\ref{prop:masked_critical_window}), the informative range for uniform and Gaussian diffusion occupies a far wider $O(\sqrt{\log(d)/d})$ fraction (Propositions~\ref{prop:uniform_critical_window} and~\ref{prop:gaussian_critical_window}). As a result $\widetilde{\Omega}(d)$ queries are needed in the former case, whereas only $\widetilde{O}(\sqrt{d})$ suffice in the latter case. Figure~\ref{fig:window-sim} illustrates this discrepancy numerically.

\begin{figure}[!t]
  \centering
  \includegraphics[width=\textwidth]{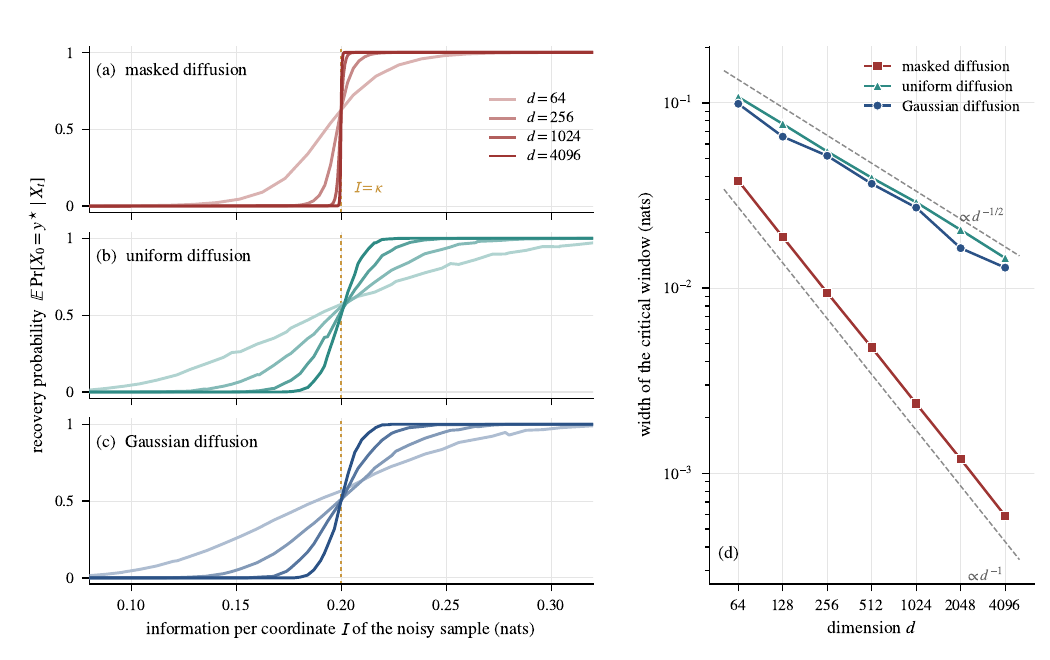}
  \caption{Simulated critical windows for the random empirical measure with $M=e^{\kappa d}$ codewords, $\kappa=0.2$, extending a simulation of~\cite{xun2026query} from Gaussian diffusion to masked and uniform diffusion. (a)--(c) Probability that the posterior at a given noise level recovers the planted codeword, plotted against the information per coordinate of the noisy sample. The information per coordinate $\mathsf{I}$ is the mutual information between a uniformly random bit and its corrupted version: $(m/d)\log2$ for masked diffusion with $m$ revealed coordinates, and $\uMI(t)$, $\gMI(t)$ for uniform and Gaussian diffusion. By Eq.~\eqref{eq:critical_times}, every process has its critical level at $\mathsf{I}=\kappa$, placing the three transitions on a common scale. (d) Width of the transition versus $d$: it decays like $1/d$ for masked diffusion but only like $1/\sqrt d$ for uniform and Gaussian diffusion. 
  }
  \label{fig:window-sim}
\end{figure}

These \emph{critical windows} have previously been studied in the context of \emph{speciation transitions} in diffusion sampling~\cite{biroli2024dynamical,li2024criticalwindowsnonasymptotictheory,pmlr-v267-li25at,sclocchi2025phase,raya2023spontaneous}, but to our knowledge this is the first time critical windows have been used to derive new upper bounds on the query complexity of diffusion sampling. Roughly speaking, after identifying the location of the window using $\widetilde{O}(\sqrt{d})$ queries to a uniform or Gaussian diffusion score oracle, we then argue that at noise levels above this critical window, the score is well-approximated by that of the uniform distribution, and at noise levels below the critical window, the posterior distribution over the clean sample is so pronounced that one can easily ``round'' to the correct point in the support. Thus, it suffices to design an algorithm that can efficiently locate and ``cross'' the window, the details of which are given in Sections~\ref{sec:locate_window_uniform} and~\ref{sec:locate-gaussian-window}. The idea behind this argument is illustrated in Figure~\ref{fig:separation}.

\paragraph{Concurrent work.} Recently, Dmitriev, Huang, and Wei~\cite{dmitriev2026provably} and Wainwright~\cite{wainwright2026information} also independently established DTC-adaptive bounds for uniform diffusion. Their proof strategies bear similarities to the proof of our Theorem~\ref{thm:informal-uniform-dtc}, in particular bounding the sampling error by a telescoping sum of differences between dual total correlation at successive noise levels. Interestingly, their works give a characterization of the discretization error for \emph{any} product forward process in terms of the change in the mutual information between one coordinate and the remaining coordinates at other times. In contrast to our work, Wainwright~\cite{wainwright2026information} further leverages this characterization to give a rate adaptive to the total correlation, and \cite{dmitriev2026provably} proves a DTC-adaptive bound for \emph{remasking} diffusion~\cite{wang2026remasking}. In contrast to these two works, in our work we prove a DTC-adaptive bound for \emph{Gaussian} diffusion and also study separations among the various frameworks in terms of score oracle query complexity.

\subsection{Related work}

\paragraph{Theory for discrete diffusion sampling.} 
Campbell et al.~\cite{campbell2022continuous} formulated continuous-time discrete denoising through continuous-time Markov chains. Chen and Ying~\cite{chen2024convergence} used uniformization to implement the reverse chain associated with an estimated score and obtained total-variation and KL guarantees on the hypercube, notably obtaining zero discretization error in $O(d)$ uniform diffusion score oracle queries. 
Li and Cai~\cite{li2026breaking} initiated the theoretical study of the query complexity of masked diffusion as a function of information-theoretic properties of the data distribution.
Chen, Cong, and Li~\cite{chen2025optimal} and Lavenant and Zanella~\cite{lavenant2025error} derived an exact characterization for the sampling error in terms of the data distribution's \emph{information curve}, which can be used to obtain query complexity bounds scaling linearly in the total correlation and dual total correlation. Closely related to our lower bound for masked diffusion, \cite{chen2025optimal} also gave an $\Omega(d)$ lower bound with a different construction based on MDS codes. Dmitriev et al.~\cite{dmitriev2026efficient} and Zhao and Cai~\cite{zhao2026adaptation} proved sharp convergence guarantees for standard uniform and masking $\tau$-leaping schemes, showing that the latter can adapt to the total and dual total correlation of the distribution. 

Separately, Ren et al.~\cite{ren2025discrete} developed a stochastic-integral framework for KL analysis of time-discretized reverse processes. Building on this framework, Ren et al.~\cite{ren2025fast} developed and analyzed higher-order solvers for discrete diffusion. More broadly, there is a large body of orthogonal work that analyzes the convergence of CTMC-based discrete diffusion samplers, with particular attention to the dependence on the alphabet size~\cite{zhang2024convergence,pham2025bitlevel,liang2025absorb,liang2025discrete,pham2026discrete,kan2026vocabulary}.

Additionally, there is an extensive theoretical literature on \emph{Gaussian} diffusion models --- see the recent notes of~\cite{lu2026mathematical} for a comprehensive overview. Most relevant to our work is the lower bound of Xun and Price~\cite{xun2026query} showing that $\widetilde{\Omega}(\sqrt{d})$ queries are necessary for Gaussian diffusion. The second half of this paper is based on one of the constructions studied in that work. In addition, there have been a number of works on \emph{algorithmic stochastic localization} for sampling from distributions over the Boolean hypercube~\cite{el2022sampling,montanari2023posterior,el2023sampling}, essentially by regarding the domain as a subset of Euclidean space, performing Gaussian diffusion, and rounding at the end.

\paragraph{Critical windows and speciation.} Numerous works~\cite{georgiev2023journey,raya2023spontaneous,sclocchi2024probinglatenthierarchicalstructure,sclocchi2025phase,biroli2024dynamical,li2024criticalwindowsnonasymptotictheory,pmlr-v267-li25at,achilli2026theory} have studied an intriguing property of real-world diffusion models whereby key aspects of the final generation emerge over a narrow range of times in the reverse process. These works characterize this behavior as a phase transition, interchangeably termed a \emph{critical window} or a \emph{speciation transition}, and provide various probabilistic models of data under which this arises. This phenomenon is the same one that manifests in the context of the random empirical measures considered in~\cite{xun2026query} and that we study in the second part of this work, which is in turn closely related to concepts from the statistics and information theory literature like the all-or-nothing phenomenon in high-dimensional inference~\cite{niles2020all} and channel resolvability and soft covering~\cite{han1993approximation}.

\subsection*{Organization} The DTC-adaptive bounds for uniform and Gaussian diffusion are proved in Sections~\ref{sec:uniform-dtc} and~\ref{sec:gaussian-dtc}. The random empirical measure is introduced in Section~\ref{sec:measure}. The uniform and Gaussian upper bounds for this family are proved in Sections~\ref{sec:uniform-upper} and~\ref{sec:gaussian-upper}, and the masked lower bound is proved in Section~\ref{sec:masked-lower}. The uniform diffusion query lower bound appears in Appendix~\ref{app:uniform-query-lower}. The deferred proofs appear in Appendix~\ref{app:deferred-proofs}.

\section{Preliminaries}
\label{sec:preliminaries}

\subsection{Uniform diffusion}
\label{sec:uniform-diffusion}

Given alphabet size $S\ge 2$, define $\Sigma = [S]$ and let $\mathcal{X} = \Sigma^d$. Let $q$ be a distribution over $\mathcal{X}$. Let $\Delta(\Sigma)$ denote the probability simplex over $\Sigma$.

For every $t \ge 0$ define
\begin{equation}
  \alpha_t = e^{-t}\,, \qquad \beta_t = \frac{1 - \alpha_t}{S}\,,
\end{equation}
and, for $z\in\Sigma$, define the one-coordinate transition kernel
\begin{equation}
  K_t^{\sf U}(\cdot \mid z) = \alpha_t \bone{z = \cdot} + \beta_t\,.
\end{equation}
These kernels $(K_t^{\sf U})$ form a semigroup. Given $X_0 \sim q$, let $(X_t)_{t\ge 0}$ denote the Markov process with transition kernel $(K_t^{\sf U})^{\otimes d}$, and let $q_t \triangleq\law(X_t)$.

Given $y\in\calX$, $i\in[d]$, and $a\in\Sigma$, let $\yia$ be the point obtained by replacing the $i$-th coordinate of $y$ with $a$. The \emph{(uniform diffusion) score at time $t$} is the function $s_t: \calX\to \RR_{\ge 0}^{[d]\times \Sigma}$ given by the likelihood ratio
\begin{equation}
  \unifscore{t}{y}{i,a} = \frac{q_t(\yia)}{q_t(y)}\,.
\end{equation}
The score naturally admits an interpretation both in terms of posterior marginals and in terms of Gibbs sampling marginals:
\begin{lemma}\label{lem:uniform_score_to_marginals}
  Given a vector $v\in\mathbb{R}_{>0}^S$ and given $h > 0$ and $b\in\Sigma$, define probability distributions $\post_{h,b}, \gibbs_h$ over $\Sigma$ by
  \begin{equation}
    \post_{h,b}[v](a) \triangleq \frac{\alpha_h \bone{a=b} + \beta_h}{\alpha_h}\Bigl(v_a - \beta_h\sum_{a'\in\Sigma} v_{a'}\Bigr)\,.
  \end{equation}
  \begin{equation}
    \gibbs_h[v](a) \triangleq \frac{1}{\alpha_h}\Bigl(\frac{v_a}{\sum_{{a'}\in\Sigma} v_{a'}} - \beta_h\Bigr)\,.
  \end{equation}
  Let $h = t - s$. Given $y\in\calX$,
  \begin{equation}
    \Pr{X^i_s = a \mid X_t = y} = \post_{h,y_i}[s_t(y)[i,\cdot]](a)\,,
  \end{equation}
  and
  \begin{equation}
    \Pr{X^i_s = a\mid (X_t)_{-i} = y_{-i}} = \gibbs_h[s_t(y)[i,\cdot]](a)\,.
  \end{equation}
  Here, vectors with the subscript $-i$ denote the vectors obtained by deleting coordinate $i$.
\end{lemma}

\noindent We defer the proof to Appendix~\ref{app:deferred-proofs}.

We will consider estimated scores $(\widehat{s}_t)$ and quantify their error as follows. First, without loss of generality, we will assume that the estimated scores are \emph{normalized}, i.e., for all $t$, $y\in\calX$, and $i\in[d]$, $\widehat{s}_t(y)[i,y_i] = 1$. Given $c,s > 0$, define the \emph{(entropic) Bregman divergence} by
\begin{equation}
  \psi(c,s) \triangleq s - c + c\log(c/s)\,.
\end{equation}

\begin{definition}[Score error -- uniform diffusion]\label{def:uniform-score-error}
  Given a normalized score estimate $\widehat{s}_t$ at time $t>0$, its \emph{score error} is
  \begin{equation}
    \epsilon_{\sf unif}(t) \triangleq \EE_{Y\sim q_t} \sum^d_{i=1} \sum_{a\in \Sigma} \psi\bigl(\unifscore{t}{Y}{i,a},\widehatunifscore{t}{Y}{i,a}\bigr)\,.
  \end{equation}
  We will assume that $\epsilon_{\sf unif}(t) < \infty$ for all $t>0$, so that $\widehat{s}_t(y)[i,a] > 0$ for all $i,a,t,y$.
\end{definition}

\noindent This is equivalent to the standard score-entropy loss used in training uniform diffusion models~\cite{lou2023discrete}. In Section~\ref{sec:approx_score_uniform}, we discuss a robust analogue of Lemma~\ref{lem:uniform_score_to_marginals} for converting an approximate score into a valid estimate for the posterior marginals.

For lower-bound arguments, it will be convenient to make the score estimate and target distribution explicit, using the oracle language. An \emph{approximate uniform diffusion oracle} $\calO$, queried at $(t,y)$, returns a score estimate $\widehat s_t^{\mathcal O}(y)$. We write
\begin{equation}
  \epsilon_{\sf unif}(\mathcal O;q,t) \triangleq \EE_{Y\sim q_t} \sum^d_{i=1} \sum_{a\in \Sigma} \psi\bigl(\unifscore{t}{Y}{i,a},\widehat{s}_t^\calO(Y)[i,a]\bigr)
\end{equation}
for the corresponding score error. When $q$ and $\calO$ are clear from context, we abbreviate $\epsilon_{\sf unif}(\calO;q,t)$ by $\epsilon_{\sf unif}(t)$ and write $\widehat{s}_t$ for $\widehat{s}^{\calO}_t$.

Finally, for the case of binary alphabet $\Sigma=\{-1,+1\}$, the forward likelihood ratio has a convenient overlap form which will be useful in Section~\ref{sec:uniform-upper}.

\begin{lemma}[Binary likelihood ratio]\label{lem:uniform-likelihood-ratio}
  For every $x,z\in\{-1,+1\}^d$,
  \begin{equation}
    2^d (K_t^{\sf U})^{\otimes d}(x\mid z)
    =\exp\!\left(\sum_{i=1}^d\log(1+e^{-t}x_i z_i)\right)
    =(1+e^{-t})^d\tanh(t/2)^{d_H(x,z)}\,,
  \end{equation}
  where $d_H(x,z)$ denotes the Hamming distance between $x$ and $z$.
\end{lemma}

\begin{proof}
  Coordinatewise, $2K_t^{\sf U}(x_i\mid z_i)=1+e^{-t}x_i z_i$. Multiplying these identities gives the first equality. A matching coordinate contributes $1+e^{-t}$ and a mismatching coordinate contributes $1-e^{-t}=(1+e^{-t})\tanh(t/2)$, which gives the second.
\end{proof}

\subsection{Gaussian diffusion}
\label{sec:gaussian-diffusion}

We begin by describing Gaussian diffusion for general distributions over $\mathbb{R}^d$, before specializing to embeddings of \emph{discrete distributions}. Let $q$ be a distribution on $\RR^d$. Define the \emph{Ornstein--Uhlenbeck process} $(X_t)_{t\ge 0}$ by
\begin{equation}
  \d X_t = -X_t\,\d t + \sqrt{2}\,\d B_t\,, \qquad X_0 \sim q\,,
\end{equation}
where $(B_t)_{t\ge 0}$ is a standard Brownian motion, and let $q_t \triangleq\law(X_t)$. Defining 
\begin{equation}
  \sigma_t = \sqrt{1 - e^{-2t}}\,,
\end{equation}
we see that the transition kernel for this process is given by
\begin{equation}
  K_t^{\sf G}(x\mid z) = \frac{1}{(2\pi\sigma_t^2)^{d/2}} \exp\!\left( -\frac{\|x-e^{-t}z\|^2}{2\sigma_t^2} \right)\,.
\end{equation}
These kernels $(K_t^{\sf G})$ form a semigroup.

The \emph{(Gaussian diffusion) score at time $t$} is the function $s_t: \RR^d\to \RR^d$ given by
\begin{equation}
  s_t(x) = \nabla \log q_t(x)\,,
\end{equation}
The score naturally admits an interpretation in terms of posterior expectations:

\begin{lemma}[Tweedie's formula]\label{lem:tweedie}
  For any $x \in \RR^d$ and $t>0$,
  \begin{equation}
    \EE[X_0 \mid X_t = x] = e^t x + e^t \sigma_t^2 s_t(x)\,.
  \end{equation}
\end{lemma}

\noindent We will consider estimated scores $(\widehat{s}_t)$ and quantify their error as follows.
\begin{definition}[Score error -- Gaussian diffusion]\label{def:gaussian-score-error}
  Given score estimate $\widehat{s}_t$ at time $t$, its \emph{$L^2$ score error} is defined by
  \begin{equation}
    \epsilon_{{\sf gauss}}(t) \triangleq \sigma_t^2 \left(\EE_{X_t\sim q_t} \|s_t(X_t) - \widehat{s}_t(X_t)\|^2_2\right)^{1/2}\,. \label{eq:gaussian-score-error}
  \end{equation}
\end{definition}

\noindent By Lemma~\ref{lem:tweedie}, $\sigma_t^2\bigl(s_t(x)-\widehat{s}_t(x)\bigr)=e^{-t}\bigl(\EE[X_0\mid X_t=x]-\widehat{m}_t(x)\bigr)$ for the denoiser $\widehat{m}_t(x)\triangleq e^{t}x+e^{t}\sigma_t^2\widehat{s}_t(x)$ associated to $\widehat{s}_t$, so $\epsilon_{\sf gauss}(t)^2=e^{-2t}\,\EE_{X_t\sim q_t}\|\widehat{m}_t(X_t)-\EE[X_0\mid X_t]\|_2^2$ is, up to the factor $e^{-2t}$, the excess risk of $\widehat{m}_t$ for the denoising objective $\EE\|\widehat{m}_t(X_t)-X_0\|_2^2$ used to train Gaussian diffusion models~\cite{sohl2015deep,vincent2011connection,song2019generative,ho2020denoising,song2021scorebased}, whose minimizer is the posterior mean.

In this work, we will study Gaussian diffusion in the context of sampling from \emph{discrete distributions}. For general finite alphabet $\Sigma$, let $q_{\sf pre}$ be any distribution over $\calX = \Sigma^d$. We consider the following standard embedding. Given a sample $A = (A_1,\ldots,A_d)$ from $q_{\sf pre}$ over $\calX = [S]^d$, apply the \emph{one-hot encoding} to obtain $X_0 = (X^{(1)}_0,\ldots,X^{(d)}_0) \in \RR^{Sd}$, where $X^{(i)}_0 = e_{A_i}$ and $e_a \in \RR^S$ denotes the $a$-th standard basis vector. We then take $q$ to denote the pushforward of the data distribution under this one-hot encoding map.

We will also consider the special case of binary alphabet $\Sigma = \{-1,1\}$, in which case we will more directly regard $q_{\sf pre}$ as a distribution over $\mathbb{R}^d$ in the natural way by regarding $\calX = \{-1,1\}^d$ as a subset of $\mathbb{R}^d$.

\subsection{Masked diffusion}
\label{sec:masked-diffusion}

Given alphabet size $S\ge 2$, unlike in the previous sections we distinguish between $[S]$ and $\Sigma$, defining $\Sigma = [S]\cup\{\star\}$, where $\star\notin[S]$ is a special mask symbol. Let
\begin{equation}
  \mathcal{X}=[S]^d\,, \qquad\overline{\mathcal{X}}=\Sigma^d\,.
\end{equation}
Let $q$ be a distribution over $\mathcal{X}$. As in the case of uniform diffusion, let $\alpha_t=e^{-t}$. For $z\in[S]$, define the one-coordinate transition kernel
\begin{equation}
  K_t^{\sf M}(\cdot\mid z) = \alpha_t \bone{z=\cdot} + (1-\alpha_t)\bone{\star=\cdot}\,.
\end{equation}
These kernels form a semigroup. Given $X_0\sim q$, let $(X_t)_{t\ge0}$ denote the Markov process with transition kernel $\left(K_t^{\sf M}\right)^{\otimes d}$, and let $q_t\triangleq\law(X_t)$.

For $y\in\overline{\mathcal X}$, define the set of revealed coordinates by
\begin{equation}
  I(y)\triangleq\{i\in[d]:y_i\neq\star\}\,.
\end{equation}
We identify $y$ with the \emph{partial assignment} $(I(y),y_{I(y)})$; conversely, given $I\subseteq[d]$ and $x\in\mathcal X$, we write $x^{(I)}\in\overline{\mathcal X}$ for the string obtained from $x$ by masking the coordinates outside $I$, so that $X_t = X_0^{(I(X_t))}$ and observing $X_t$ is equivalent to observing $(I(X_t),X_{0,I(X_t)})$. We call a partial assignment $(I,x_I)$ \emph{consistent} if $\Pr{X_{0,I}=x_I}>0$. Given a consistent partial assignment and an unrevealed coordinate $i\notin I$, define the posterior marginal
\begin{equation}
  q_{i\mid I}(a\mid x_I) \triangleq \Pr{X_{0,i}=a\mid X_{0,I}=x_I}\,, \qquad a\in[S]\,.
\end{equation}

Given $y\in\overline{\mathcal X}$, $i\notin I(y)$, and $a\in[S]$, let $y^{i\leftarrow a}$ be given by replacing the $i$-th coordinate of $y$ with $a$. The \emph{(masked diffusion) score at time $t$} is the function $s_t:\overline{\mathcal X}\to \RR_{\ge 0}^{[d]\times [S]}$ given by
\begin{equation}
  s_t(y)[i,a] \triangleq \frac{q_t(y^{i\leftarrow a})}{q_t(y)}\,, \qquad i\notin I(y)\,, \quad a\in[S]\,,
\end{equation}
and $s_t(y)[i,a]$ is undefined for $i\in I(y)$.

The masked diffusion score has a direct interpretation in terms of posterior marginals.

\begin{lemma}\label{lem:masked_score_interpretation}
  Let $y\in\overline{\mathcal X}$ be consistent and let $I=I(y)\subsetneq[d]$. For any $i\notin I$ and $a\in[S]$,
  \begin{equation}
    s_t(y)[i,a] = \frac{\alpha_t}{1-\alpha_t} q_{i\mid I}(a\mid y_I)\,.
  \end{equation}
\end{lemma}

\begin{proof}
  Observe that $q_{i\mid I}(a\mid y_I)=\frac{\Pr{X_{0,i}=a, X_{0,I}=y_I}}{\Pr{X_{0,I}=y_I}}$, while $q_t(\yia)=\alpha_t^{|I|+1}(1-\alpha_t)^{d-|I|-1} \Pr{X_{0,i}=a, X_{0,I}=y_I}$ and $q_t(y)=\alpha_t^{|I|}(1-\alpha_t)^{d-|I|} \Pr{X_{0,I}=y_I}$. Taking the ratio gives the desired result.
\end{proof}

By Lemma~\ref{lem:masked_score_interpretation}, the masked diffusion score is equivalent to posterior marginals, up to a deterministic time-dependent factor.

We will consider estimated scores $(\widehat{s}_t)$ and quantify their error as follows. First, we will assume that the estimated scores are \emph{consistent}, i.e., that for all $t$, $y\in \overline{\mathcal X}$, and $i\notin I(y)$, $\sum_{a\in[S]}\frac{1-\alpha_t}{\alpha_t}\widehat{s}_t(y)[i,a]=1$, and define the \emph{estimated posterior marginals} by $\widehat{q}_{i\mid I}(a\mid y_I) \triangleq \frac{1-\alpha_t}{\alpha_t}\widehat{s}_t(y)[i,a]$, which we assume do not depend on $t$, as is the case for the true posterior marginals. By Lemma~\ref{lem:masked_score_interpretation}, querying such a score estimate at $(t,y)$ is then the same as querying an \emph{approximate masked diffusion oracle} $\calO$ with the partial assignment $(I,x_I)=(I(y),y_{I(y)})$ and receiving the distributions $\widehat{q}^{\calO}_{i\mid I}(\cdot\mid x_I)$ over $[S]$ for all $i\notin I$; we use both descriptions interchangeably, writing $\widehat{q}_{i\mid I}$ for $\widehat{q}^{\calO}_{i\mid I}$ when $\calO$ is clear from context. On inconsistent partial assignments the posterior marginals are undefined, and the oracle may return arbitrary distributions.

\begin{definition}[Score error -- masked diffusion]\label{def:masked-score-error}
  Given consistent score estimate $\widehat{s}_t$ at time $t$, its \emph{score error} is defined as
  \begin{equation}
    \epsilon_{\sf mask}(t) \triangleq \frac{1-\alpha_t}{\alpha_t}\EE_{Y\sim q_t}\left[\frac{1}{d-|I(Y)|}\sum_{i\in [d]\setminus I(Y)} \sum_{a\in[S]} \psi(s_t(Y)[i,a],\widehat{s}_t(Y)[i,a])\right]\,,
  \end{equation}
  with the convention that the bracket is $0$ when $I(Y)=[d]$. For the associated oracle $\calO$ and $0\le m<d$, the \emph{score error of $\calO$ at level $m$} is, for any $t>0$, the same quantity conditioned on $|I(Y)|=m$ (it does not depend on $t$, see below):
  \begin{equation}
    \epsilon_{\sf mask}(\calO;q,m) \triangleq \frac{1-\alpha_t}{\alpha_t}\EE_{\substack{I\sim\Unif\bigl(\binom{[d]}{m}\bigr)\\ X_0\sim q}}\left[\frac{1}{d-m}\sum_{i\notin I} \sum_{a\in[S]} \psi\Bigl(s_t\bigl(X_0^{(I)}\bigr)[i,a],\widehat{s}_t\bigl(X_0^{(I)}\bigr)[i,a]\Bigr)\right]\,.
  \end{equation}
\end{definition}

\noindent Since $I(X_t)$ is a uniformly random subset of size $|I(X_t)|\sim\mathrm{Bin}(d,\alpha_t)$ independent of $X_0$, we have $\epsilon_{\sf mask}(t)=\EE_{m\sim\mathrm{Bin}(d,\alpha_t)}[\epsilon_{\sf mask}(\calO;q,m)\bone{m<d}]$, and $\epsilon_{\sf mask}(\calO;q,m)$ does not depend on $t$ (see Lemma~\ref{lem:masked_score_error_kl}). The level $m$ is the natural notion of noise level for masked diffusion: unlike for uniform and Gaussian diffusion, it is observable from $X_t$, and it is the quantity that samplers control in practice.

The following lemma reinterprets the masked diffusion score error in terms of the estimated posterior marginals.
\begin{lemma}\label{lem:masked_score_error_kl}
  Let $(\widehat{s}_t)$ be consistent, with associated oracle $\calO$. Then for every $0\le m<d$,
  \begin{equation}
    \epsilon_{\sf mask}(\calO;q,m) = \EE_{\substack{I\sim\Unif\bigl(\binom{[d]}{m}\bigr)\\ X_0\sim q}}\left[\frac{1}{d-m}\sum_{i\notin I} \KL{q_{i\mid I}(\cdot\mid X_{0,I})}{\widehat{q}^{\calO}_{i\mid I}(\cdot\mid X_{0,I})}\right]\,. \label{eq:mask_error_kl}
  \end{equation}
  In particular, $\epsilon_{\sf mask}(\calO;q,m)$ depends on $\calO$ only through the marginals it returns on consistent partial assignments with $m$ revealed coordinates.
\end{lemma}

\noindent We defer the proof to Appendix~\ref{app:deferred-proofs}. The right-hand side of Eq.~\eqref{eq:mask_error_kl} is, up to an additive term independent of $\calO$ and the weighting over levels $m$, the cross-entropy loss $\sum_{i\notin I}-\log\widehat{q}_{i\mid I}(X_{0,i}\mid X_{0,I})$ on masked positions, commonly used to train masked diffusion models~\cite{sahoo2024simple,shi2024simplified}.

\subsection{Information-theoretic quantities}

Here we formally define the relevant information-theoretic quantities that we use to quantify the intrinsic complexity of a data distribution.

\begin{definition}
  Given a random vector $X = (X_1,\ldots,X_d)$, the \emph{total correlation} of $X$ is given by
  \begin{equation}
    \TC(X) = \KL{\law(X)}{\otimes_i \law(X_i)} = \sum_i I(X_i; X_{i+1:d})\,.
  \end{equation}
  The \emph{dual total correlation} of $X$ is given by
  \begin{equation}
    \DTC(X) \triangleq H(X) - \sum_i H(X_i \mid X_{-i})\,.
  \end{equation}
  More generally, their conditional versions are
  \begin{equation}
    \TC(X\mid Y)\triangleq\sum_iH(X_i\mid Y)-H(X\mid Y)=\sum_iI(X_i;X_{i+1:d}\mid Y)\,,
  \end{equation}
  and
  \begin{equation}
    \DTC(X\mid Y)
    \triangleq H(X\mid Y)-\sum_i H(X_i\mid X_{-i},Y)\,.
  \end{equation}
\end{definition}

\noindent The following lemma bounds the extent to which a coordinatewise channel decreases dual total correlation in terms of a sum of conditional mutual information terms.

\begin{lemma}\label{lem:dtc_decrement}
  Let $A = (A_1,\ldots,A_d)$ and let $B = (B_1,\ldots,B_d)$ be obtained by applying independent coordinate channels to $A$. Then
  \begin{equation}
    \DTC(A) - \DTC(B) = \DTC(A\mid B) + \sum^d_{i=1} I(B_i; A_{-i} \mid B_{-i})\,.
  \end{equation}
  In particular,
  \begin{equation}
    \sum^d_{i=1} I(B_i; A_{-i} \mid B_{-i}) \le \DTC(A) - \DTC(B)\,. \label{eq:decrement}
  \end{equation}
\end{lemma}

\begin{proof}
  By definition,
  \begin{multline}
    \DTC(A) - \DTC(B) - \DTC(A\mid B) \\
    = H(A) - H(B) - H(A\mid B) - \sum_i \Bigl(H(A_i\mid A_{-i}) - H(B_i \mid B_{-i}) - H(A_i \mid A_{-i}, B)\Bigr)\,. \label{eq:expand_DTC}
  \end{multline}
  Note that
  \begin{equation}
    H(A) - H(B) - H(A\mid B) = -H(B\mid A) = -\sum_i H(B_i \mid A_i)\,,
  \end{equation}
  where the second equality follows because $B$ is obtained through a product channel applied to $A$. For the same reason,
  \begin{equation}
    H(A_i \mid A_{-i}, B) = H(A_i \mid A_{-i}, B_i)\,,
  \end{equation}
  and
  \begin{multline}
    H(A_i \mid A_{-i}) - H(A_i\mid A_{-i}, B_i) + H(B_i \mid A_i) \\
    = H(A_i \mid A_{-i}) - H(A_i \mid A_{-i}, B_i) + H(B_i \mid A_i, A_{-i}) = H(B_i \mid A_{-i})\,.
  \end{multline}
  Substituting these into Eq.~\eqref{eq:expand_DTC},
  \begin{align}
    \DTC(A) - \DTC(B) - \DTC(A\mid B) &= \sum_i \Bigl(H(B_i \mid B_{-i}) - H(B_i \mid A_{-i})\Bigr) \\
    \intertext{Finally, note that $H(B_i \mid A_{-i}) = H(B_i \mid A_{-i}, B_{-i})$, so the above can be written as}
    &= \sum_i I(B_i; A_{-i} \mid B_{-i})\,,
  \end{align}
  as claimed.
\end{proof}

\section{Uniform diffusion can scale with dual total correlation}
\label{sec:uniform-dtc}

In this section we prove our first main result that uniform diffusion can achieve query complexity scaling with the dual total correlation of the distribution, in analogy to what was previously shown for masked diffusion~\cite{chen2025optimal}.

\begin{theorem}
\label{thm:uniform-dtc-sampler}
Let $q$ be any distribution on $[S]^d$, and suppose that score estimates $(\widehat{s}_t)$ satisfying $\epsilon_{\sf unif}(t) \le \epsilon_{\sf unif}$ for all $t$, along with a number $\overline{\DTC}\geq\DTC(q)$, are available. For every $0<\epsilon<1$, there is a sampler whose output law $\widehat q$ satisfies
\begin{equation}
  \KL{q}{\widehat q}\le \epsilon + O(\epsilon_{\sf unif} \cdot \log(dS/\epsilon))\,,
\end{equation}
and which uses
\begin{equation}
  O\!\left(
  \left(1+\frac{\overline{\DTC}}{\epsilon}\right)
  \log\frac{edS}{\epsilon}\right)
\end{equation}
score oracle queries.
\end{theorem}

\noindent In Section~\ref{sec:tc-dtc-uniform}, we prove the key estimate bounding the conditional TC in terms of the amount by which the DTC decreases when going from a lower noise level to a higher noise level, enabling a telescoping argument. In Section~\ref{sec:approx_score_uniform}, we control the effect of score error. Finally, in Section~\ref{sec:uniform_sampler}, we describe the sampling algorithm and complete the proof of Theorem~\ref{thm:uniform-dtc-sampler}.

\subsection{Conditional TC versus DTC decrement}
\label{sec:tc-dtc-uniform}

The main result of this section is an upper bound on the conditional TC, $\TC(X_s\mid X_t)$, in terms of the amount by which the dual total correlation decreases when going from noise level $s$ to a higher noise level $t$:

\begin{lemma}\label{lem:uniform-tc-dtc}
    For $0 < s < t$, let $h = t - s$. Then
    \begin{equation}
      \TC(X_s \mid X_t) \le \frac{e^{2h} - 1}{1 - e^{-s}} (\DTC(q_s) - \DTC(q_t))\,.
    \end{equation}
\end{lemma}

\noindent We first make a simple but key observation about the Gibbs sampling marginals which will drive the proof of the lemma above:

\begin{proposition}\label{prop:floor}
    For $0 < s \le t$ and for all $i\in[d], a\in\Sigma, y\in \calX$,
    \begin{equation}
        \Pr{(X_s)_i = a \mid (X_t)_{-i} = y_{-i}} \ge \beta_s\,.
    \end{equation}
\end{proposition}

\begin{proof}
    If in addition to $(X_t)_{-i} = y_{-i}$, one also conditions on all of $X_0$, then $(X_s)_i$ and $(X_t)_{-i}$ become conditionally independent. But conditioned on $X_0$, the probability that $(X_s)_i = a$ is, by definition of the transition kernel $K_s^{\sf U}$, lower bounded by $\beta_s$. Averaging this bound over the conditional law of $X_0$ given $(X_t)_{-i} = y_{-i}$, we prove the claim.
\end{proof}

\noindent We next establish the following helper lemma, motivated by the observation above. It shows that if two distributions over $\Sigma$ place some nonzero amount of mass on every element, then applying a small amount of noise does not decrease the KL divergence between them by too much, yielding a reverse data-processing inequality (DPI) for uniform diffusion on a single token.

\begin{lemma}[Reverse DPI for uniform diffusion]\label{lem:reversedpi}
  Let $p,p'$ be distributions over $\Sigma$ such that $\min_{a\in \Sigma} \min\{p_a, p'_a\} \ge c/S$ for some $c > 0$. Then
  \begin{equation}
    \KL{pK^{\sf U}_t}{p'K_t^{\sf U}} \ge \KL{p}{p'}\cdot \frac{e^{-2t}c}{1 - e^{-t} + e^{-t}c}\,.
  \end{equation}
\end{lemma}

\begin{proof}
  Because for any $x,y>0$, $x\log(x/y)-x+y=(x-y)^2\int^1_0 \frac{1-\theta}{y+\theta(x-y)}\,\d\theta$, we have
  \begin{equation}
    \KL{p}{p'} = \sum_a (p_a - p'_a)^2 \int^1_0 \frac{1 - \theta}{p'_a + \theta(p_a - p'_a)}\,\d\theta\,. \label{eq:bregman}
  \end{equation}
  By assumption, for all $a\in\Sigma$, $p'_a + \theta(p_a - p'_a) \ge c/S$ for all $\theta\in[0,1]$, and thus
  \begin{equation}
    e^{-t}(p'_a + \theta(p_a - p'_a)) + \frac{1 - e^{-t}}{S} \le \Bigl(e^{-t} + \frac{1 - e^{-t}}{c}\Bigr)(p'_a + \theta(p_a - p'_a))\,.
  \end{equation}
  Additionally, $((pK^{\sf U}_t)_a-(p'K_t^{\sf U})_a)^2=e^{-2t}(p_a-p'_a)^2$. Applying the same identity that gives rise to Eq.~\eqref{eq:bregman}, we conclude that
  \begin{equation}
    \KL{pK^{\sf U}_t}{p'K_t^{\sf U}} \ge \frac{e^{-2t}}{e^{-t} + (1 - e^{-t})/c}\KL{p}{p'}\,,
  \end{equation}
  as desired.
\end{proof}

\noindent We can use the above estimate to compare two quantities: (1) the mutual information between a coordinate $i$ of $X_s$ and some other coordinates of $X_s$, conditioned on the sequence $X_t$ at a higher noise level $t$, and (2) the same quantity except $(X_s)_i$ is replaced with $(X_t)_i$ and only the remaining coordinates of $X_t$ are conditioned upon. See Figure~\ref{fig:mi-comparison} for an illustration.

\begin{figure}[t]
  \centering
  \definecolor{coordblue}{RGB}{43,82,134}
\definecolor{blockred}{RGB}{158,53,51}
\begin{tikzpicture}[
  x=1cm, y=1cm,
  cell/.style={rectangle, draw=black!45, line width=0.4pt, fill=white,
               inner sep=0pt, minimum size=0.36cm},
  cond/.style={cell, fill=black!12},
  coord/.style={cell, draw=coordblue, fill=coordblue!15, line width=0.7pt},
  block/.style={cell, draw=blockred, fill=blockred!12, line width=0.7pt},
  link/.style={stealth-stealth, black!60, line width=0.55pt,
               shorten <=1pt, shorten >=1pt},
  rowlab/.style={anchor=east, inner sep=1pt},
  idx/.style={black!55, font=\scriptsize, inner sep=1pt},
  tag/.style={black!55, font=\footnotesize},
]
\def\pitch{0.44}
\def\yt{1.15}
\def\ys{0}
\begin{scope}
  \foreach \c in {1,...,8} \node[cond] (t\c) at (\c*\pitch,\yt) {};
  \foreach \c in {1,2}     \node[cell] (s\c) at (\c*\pitch,\ys) {};
  \node[coord] (s3) at (3*\pitch,\ys) {};
  \foreach \c in {4,...,8} \node[block] (s\c) at (\c*\pitch,\ys) {};
  \node[rowlab] at (0.2,\yt) {$X_t$};
  \node[rowlab] at (0.2,\ys) {$X_s$};
  \draw[-stealth, black!55, line width=0.5pt] (-0.62,\ys) -- (-0.62,\yt);
  \node[black!55, font=\footnotesize, rotate=90, anchor=center] at (-0.82,0.575) {noise};
  \node[idx] at (1*\pitch,-0.38) {$1$};
  \node[idx] at (3*\pitch,-0.38) {$i$};
  \node[idx] at (8*\pitch,-0.38) {$d$};
  \draw[link] (s3.north) .. controls +(0.15,0.55) and +(-0.15,0.55) .. (s6.north);
  \node[tag] at (1.98,1.6) {(1)};
  \node at (1.98,-0.95)
    {\small $I\bigl({\color{blockred}(X_s)_{i+1:d}}\,;\,{\color{coordblue}(X_s)_i}\mid X_t\bigr)$};
\end{scope}
\node at (4.8,0.575) {\small $\le\;\dfrac{e^{2h}-1}{1-e^{-s}}\,\cdot$};
\begin{scope}[xshift=6.5cm]
  \foreach \c in {1,2,4,5,...,8} \node[cond] (t\c) at (\c*\pitch,\yt) {};
  \node[coord] (t3) at (3*\pitch,\yt) {};
  \foreach \c in {1,2,3}   \node[cell] (s\c) at (\c*\pitch,\ys) {};
  \foreach \c in {4,...,8} \node[block] (s\c) at (\c*\pitch,\ys) {};
  \node[rowlab] at (0.2,\yt) {$X_t$};
  \node[rowlab] at (0.2,\ys) {$X_s$};
  \node[idx] at (1*\pitch,-0.38) {$1$};
  \node[idx] at (3*\pitch,-0.38) {$i$};
  \node[idx] at (8*\pitch,-0.38) {$d$};
  \draw[link] (t3.south) .. controls +(0.05,-0.5) and +(-0.12,0.5) .. (s6.north);
  \node[tag] at (1.98,1.6) {(2)};
  \node at (1.98,-0.95)
    {\small $I\bigl({\color{blockred}(X_s)_{i+1:d}}\,;\,{\color{coordblue}(X_t)_i}\mid (X_t)_{-i}\bigr)$};
\end{scope}
\end{tikzpicture}
  \caption{The two conditional mutual informations compared in
  Corollary~\ref{cor:mutual}, for a sequence of $d$ tokens at noise levels
  $s<t$. Gray cells are conditioned on, white cells are marginalized out,
  and the arc joins the two arguments of the mutual information. In (2),
  the coordinate $(X_s)_i$ is replaced by $(X_t)_i$ and only the remaining
  coordinates $(X_t)_{-i}$ are conditioned on.}
  \label{fig:mi-comparison}
\end{figure}
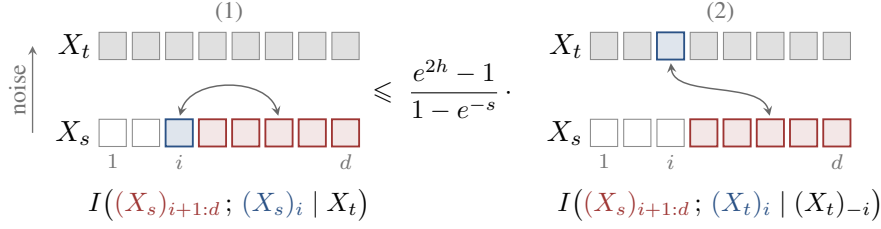

\begin{corollary}\label{cor:mutual}
  For $0 < s < t$, let $h = t - s$. Then for every $i\in[d-1]$,
  \begin{equation}
    I((X_s)_{i+1:d}; (X_s)_i \mid X_t) \le I((X_s)_{i+1:d}; (X_t)_i \mid (X_t)_{-i})\cdot \frac{e^{2h}-1}{1 - e^{-s}}\,. \label{eq:reverseDPI}
  \end{equation}
\end{corollary}

\begin{proof}
  Conditioned on $(X_t)_{-i}$, we have a Markov chain
  \begin{equation}
    (X_s)_{i+1:d} - (X_s)_i - (X_t)_i\,,
  \end{equation}
  as $(X_t)_i$ is generated from $(X_s)_i$ by applying $K_h^{\sf U}$ with randomness independent of $((X_s)_{i+1:d}, (X_t)_{-i})$.
  Note that $I(A;B \mid C) = \EE_{A,C} \KL{\law(B\mid A,C)}{\law(B\mid C)}$, so
  \begin{align}
    I((X_s)_{i+1:d}; (X_t)_i \mid (X_t)_{-i})
    &= \EE \;\KL{\law((X_t)_i\mid (X_s)_{i+1:d}, (X_t)_{-i})}
       {\law((X_t)_i\mid (X_t)_{-i})}, \label{eq:I1}\\
    I((X_s)_{i+1:d}; (X_s)_i \mid (X_t)_{-i})
    &= \EE\;\KL{\law((X_s)_i\mid (X_s)_{i+1:d}, (X_t)_{-i})}
       {\law((X_s)_i\mid (X_t)_{-i})}\,. \label{eq:I2}
  \end{align}
    Conditioned on any realization of $(X_s)_{i+1:d}$ and $(X_t)_{-i}$, the conditional distribution on $(X_s)_i$ is given by the mixture
  \begin{equation}
    \law((X_s)_i \mid (X_s)_{i+1:d}, (X_t)_{-i}) = e^{-s} \law((X_0)_i \mid (X_s)_{i+1:d}, (X_t)_{-i}) + \frac{1 - e^{-s}}{S}\,,
  \end{equation}
  so if we take $p$ and $p'$ in Lemma~\ref{lem:reversedpi} to be $\law((X_s)_i\mid (X_s)_{i+1:d}, (X_t)_{-i})$ and $\law((X_s)_i \mid (X_t)_{-i})$, and if we take $t$ therein to be $h = t - s$, then by the display above and Proposition~\ref{prop:floor}, the condition in the lemma is satisfied with $c = 1 - e^{-s}$, and we conclude from this and Eqs.~\eqref{eq:I1} and~\eqref{eq:I2} after averaging over all conditionings that
  \begin{equation}
    I((X_s)_{i+1:d}; (X_t)_i \mid (X_t)_{-i}) \ge I((X_s)_{i+1:d}; (X_s)_i \mid (X_t)_{-i})\cdot \frac{e^{-2h}(1 - e^{-s})}{1 - e^{-h} + e^{-h}(1 - e^{-s})}\,. \label{eq:intermed_I}
  \end{equation}
  Finally, by the Markov chain property and then by chain rule,
  \begin{align}
    I((X_s)_{i+1:d}; (X_s)_i \mid (X_t)_{-i}) &= I((X_s)_{i+1:d}; (X_s)_i, (X_t)_i \mid (X_t)_{-i}) \\
    &= I((X_s)_{i+1:d}; (X_s)_i \mid X_t)
      + I((X_s)_{i+1:d}; (X_t)_i\mid (X_t)_{-i})\,. \label{eq:intermed_I2}
  \end{align}
  Note that the second term on the right-hand side is the same as the left-hand side of Eq.~\eqref{eq:intermed_I}. Denoting the factor on the right-hand side of Eq.~\eqref{eq:intermed_I} by $C$, if we substitute Eq.~\eqref{eq:intermed_I2} into Eq.~\eqref{eq:intermed_I} and rearrange, we conclude that
  \begin{equation}
    I((X_s)_{i+1:d}; (X_s)_i \mid X_t) \le (1/C - 1)\cdot I((X_s)_{i+1:d}; (X_t)_i \mid (X_t)_{-i})\,.
  \end{equation}
  Finally, we have
  \begin{equation}
    1/C - 1 = \frac{(1 - e^{-h})(1 + e^{-h}(1-e^{-s}))}{e^{-2h}(1-e^{-s})} \le \frac{e^{2h}-1}{1 - e^{-s}}\,, \label{eq:1C}
  \end{equation}
  as claimed.
\end{proof}

\noindent This corollary allows us to establish our key result relating the conditional total correlation $\TC(X_s \mid X_t)$ (in which the terms on the left-hand side of Eq.~\eqref{eq:reverseDPI} appear) to the dual total correlation decrement in Eq.~\eqref{eq:decrement} (in which the terms on the right-hand side of Eq.~\eqref{eq:reverseDPI} appear).

\begin{proof}[Proof of Lemma~\ref{lem:uniform-tc-dtc}]
  By the definition of $\TC$ and Corollary~\ref{cor:mutual},
  \begin{equation}
    \TC(X_s \mid X_t)
    = \sum^{d-1}_{i=1} I((X_s)_i; (X_s)_{i+1:d} \mid X_t)
    \le \frac{e^{2h} - 1}{1 - e^{-s}}
    \sum^{d-1}_{i=1} I((X_t)_i; (X_s)_{i+1:d}\mid (X_t)_{-i})\,.
  \end{equation}
  Note that
  \begin{equation}
    I((X_t)_i; (X_s)_{i+1:d}\mid (X_t)_{-i})
    \le I((X_t)_i; (X_s)_{-i}\mid (X_t)_{-i})\,,
  \end{equation}
  and the sum of these over $i=1,\ldots,d-1$ is, by Lemma~\ref{lem:dtc_decrement}, at most $\DTC(q_s)-\DTC(q_t)$, as desired.
\end{proof}

\noindent Henceforth, for convenience, we denote the factor in
Lemma~\ref{lem:uniform-tc-dtc} by
\begin{equation}
  c_s(h) \triangleq \frac{e^{2h} - 1}{1 - e^{-s}}\,,
\end{equation}
so that
\begin{equation}
    \TC(X_s \mid X_t) \le c_s(h)(\DTC(q_s) - \DTC(q_t))\,.
\end{equation}

\subsection{Score error and approximate reverse kernel}
\label{sec:approx_score_uniform}

In this section, we show that the approximate score $\widehat{s}$ can be converted into approximate posterior marginals, by giving a ``robust'' version of Lemma~\ref{lem:uniform_score_to_marginals}. We then bound the distance between these approximate posterior marginals and the true posterior marginals in terms of the score entropy loss $\epsilon_{\sf unif}$.

Given $b\in\Sigma$, define
\begin{equation}
    \mathcal{S}_{t,b} \triangleq \Bigl\{v\in \mathbb{R}_{> 0}^S: v_b = 1 \ \ \text{and} \ \ \frac{v_a}{\sum_{a'\in\Sigma} v_{a'}} \ge \beta_t \ \ \forall \ a \in \Sigma\Bigr\}\,.
\end{equation}
As the $S + 1$ constraints defining this set are linear, this is a convex set. Furthermore, by the constraint corresponding to $a = b$, the set is compact.

\begin{proposition}\label{prop:robust-marginals}
    For every $t>0$, $y\in\calX$, and $i\in[d]$, the true score $(\unifscore{t}{y}{i,a})_{a\in\Sigma}$ is an element of $\mathcal{S}_{t,y_i}$.

    Furthermore, for $0<s<t$, $h = t - s$, and any $v\in \mathcal{S}_{t,b}$, the distribution $\gibbs_h[v]$ over $\Sigma$ defined in Lemma~\ref{lem:uniform_score_to_marginals} is a valid probability distribution and satisfies $\gibbs_h[v](a) \ge \beta_s$ for all $a\in\Sigma$. Likewise, $\post_{h,b}[v](a)=\bigl(\sum_{a'\in\Sigma}v_{a'}\bigr)K_h^{\sf U}(b\mid a)\,\gibbs_h[v](a)$ is a valid probability distribution.
\end{proposition}

\noindent We defer the proof to Appendix~\ref{app:deferred-proofs}.

Let $\mathrm{proj}_{t,b}: \mathbb{R}^S_{> 0} \to \mathbb{R}^S_{> 0}$ denote the Bregman projection
\begin{equation}
    \mathrm{proj}_{t,b}(v) \triangleq \arg\min_{v'\in \mathcal{S}_{t,b}} \sum_{a\in \Sigma}\psi(v'_a, v_a)\,.
\end{equation}
This is uniquely defined because $\mathcal{S}_{t,b}$ is convex and compact and $\psi$ is strictly convex in its first argument. By the Pythagorean theorem for Bregman divergences, for any $v^*\in\mathcal{S}_{t,b}$ and any $v\in \mathbb{R}_{>0}^S$,
\begin{equation}
    \sum_{a\in \Sigma} \psi(v^*_a,\mathrm{proj}_{t,b}(v)_a) \le \sum_{a\in \Sigma}\psi(v^*_a, v_a)\,. \label{eq:pythagorean}
\end{equation}
Given score estimates $(\widehat{s}_t)$, define their projections $(\widehat{s}^\sharp_t)$ by
\begin{equation}
    \widehat{s}^\sharp_t(y)[i,\cdot] \triangleq \mathrm{proj}_{t,y_i}(\widehat{s}_{t}(y)[i,\cdot])\,.
\end{equation}

For $0 < s < t$ and $h = t -s$, we will consider the approximate reverse kernel
\begin{equation}
    \approxrevu_{t\to s}(y,x) \triangleq \prod^d_{i=1} \widehat{R}^i_{t\to s}(y,x_i)\,, \qquad \widehat{R}^i_{t\to s}(y,\cdot) \triangleq \post_{h,y_i}[\widehat{s}^\sharp_t(y)[i,\cdot]]\,.
\end{equation}
We will compare this against the true product reverse kernel:
\begin{equation}
  \prodrevu_{t\to s}(y,x)
  \triangleq \prod_{i=1}^d R^i_{t\to s}(y,x_i)\,, \qquad R^i_{t\to s}(y,\cdot) \triangleq \post_{h,y_i}[\unifscore{t}{y}{i,\cdot}] = \Pr{(X_s)_i=\cdot\mid X_t=y}\,.
  \label{eq:product-reverse-kernel}
\end{equation}

\begin{lemma}\label{lem:uniform_score_error}
    For $0<s<t$ and $h=t-s$,
    \begin{equation}
        \EE_{Y\sim q_t}\KL{\prodrevu_{t\to s}(Y,\cdot)}{\approxrevu_{t\to s}(Y,\cdot)} \le c_s(h)\epsilon_{\sf unif}(t)\,.
    \end{equation}
\end{lemma}

\begin{proof}
    Fix any $y\in\calX$ and $i\in[d]$ and condition on $(X_t)_{-i} = y_{-i}$. For convenience, denote by $p$ the conditional law of $(X_s)_i$, and also define the corresponding approximation $\widehat{p} = \gibbs_h[\widehat{s}^\sharp_t(y)[i,\cdot]]$ given by the score estimate. Additionally, given any $b\in\Sigma$, denote by $p_{\mid b}$ the conditional law of $(X_s)_i$ upon further conditioning on $(X_t)_i = b$. Note that
    \begin{equation}
        p_{\mid y_i} = R^i_{t\to s}(y,\cdot)\,,
    \end{equation}
    so we would like to bound $\KL{p_{\mid y_i}}{\widehat{p}_{\mid y_i}}$ for a suitable approximation $\widehat{p}_{\mid y_i}$ defined below.
    Finally, let $r = pK^{\sf U}_h$ and $\widehat{r} = \widehat{p}K_h^{\sf U}$, noting that $r$ is the conditional law of $(X_t)_i$ given $(X_t)_{-i} = y_{-i}$, and $\widehat{r}$ is an approximation thereof, so that
    \begin{equation}
        \unifscore{t}{y}{i,a} = \frac{r(a)}{r(y_i)} \qquad \text{and} \qquad \widehat{s}^\sharp_t(y)[i,a] = \frac{\widehat{r}(a)}{\widehat{r}(y_i)}\,.
    \end{equation}
    By chain rule applied to the joint conditional law of $(X_s)_i, (X_t)_i$ and its approximation,
    \begin{equation}
        \KL{p}{\widehat{p}} = \KL{r}{\widehat{r}} + \EE_{b\sim r} \KL{p_{\mid b}}{\widehat{p}_{\mid b}}\,, \label{eq:apply_chain}
    \end{equation}
    where $\widehat{p}_{\mid b}(\cdot) \triangleq \widehat{p}(\cdot)K_h^{\sf U}(b\mid \cdot)/\widehat{r}(b)$ is the posterior law of a sample from $\widehat{p}$ given that its image under the channel $K_h^{\sf U}$ equals $b$. Writing $v = \widehat{s}^\sharp_t(y)[i,\cdot]$, we have $\widehat{r}(a) = \alpha_h\widehat{p}(a) + \beta_h = v_a/\sum_{a'\in\Sigma}v_{a'}$, and substituting this into the definition of $\widehat{p}_{\mid y_i}$ and using $v_{y_i} = 1$ gives
    \begin{equation}
        \widehat{p}_{\mid y_i} = \post_{h,y_i}[v] = \widehat{R}^i_{t\to s}(y,\cdot)\,.
    \end{equation}
    We will lower bound the expectation on the right-hand side of Eq.~\eqref{eq:apply_chain} by the contribution from $b = y_i$ to get
    \begin{equation}
        \KL{p}{\widehat{p}} \ge \KL{r}{\widehat{r}} + r(y_i)\cdot \KL{p_{\mid y_i}}{\widehat{p}_{\mid y_i}}\,.
    \end{equation}
    By the reverse DPI in Lemma~\ref{lem:reversedpi}, whose hypothesis holds with $c=1-e^{-s}$ by Propositions~\ref{prop:floor} and~\ref{prop:robust-marginals},
    \begin{equation}
        \KL{r}{\widehat{r}} \ge \KL{p}{\widehat{p}} \cdot \frac{e^{-2h}(1 - e^{-s})}{1 - e^{-h} + e^{-h}(1 - e^{-s})}\,.
    \end{equation}
    Recall that this factor is the same one appearing in Eq.~\eqref{eq:intermed_I}, which was denoted in the proof of Corollary~\ref{cor:mutual} by $C$ and which satisfies $1/C - 1\le c_s(h)$ by Eq.~\eqref{eq:1C}.
    Rearranging, we have
    \begin{equation}
        \KL{p_{\mid y_i}}{\widehat{p}_{\mid y_i}} \le \frac{c_s(h)}{r(y_i)}\KL{r}{\widehat{r}}\,.
    \end{equation}
    We have
    \begin{align}
        \sum_a\psi(\unifscore{t}{y}{i,a},\widehat{s}^\sharp_t(y)[i,a]) &= \frac{1}{\widehat{r}(y_i)} - \frac{1}{r(y_i)} + \sum_{a\in \Sigma} \frac{r(a)}{r(y_i)} \log\Bigl(\frac{r(a)}{\widehat{r}(a)} \cdot \frac{\widehat{r}(y_i)}{r(y_i)}\Bigr) \\
        &= \psi\Bigl(\frac{1}{r(y_i)}, \frac{1}{\widehat{r}(y_i)}\Bigr) + \frac{1}{r(y_i)}\KL{r}{\widehat{r}} \\
        &\ge \frac{1}{r(y_i)}\KL{r}{\widehat{r}}\,,
    \end{align}
    so
    \begin{equation}
        \KL{p_{\mid y_i}}{\widehat{p}_{\mid y_i}} \le c_s(h) \sum_a\psi(\unifscore{t}{y}{i,a},\widehat{s}^\sharp_t(y)[i,a]) \le c_s(h) \sum_a \psi(\unifscore{t}{y}{i,a}, \widehat{s}_t(y)[i,a])\,,
    \end{equation}
    where in the last step we used the Pythagorean theorem in Eq.~\eqref{eq:pythagorean}.
    Summing over coordinates $i$, averaging over $y\sim q_t$, and recalling Definition~\ref{def:uniform-score-error}, we conclude the claimed bound.
\end{proof}

\subsection{Sampler and telescoping argument}
\label{sec:uniform_sampler}

We consider a sampler of the following form. Fix a grid of times
\begin{equation}
  0 < t_0 < t_1 < \cdots < t_M\,.
\end{equation}
\begin{enumerate}
  \item Initialize at $\Unif(\calX)$ at time $t_M$.
  \item For each $j = M - 1,\ldots,0$:
  \begin{itemize}
    \item Apply $\approxrevu_{t_{j+1}\to t_j}$ to get the next iterate.
  \end{itemize}
  \item Given the iterate $Y$ at time $t_0$, output a draw from $(K^{\sf U}_{t_0})^{\otimes d}(\cdot\mid Y)$.
\end{enumerate}

\noindent The following lemma provides a straightforward chain rule calculation that utilizes the main estimates from the previous subsections (Lemma~\ref{lem:uniform-tc-dtc} and Lemma~\ref{lem:uniform_score_error}).

\begin{lemma}\label{lem:chain}
  Let $\widehat{q}$ denote the output law of the above sampler. Then
  \begin{equation}
    \KL{q}{\widehat{q}} \le \KL{q_{t_M}}{\Unif(\calX)} + H(q_{t_0}) - H(q) + \sum^{M-1}_{j=0} c_{t_j}(t_{j+1} - t_j)\cdot \Bigl[\DTC(q_{t_j}) - \DTC(q_{t_{j+1}}) + \epsilon_{\sf unif}(t_{j+1})\Bigr]\,.
  \end{equation}
\end{lemma}

\begin{proof}
  Let $P$ be the joint law of $(X_{t_M},X_{t_{M-1}},\ldots,X_{t_0},X_0)$ under the forward process, written in reverse order. Let $\widehat P$ be the comparison path law that initializes its first coordinate from $\Unif(\calX)$, uses $\approxrevu_{t_{j+1}\to t_j}$ from $t_{j+1}$ to $t_j$, and uses $(K_{t_0}^{\sf U})^{\otimes d}(\cdot\mid X_{t_0})$ for the last transition. For $0\leq j<M$, define
  \begin{align}
    E^{\rm disc}_j&\triangleq\EE_{X_{t_{j+1}}}\KL{\law(X_{t_j}\mid X_{t_{j+1}})}{\prodrevu_{t_{j+1}\to t_j}(X_{t_{j+1}},\cdot)},\label{eq:chain-middle}\\
    E^{\rm sc}_j&\triangleq\EE_{X_{t_{j+1}}}\KL{{\prodrevu_{t_{j+1}\to t_j}(X_{t_{j+1}},\cdot)}}{{\approxrevu_{t_{j+1}\to t_j}(X_{t_{j+1}},\cdot)}},\label{eq:score_err_unif}\\
    E_0^{\rm end}&\triangleq\EE_{X_{t_0}}\KL{\law(X_0\mid X_{t_0})}{(K_{t_0}^{\sf U})^{\otimes d}(\cdot\mid X_{t_0})}\,.
  \end{align}
  Because the marginals of the product reverse kernel $\prodrevu_{t_{j+1}\to t_j}$ are by definition the marginals of $\mathrm{law}(X_{t_j}\mid X_{t_{j+1}} = y)$,
  \begin{equation}
    \EE_{X_{t_{j+1}}}\;\KL{\law(X_{t_j}\mid X_{t_{j+1}})}{\approxrevu_{t_{j+1}\to t_j}(X_{t_{j+1}},\cdot)} = E^{\rm disc}_j + E^{\rm sc}_j\,.
  \end{equation}
  The $X_0$ marginal of $P$ is $q$, while the last-coordinate marginal of $\widehat P$ is $\widehat q$. Data processing and chain rule for KL therefore give
  \begin{align}
    \KL{q}{\widehat q}
    &\leq \KL{P}{\widehat P}\notag\\
    &=\KL{q_{t_M}}{\Unif(\calX)}+\sum_{j=0}^{M-1}(E^{\rm disc}_j + E^{\rm sc}_j) +E_0^{\rm end}\,.
  \end{align}
  By Lemma~\ref{lem:entropy_diff} below, the last term is equal to $H(q_{t_0})-H(q)$. Furthermore, by definition $E^{\rm disc}_j = \TC(X_{t_j}\mid X_{t_{j+1}})$, so by Lemma~\ref{lem:uniform-tc-dtc},
  \begin{equation}
    E^{\rm disc}_j \le c_{t_j}(t_{j+1}-t_j)\bigl(\DTC(q_{t_j})-\DTC(q_{t_{j+1}})\bigr)\,.
  \end{equation}
  Finally, by Lemma~\ref{lem:uniform_score_error},
  \begin{equation}
    E^{\rm sc}_j \le c_{t_j}(t_{j+1} - t_j) \epsilon_{\sf unif}(t_{j+1})\,.
  \end{equation}
  Substituting all of these bounds proves the claim.
\end{proof}

\noindent The lemma below was used above to control the error incurred by the rounding step at the end of the sampler.

\begin{lemma}\label{lem:entropy_diff}
  For any $t > 0$, we have
  \begin{equation}
    \EE_{X_t\sim q_t}
    \KL{\law(X_0\mid X_t)}{(K^{\sf U}_t)^{\otimes d}(\cdot\mid X_t)}
    = H(q_t) - H(q)\,.
  \end{equation}
\end{lemma}

\begin{proof}
  The joint law of $(X_0,X_t)$ is $q(x)(K^{\sf U}_t)^{\otimes d}(y\mid x)$, while the comparison joint law is $q_t(y)(K^{\sf U}_t)^{\otimes d}(x\mid y)$. The kernel is symmetric, so their log-likelihood ratio is $\log q(x)-\log q_t(y)$. Averaging gives $-H(q)+H(q_t)$.
\end{proof}

\noindent We next control the two endpoint terms in
Lemma~\ref{lem:chain} by taking $t_0$ sufficiently small and $t_M$ sufficiently large.

\begin{lemma}\label{lem:forward_decay}
  For any $T > 0$, $\KL{q_T}{\Unif(\calX)} \le e^{-T}d\log S$.
\end{lemma}

\noindent We defer the proof to Appendix~\ref{app:deferred-proofs}.

\begin{lemma}\label{lem:shorttime}
  For $0 < t \le 1$, $H(q_t) - H(q) \le dt\log(eS/t)$.
\end{lemma}

\noindent We defer the proof to Appendix~\ref{app:deferred-proofs}.

We are now ready to prove our main result, a DTC-adaptive query complexity for uniform diffusion sampling:

\begin{proof}[Proof of Theorem~\ref{thm:uniform-dtc-sampler}]
It remains to set the step sizes so that the sum in Lemma~\ref{lem:chain} telescopes (see Figure~\ref{fig:uniform-schedule}). Define
\begin{equation}
  L_{\sf U}=\log(16edS/\epsilon),\qquad
  \delta_{\sf U}=\frac{\epsilon}{16dL_{\sf U}},\qquad
  T_{\sf U}=\log(4d\log(S)/\epsilon)\,,
\end{equation}
and
\begin{equation}
  a_{\sf U}
  =\frac{\epsilon}
  {12\max\{\overline{\DTC},\epsilon\}},\qquad
  u_0=e^{\delta_{\sf U}}-1,\qquad U=e^{T_{\sf U}}-1\,,
\end{equation}
noting that $a_{\sf U}\le 1/12$.
Recursively set
\begin{equation}
  u_{j+1} = \min((1 + a_{\sf U})u_j, U)\,,
\end{equation}
and let $N_{\sf U}$ be the first index for which $u_j=U$. Set $t_j=\log(1+u_j)$, so that $\delta_{\sf U}=t_0<\cdots<t_{N_{\sf U}}=T_{\sf U}$. Because $u_j/(1+u_j)=1-e^{-t_j}$, we have
\begin{equation}
  e^{t_{j+1} - t_j} = \frac{1 + u_{j+1}}{1+u_j} \le \frac{1 + (1 + a_{\sf U})u_j}{1 + u_j} = 1 + a_{\sf U}(1 - e^{-t_j})\,,
\end{equation}
and hence, writing $x_j \triangleq a_{\sf U}(1 - e^{-t_j}) \le 1$,
\begin{equation}
  c_{t_j}(t_{j+1} - t_j) = \frac{e^{2(t_{j+1} - t_j)} - 1}{1 - e^{-t_j}} \le \frac{(1 + x_j)^2 - 1}{1 - e^{-t_j}} = \frac{(2 + x_j)x_j}{1 - e^{-t_j}} \le 3a_{\sf U}\,.
\end{equation}
Therefore,
\begin{align}
  \sum^{N_{\sf U}-1}_{j=0} c_{t_j}(t_{j+1}-t_j)
  \bigl(\DTC(q_{t_j}) - \DTC(q_{t_{j+1}}) + \epsilon_{\sf unif}(t_{j+1})\bigr) &\leq 3a_{\sf U}\DTC(q) + 3a_{\sf U}N_{\sf U}\epsilon_{\sf unif} \\
  &\le \epsilon/4 + 3a_{\sf U}N_{\sf U} \epsilon_{\sf unif}\,,
\end{align}
since $\DTC(q_t)$ is nonincreasing in $t$ by Lemma~\ref{lem:dtc_decrement}, so the increments are nonnegative and telescope.

\begin{figure}[t]
  \centering
  \definecolor{coordblue}{RGB}{43,82,134}
\definecolor{blockred}{RGB}{158,53,51}
\begin{tikzpicture}[
  x=1cm, y=1cm,
  axis/.style={black!45, line width=0.4pt},
  map/.style={black!25, line width=0.45pt},
  gridpt/.style={draw=coordblue, fill=coordblue!15, line width=0.7pt},
  endpt/.style={draw=blockred, fill=blockred!12, line width=0.7pt},
  steparc/.style={-stealth, black!60, line width=0.55pt,
               shorten <=1pt, shorten >=1pt},
  rowlab/.style={anchor=east, inner sep=1pt},
  idx/.style={black!55, font=\scriptsize, inner sep=1pt},
  tag/.style={black!55, font=\footnotesize},
]
\def\yt{1.8}
\def\ys{0}
\foreach \xt/\xb in {0/0.080, 1.214/0.158, 2.428/0.311, 3.642/0.599,
                     4.856/1.120, 6.071/1.994, 7.285/3.319, 8.499/5.108,
                     9.713/7.280, 10.927/9.714, 12.0/12.0}
  \draw[map] (\xt,\yt-0.09) -- (\xb,\ys+0.09);
\draw[axis] (-0.15,\yt) -- (12.15,\yt);
\draw[axis] (-0.15,\ys) -- (12.15,\ys);
\foreach \xt/\xb in {1.214/0.158, 2.428/0.311, 3.642/0.599, 4.856/1.120,
                     6.071/1.994, 7.285/3.319, 8.499/5.108, 9.713/7.280,
                     10.927/9.714}{
  \fill[gridpt] (\xt,\yt) circle (0.058);
  \fill[gridpt] (\xb,\ys) circle (0.058);
}
\foreach \xt/\xb in {0/0.080, 12.0/12.0}{
  \fill[endpt] (\xt,\yt) circle (0.058);
  \fill[endpt] (\xb,\ys) circle (0.058);
}
\node[rowlab] at (-0.35,\yt)
  {\shortstack{$u_j$\\[-1pt]{\scriptsize\color{black!55}(log scale)}}};
\node[rowlab] at (-0.35,\ys) {$t_j$};
\draw[-stealth, black!55, line width=0.5pt] (-1.85,\yt) -- (-1.85,\ys);
\node[black!55, font=\footnotesize, rotate=90, anchor=center]
  at (-2.07,0.9) {$t=\log(1+u)$};
\node[idx] at (0,\yt+0.33) {$u_0$};
\node[idx] at (12,\yt+0.33) {$U$};
\node[idx] at (0.08,\ys-0.30) {$\delta_{\sf U}$};
\node[idx] at (12,\ys-0.30) {$T_{\sf U}$};
\draw[steparc] (3.642,\yt+0.10) .. controls +(0.15,0.32) and +(-0.15,0.32)
  .. (4.856,\yt+0.10);
\node[tag] at (4.25,\yt+0.62) {$\times(1+a_{\sf U})$};
\draw[steparc] (7.280,\ys-0.10) .. controls +(0.3,-0.34) and +(-0.3,-0.34)
  .. (9.714,\ys-0.10);
\node[tag] at (8.5,\ys-0.72) {$t_{j+1}-t_j\approx\log(1+a_{\sf U})$};
\node[tag] at (1.9,\ys-0.72) {$t_{j+1}\approx(1+a_{\sf U})\,t_j$};
\end{tikzpicture}
  \caption{Step schedule for uniform diffusion sampler}
  \label{fig:uniform-schedule}
\end{figure}
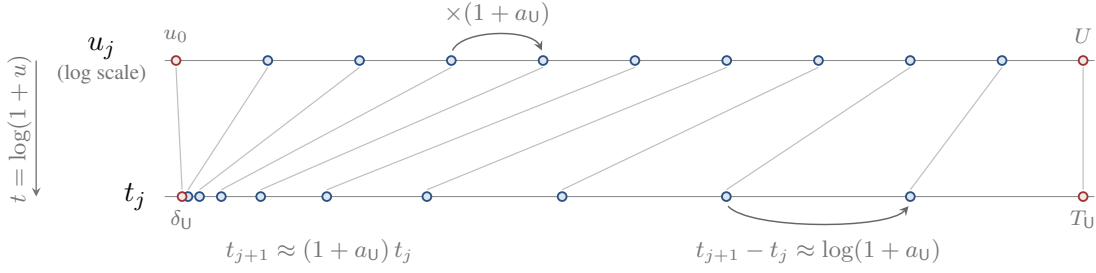

By Lemma~\ref{lem:forward_decay} and the choice of $T_{\sf U}$, $\KL{q_{T_{\sf U}}}{\Unif(\calX)}\leq\epsilon/4$. By Lemma~\ref{lem:shorttime} and the choice of $\delta_{\sf U}$, $H(q_{\delta_{\sf U}})-H(q)\le d\delta_{\sf U}\log(eS/\delta_{\sf U})=\frac{\epsilon}{16}\bigl(1+\frac{\log L_{\sf U}}{L_{\sf U}}\bigr)\le\epsilon/8$.

Finally, $u_j$ increases by a factor $1+a_{\sf U}$ until the last step, so
\begin{equation}
  N_{\sf U}\leq1+\frac{\log(U/u_0)}{\log(1+a_{\sf U})}\,.
\end{equation}
Here $\log(U/u_0)\lesssim L_{\sf U}$ and $1/\log(1+a_{\sf U})\leq2/a_{\sf U}\lesssim 1+\overline{\DTC}/\epsilon$. Each application of $\approxrevu_{t_{j+1}\to t_j}$ requires only the score matrix $\widehat{s}_{t_{j+1}}(y)\in\mathbb{R}^{[d]\times\Sigma}$ at the current iterate $y$, i.e., one score oracle query, so the sampler makes $N_{\sf U}$ queries in total, which proves the query bound.

Substituting all of the above bounds into Lemma~\ref{lem:chain} gives the claimed KL bound.
\end{proof}

\section{Gaussian diffusion can scale with dual total correlation}
\label{sec:gaussian-dtc}

We now prove the Gaussian diffusion analogue of Theorem~\ref{thm:uniform-dtc-sampler}. Recall from Section~\ref{sec:gaussian-diffusion} that we consider the pushforward $q$ of a discrete distribution $q_{\sf pre}$ on $\calX = \Sigma^d$ to $(\RR^S)^d$ via the one-hot encoding. When we refer to the dual total correlation of $q$, we treat each block of $S = |\Sigma|$ entries as a single coordinate in the definition of $\DTC(q)$, so that $\DTC(q)=\DTC(q_{\sf pre})$.

Given $X_0\sim q$, let $(X_t)$ be the Ornstein--Uhlenbeck process of Section~\ref{sec:gaussian-diffusion}, with $q_t\triangleq\law(X_t)$ and score $s_t=\nabla\log q_t$. The error of estimated scores $(\widehat{s}_t)$ is quantified by the $L^2$ score error $\epsilon_{\sf gauss}(t)$ of Definition~\ref{def:gaussian-score-error}.

\begin{theorem}
\label{thm:gaussian-dtc-sampler}
Let $q$ be any one-hot encoding of a distribution $q_{\sf pre}$ on $\calX$, and suppose that score estimates $(\widehat{s}_t)$ satisfying $\epsilon_{\sf gauss}(t)\le\epsilon_{\sf gauss}$ for all $t$, along with a number $\overline{\DTC}\geq\DTC(q)$, are available. For every $0<\epsilon<1$, there is a sampler whose output law $\widehat q$ on $[S]^d$ satisfies
\begin{equation}
  \KL{q_{\sf pre}}{\widehat q}\le \epsilon + O(\epsilon_{\sf gauss}^2\cdot\log(dS/\epsilon))\,,
\end{equation}
and which uses
\begin{equation}
  O\!\left(
  \left(1+\frac{\overline{\DTC}}{\epsilon}\right)
  \log\frac{edS}{\epsilon}\right)
\end{equation}
score oracle queries.
\end{theorem}

\noindent The general outline of the argument remains the same as the one for uniform diffusion in Section~\ref{sec:uniform-dtc}. An important difference is that there is no analogue of Proposition~\ref{prop:floor}, so the reverse DPI used to relate conditional TC to the decrease in DTC must be proven using a different route (see Section~\ref{sec:reverse-comparison-proof}).

\paragraph{Change of variable.} Throughout this section, in place of the OU parametrization it will be more convenient to work with the additive noise parametrization. As such, we consider the change of variable
\begin{equation}
  u=e^{2t}-1\,,
  \qquad t(u)=\tfrac12\log(1+u)\,,
  \qquad Z_u\triangleq e^{t(u)}X_{t(u)}\,,
  \label{eq:categorical-additive-parameter}
\end{equation}
so that
\begin{equation}
  Z_u=X_0+B_u
  \label{eq:additive-brownian}
\end{equation}
for a standard Brownian motion $(B_u)_{u\ge0}$ in $\RR^{Sd}$ independent of $X_0$.

In lieu of $q_t$ and $s_t$, we use $\pi_u$ and $g_u = \nabla \log \pi_u$ to denote the density of $Z_u$ and its score, noting that
\begin{equation}
    g_u(z) \triangleq \frac{1}{\sqrt{1 + u}} s_{t(u)}\,\Bigl(\frac{z}{\sqrt{1+u}}\Bigr)\,.
\end{equation}
Similarly, we will consider the score estimate
\begin{equation}
  \widehat g_u(z)
  \triangleq\frac1{\sqrt{1+u}}\,\widehat s_{t(u)}\,\Bigl(\frac z{\sqrt{1+u}}\Bigr)\,,
\end{equation}
which has score error
\begin{equation}
  \EE_{Z_u\sim\pi_u}\|\widehat g_u(Z_u)-g_u(Z_u)\|_2^2
  =\frac1{1+u}\,\EE_{X_{t(u)}\sim q_{t(u)}}\|\widehat s_{t(u)}(X_{t(u)})-s_{t(u)}(X_{t(u)})\|_2^2
  =\frac{1+u}{u^2}\,\epsilon_{\sf gauss}(t(u))^2\,,
  \label{eq:categorical-additive-score-error}
\end{equation}
where in the last step we re-expressed the $\sigma^2_{t(u)}$ prefactor in the definition of $\epsilon_{\sf gauss}$ via $\sigma_{t(u)}^2=u/(1+u)$.

\subsection{Conditional TC versus DTC decrement}

The main result of this subsection is the following analogue of Lemma~\ref{lem:uniform-tc-dtc}: an upper bound on the conditional TC, $\TC(Z_u\mid Z_v)$, in terms of the amount by which the dual total correlation decreases when going from noise level $u$ to higher noise level $v$. Here and below, for laws with densities $\DTC$ is defined by the same formula with differential entropies, and Lemma~\ref{lem:dtc_decrement} applies verbatim (also when $A$ is discrete and $B$ continuous).

\begin{lemma}\label{lem:gaussian-tc-dtc}
  For $0<u<v$,
  \begin{equation}
    \TC(Z_u\mid Z_v)
    \le\bigl(\frac{v}{u}e^{3/u-3/v}-1\bigr)\bigl(\DTC(\pi_u)-\DTC(\pi_v)\bigr)\,.
  \end{equation}
\end{lemma}

\noindent As in the uniform diffusion case, the key ingredient in the proof is a reverse DPI for a single token. In the uniform diffusion case, however, the proof of the reverse DPI crucially relied on a lower bound on the mass that the relevant conditional distributions placed on each possible token. In the Gaussian setting we cannot rely on such a property.

For a distribution $\alpha=(\alpha_1,\ldots,\alpha_S)$ over $\Sigma$ and $u>0$, if we regard it as a distribution over standard basis vectors in $\RR^S$ and apply Gaussian noise at level $u$, the resulting distribution is a mixture of Gaussians with density
\begin{equation}
  f_{\alpha,u}(z) = \sum_{a\in\Sigma}\alpha_a\phi_u(z-e_a)\,,
  \label{eq:categorical-mixture}
\end{equation}
where $\phi_u$ denotes the density of $N(0,u\cdot \Id_S)$.

The main estimate we show is the following reverse DPI for such mixture distributions.

\begin{lemma}[Reverse DPI for Gaussian diffusion]
\label{lem:categorical-reverse-comparison}
For every $\alpha,\beta\in\Delta(\Sigma)$ and $0<u<v$,
\begin{equation}
  \KL{f_{\alpha,v}}{f_{\beta,v}} \ge \KL{f_{\alpha,u}}{f_{\beta,u}}\cdot \Bigl(\frac{u}{v}e^{3/v-3/u}\Bigr)\,.
  \label{eq:categorical-reverse-comparison}
\end{equation}
\end{lemma}

\noindent Lemma~\ref{lem:categorical-reverse-comparison} is a consequence of standard identities for the additive Gaussian channel, and we defer its proof to Section~\ref{sec:reverse-comparison-proof} and first derive Lemma~\ref{lem:gaussian-tc-dtc} from it.

In exact analogy with Corollary~\ref{cor:mutual} in the uniform diffusion case, we use this bound to compare two quantities: (1) the mutual information between a block $i$ of $Z_u$ and the later blocks of $Z_u$, conditioned on the sequence $Z_v$ at a higher noise level $v$, and (2) the same quantity except $Z_u^{(i)}$ is replaced with $Z_v^{(i)}$ and only the remaining blocks of $Z_v$ are conditioned upon (recall Figure~\ref{fig:mi-comparison}).

\begin{corollary}\label{cor:gaussian-mutual}
  For $0<u<v$ and $i\in[d-1]$,
  \begin{equation}
    I(Z_u^{(i+1:d)};Z_u^{(i)}\mid Z_v)
    \le \bigl(\frac{v}{u}e^{3/u-3/v}-1\bigr) \cdot I(Z_u^{(i+1:d)};Z_v^{(i)}\mid Z_v^{(-i)})\,.
    \label{eq:gaussian-reverseDPI}
  \end{equation}
\end{corollary}

\begin{proof}
  Conditionally on $(Z_u^{(i+1:d)},Z_v^{(-i)})$, the block $Z_u^{(i)}$ has law $f_{\alpha,u}$ and $Z_v^{(i)}$ has law $f_{\alpha,v}$ with $\alpha=\law(A_i\mid Z_u^{(i+1:d)},Z_v^{(-i)})$, because the noise in block $i$ is independent of the conditioning; likewise with $\beta=\law(A_i\mid Z_v^{(-i)})$ when conditioning on $Z_v^{(-i)}$ alone. Lemma~\ref{lem:categorical-reverse-comparison} therefore applies to these conditional laws, and the proof follows verbatim from the argument for Corollary~\ref{cor:mutual} in the uniform diffusion case, with the parameter $C$ defined therein taken to be $C = \frac{u}{v}e^{3/v-3/u}$ in light of Lemma~\ref{lem:categorical-reverse-comparison}.
\end{proof}

\begin{proof}[Proof of Lemma~\ref{lem:gaussian-tc-dtc}]
  The proof also follows verbatim from the argument for Lemma~\ref{lem:uniform-tc-dtc}, but where $\frac{e^{2h}-1}{1 - e^{-s}}$ therein is replaced with $\frac{v}{u}e^{3/u-3/v} - 1$.
\end{proof}

\noindent Henceforth, for convenience denote the factor in Lemma~\ref{lem:gaussian-tc-dtc} by
\begin{equation}
  c_u(v) \triangleq \frac{v}{u}e^{3/u-3/v}-1\,,
\end{equation}
so that
\begin{equation}
    \TC(Z_u \mid Z_v) \le c_u(v)(\DTC(\pi_u) - \DTC(\pi_v))\,.
\end{equation}

\subsection{Reverse DPI for Gaussian diffusion}
\label{sec:reverse-comparison-proof}

We will use the following standard identities for the additive Gaussian channel, see, e.g., \cite[Lemmas 1 and 2]{pmlr-v291-wibisono25a}:
\begin{lemma}\label{lem:heatflow}
  Let $(\mu_u)_{u\ge 0}$ and $(\nu_u)_{u\ge 0}$ be measures evolving according to the standard Gaussian channel, that is, which satisfy $\partial_u p_u = \frac{1}{2}\Delta p_u$ for $p = \mu,\nu$. Denote the log-density ratio between them by $\ell_u \triangleq \log(\mu_u / \nu_u)$, and denote the \emph{relative Fisher information} between them by
  \begin{equation}
    \FI{\mu_u}{\nu_u} \triangleq \EE_{\mu_u}\norm{\nabla \ell_u}_2^2\,.
  \end{equation}
  Then
  \begin{itemize}
    \item $\partial_u \KL{\mu_u}{\nu_u} = -\frac{1}{2}\FI{\mu_u}{\nu_u}$
    \item $\partial_u \FI{\mu_u}{\nu_u} = -\EE_{\mu_u} \|\nabla^2 \ell_u\|^2_F + 2\EE_{\mu_u} (\nabla \ell_u)^\top (\nabla^2 \log \nu_u) (\nabla \ell_u)$
  \end{itemize}
\end{lemma}

\noindent Specializing $\mu_u$ and $\nu_u$ above to $f_{\alpha,u}$ and $f_{\beta,u}$ respectively, we obtain the following:

\begin{lemma}\label{lem:gaussian_kl_fi_evolution}
  Let $\alpha,\beta\in\Delta(\Sigma)$ and $u>0$. Given $\lambda\in \Delta(\Sigma)$, define
  \begin{equation}
    \mathsf{m}_{\lambda,u}(z) \triangleq \biggl(\frac{\lambda_a e^{z_a/u}}{\sum_b \lambda_b e^{z_b/u}}\biggr)_{a\in\Sigma} \qquad \text{and} \qquad \mathsf{C}_{\lambda,u}(z) \triangleq \mathrm{diag}(\mathsf{m}_{\lambda,u}(z)) - \mathsf{m}_{\lambda,u}(z)^{\otimes 2}\,.
  \end{equation}
  Then
  \begin{itemize}
    \item $\partial_u \KL{f_{\alpha,u}}{f_{\beta,u}} = -\frac{1}{2}\FI{f_{\alpha,u}}{f_{\beta,u}}$
    \item $\partial_u \FI{f_{\alpha,u}}{f_{\beta,u}} \ge -\Bigl(\frac{2}{u} + \frac{3}{u^2}\Bigr)\FI{f_{\alpha,u}}{f_{\beta,u}}$
  \end{itemize}
\end{lemma}

\noindent We defer the proof to Appendix~\ref{app:deferred-proofs}.

We are now ready to complete the proof of Lemma~\ref{lem:categorical-reverse-comparison}.

\begin{proof}[Proof of Lemma~\ref{lem:categorical-reverse-comparison}]
  For convenience, define $K(u) \triangleq \KL{f_{\alpha,u}}{f_{\beta,u}}$. By Lemma~\ref{lem:gaussian_kl_fi_evolution}, we have
  \begin{equation}
    -K'(u) = \frac{1}{2}\FI{f_{\alpha,u}}{f_{\beta,u}} \ge 0 \quad\text{and}\quad K''(u) = -\frac{1}{2}\partial_u \FI{f_{\alpha,u}}{f_{\beta,u}} \le -\Bigl(\frac{2}{u} + \frac{3}{u^2}\Bigr)K'(u)\,.
  \end{equation}
  Let $\xi = v/u > 1$. By Gr\"onwall's inequality applied to $-K'$, for all $s \ge u$ we have
  \begin{equation}
    -K'(\xi s) \ge -K'(s) \exp\Bigl(-\int_s^{\xi s} \Bigl(\frac{2}{r} + \frac{3}{r^2}\Bigr)\,\d r\Bigr) = -K'(s) \frac{1}{\xi^2} e^{3/(\xi s) - 3/s}\ge -K'(s) \frac{1}{\xi^2} e^{3/v - 3/u}\,.
  \end{equation}
  Note that $K(\infty) = 0$, so
  \begin{equation}
    K(u) = -\int^\infty_u K'(s)\,\d s \le -e^{3/u-3/v} \int^\infty_u \xi^2  K'(\xi s)\,\d s = \frac{v}{u}e^{3/u - 3/v} K(v)\,,
  \end{equation}
  as claimed.
\end{proof}

\subsection{Score error and approximate reverse kernel}
\label{sec:approx_score_gaussian}

Note that from the true score we can read off information about the posterior per-token marginals. Indeed, by Tweedie's formula (Lemma~\ref{lem:tweedie}),
\begin{equation}
  m_v(z)\triangleq\EE[X_0\mid Z_v=z]=z+vg_v(z)\,,
\end{equation}
and for every block $i\in[d]$,
\begin{equation}
  m_v^{(i)}(z)
  = \Pr{A_i=\cdot\mid Z_v=z}\,.
  \label{eq:onehot-posterior-vector}
\end{equation}
Conditioned on $(X_0,Z_v=z)$, the bridge law of the forward process in Eq.~\eqref{eq:additive-brownian} is
\begin{equation}
  Z_u\sim N\!\left(\frac uvz+\left(1-\frac uv\right)X_0,\;
  u\left(1-\frac uv\right)\Id_{Sd}\right)\,.
  \label{eq:categorical-finite-bridge}
\end{equation}
Therefore, the reverse kernel's marginals are given by a mixture of Gaussians:
\begin{align}
  R^{(i)}_{v\to u}(y,\cdot) &\triangleq\law\bigl(Z_u^{(i)}\mid Z_v=y\bigr) \\
  &= \sum_{a\in \Sigma} m^{(i)}_{v,a}(y)\cdot
  N\!\left(\frac uvy^{(i)}+\left(1-\frac uv\right)e_a,\;
  u\left(1-\frac uv\right)\Id_S\right)\,.
\end{align}
To sample from $R^{(i)}_{v\to u}(y,\cdot)$, one draws $\xi_i=e_a$ with probability $m^{(i)}_{v,a}(y)$ and outputs $\frac uvy^{(i)}+(1-\frac uv)\xi_i+\sqrt{u(1-\frac uv)}\,G_i$, where $G_i\sim N(0,\Id_S)$ is drawn independently across blocks.
This naturally suggests the following product kernel in analogy with Eq.~\eqref{eq:product-reverse-kernel}:
\begin{equation}
  \prodrevg_{v\to u}(y,\cdot)
  \triangleq \prod_{i=1}^d R^{(i)}_{v\to u}(y,\cdot)\,.
  \label{eq:gaussian-product-reverse-kernel}
\end{equation}

In this subsection, we show that the approximate score $\widehat{g}$ can likewise give rise to an approximation of this product kernel. We then bound the distance between the approximation and the true product kernel in terms of the score error $\epsilon_{\sf gauss}$.

In lieu of the Bregman projection from Section~\ref{sec:approx_score_uniform}, here we simply use a Euclidean projection $\Pi_{\Delta(\Sigma)^d}$: given the rescaled estimate $\widehat g_v$, define
\begin{equation}
    \widehat m_v(z)
    \triangleq\Pi_{\Delta(\Sigma)^d}\bigl(z+v\widehat g_v(z)\bigr)\,.
    \label{eq:categorical-projected-mean}
\end{equation}
Since $m_v(z)\in\Delta(\Sigma)^d$ by Eq.~\eqref{eq:onehot-posterior-vector} and Euclidean projection onto the closed convex set $\Delta(\Sigma)^d$ is nonexpansive, we have
\begin{equation}
  \|\widehat m_v(z)-m_v(z)\|_2
  \leq v\|\widehat g_v(z)-g_v(z)\|_2\,,
  \label{eq:categorical-projection-error}
\end{equation}
in analogy with the Pythagorean theorem in Eq.~\eqref{eq:pythagorean}. We can then construct the approximate reverse kernel
\begin{equation}
    \approxrevg_{v\to u}(y,\cdot) \triangleq \prod^d_{i=1} \widehat{R}^{(i)}_{v\to u}(y,\cdot)\,,
\end{equation}
by replacing $m_v$ above with $\widehat{m}_v$ to obtain
\begin{equation}
  \widehat{R}^{(i)}_{v\to u}(y,\cdot)
  \triangleq\sum_{a\in\Sigma}\widehat m^{(i)}_{v,a}(y)\cdot
  N\!\left(\frac uvy^{(i)}+\left(1-\frac uv\right)e_a,\;
  u\left(1-\frac uv\right)\Id_S\right)\,.
  \label{eq:categorical-product-bridge}
\end{equation}
We now proceed to bound the distance between the product kernel and its approximation. To do so, we would like to quantify the distance between two mixtures of Gaussians at a given noise level in terms of the squared Euclidean distance between their mixing coefficients:

\begin{lemma}\label{lem:categorical-mixture-perturbation}
For $\alpha,\beta\in\Delta(\Sigma)$ and $\rho>0$,
\begin{equation}
  \KL{f_{\alpha,\rho}}{f_{\beta,\rho}}
  \leq\bigl(e^{2/\rho}-1\bigr)\|\alpha-\beta\|_2^2\,.
  \label{eq:categorical-mixture-perturbation}
\end{equation}
\end{lemma}

\noindent We defer the proof to Appendix~\ref{app:deferred-proofs}.

From this, we can readily control the distance between the product kernel and its approximation in terms of the score estimation error:

\begin{lemma}\label{lem:gaussian_score_error}
    For $0<u<v$, let
    \begin{equation}
      w_u(v)
        \triangleq(1+v)\left[\exp\!\left(\frac{2(v-u)}{uv}\right)-1\right]\,.
    \end{equation}
    Then
    \begin{equation}
        \EE_{y\sim\pi_v}\KL{\prodrevg_{v\to u}(y,\cdot)}{\approxrevg_{v\to u}(y,\cdot)}
        \le w_u(v)\,\epsilon_{\sf gauss}(t(v))^2
        \,.
        \label{eq:categorical-one-step-score}
    \end{equation}
    Moreover, if $c_u(v) \le 1$, then
    \begin{equation}
      w_u(v)
      \leq 8\log(1+c_u(v))\,.
      \label{eq:categorical-short-step-weight}
    \end{equation}
\end{lemma}

\begin{proof}
  As both kernels are products over blocks, it suffices to bound the KL for each block. Over a fixed block $i\in[d]$, the kernels are given by mixtures of Gaussians. After translating both by $(u/v)y^{(i)}$ and scaling by $(1-u/v)^{-1}$, these mixtures become the mixtures $f_{m_v^{(i)}(y),\rho}$ and $f_{\widehat m_v^{(i)}(y),\rho}$ from Lemma~\ref{lem:categorical-mixture-perturbation}, with $\rho=uv/(v-u)$. By Lemma~\ref{lem:categorical-mixture-perturbation} and Eq.~\eqref{eq:categorical-projection-error},
\begin{equation}
  \sum_{i=1}^d\KL{R^{(i)}_{v\to u}(y,\cdot)}{\widehat{R}^{(i)}_{v\to u}(y,\cdot)}
  \leq\bigl(e^{2(v-u)/(uv)}-1\bigr)v^2\|\widehat g_v(y)-g_v(y)\|_2^2\,.
\end{equation}
Averaging over $y\sim\pi_v$ and applying Eq.~\eqref{eq:categorical-additive-score-error} gives the first part of the claim.

For Eq.~\eqref{eq:categorical-short-step-weight}, write $\theta\triangleq(v-u)/(uv)$. The hypothesis implies $v/u\leq2$ and $3\theta\leq\log2$, hence $e^{2\theta}-1\leq4\theta$ and also $(v-u)/u\leq2\log(v/u)$. It follows that
\begin{align}
  w_u(v)
  &\leq4\theta+\frac{4(v-u)}{u}
  \leq4\left(\frac1u-\frac1v\right)+8\log\frac vu
  \leq 8\log(1+c_u(v))\,. \qedhere
\end{align}
\end{proof}

\subsection{Sampler and telescoping argument}

We consider a sampler of the following form. Fix a grid of noise levels
\begin{equation}
  0 < u_0 < u_1 < \cdots < u_M\,.
\end{equation}
\begin{enumerate}
  \item Initialize at $N(0,u_M\Id_{Sd})$ at noise level $u_M$.
  \item For each $j = M - 1,\ldots,0$:
  \begin{itemize}
    \item Apply $\approxrevg_{u_{j+1}\to u_j}$ to get the next iterate.
  \end{itemize}
  \item Given the iterate $Z$ at the smallest noise level $u_0$, round it to a string in $\calX$.
\end{enumerate}

The rounding in Step 3 proceeds as follows. In the uniform diffusion case, recall that we simply applied the forward kernel coordinatewise. In the Gaussian diffusion case, we will instead round the final iterate to a string in $\calX$ by drawing from the product of block-wise posterior marginals (under uniform prior) at $u=u_0$:
\begin{equation}
  Q_{u}(a_1,\ldots,a_d\mid z)
  \triangleq\prod_{i=1}^d Q^{(i)}_{u}(a_i\mid z^{(i)})\,, \qquad Q^{(i)}_{u}(a\mid z^{(i)}) \triangleq
  \frac{\exp(z^{(i)}_{a}/u)}{\sum_{b\in\Sigma}\exp(z^{(i)}_b/u)}\,.
  \label{eq:categorical-terminal-kernel}
\end{equation}

\noindent We have the following analogue of the chain rule calculation from Lemma~\ref{lem:chain} which combines the main estimates from the previous subsections (Lemma~\ref{lem:gaussian-tc-dtc} and Lemma~\ref{lem:gaussian_score_error}).

\begin{lemma}\label{lem:gaussian_chain}
  Let $\widehat q$ denote the output law of the above sampler. Then
  \begin{multline}
    \KL{q_{\sf pre}}{\widehat q}
    \le \KL{\pi_{u_M}}{N(0,u_M\Id_{Sd})}
    +\EE_{Z_{u_0}\sim\pi_{u_0}}\KL{\law(A\mid Z_{u_0})}{Q_{u_0}(\cdot\mid Z_{u_0})}\\
    +\sum^{M-1}_{j=0}\Bigl[c_{u_j}(u_{j+1})\bigl(\DTC(\pi_{u_j})-\DTC(\pi_{u_{j+1}})\bigr)
    +w_{u_j}(u_{j+1})\,\epsilon_{\sf gauss}(t(u_{j+1}))^2\Bigr]\,.
  \end{multline}
\end{lemma}

\begin{proof}
  The proof follows verbatim from the argument for Lemma~\ref{lem:chain}, with the reversed forward path $(X_{t_M},\ldots,X_{t_0},X_0)$ replaced by $(Z_{u_M},\ldots,Z_{u_0},A)$, the initialization $\Unif(\calX)$ by $N(0,u_M\Id_{Sd})$, the kernel $(K^{\sf U}_{t_0})^{\otimes d}$ by $Q_{u_0}$, and Lemmas~\ref{lem:uniform-tc-dtc} and~\ref{lem:uniform_score_error} by Lemmas~\ref{lem:gaussian-tc-dtc} and~\ref{lem:gaussian_score_error}. The error from the rounding step $\EE_{Z_{u_0}\sim\pi_{u_0}}\KL{\law(A\mid Z_{u_0})}{Q_{u_0}(\cdot\mid Z_{u_0})}$ will be analyzed in a subsequent lemma.
\end{proof}

\noindent We next control the initialization and rounding errors in Lemma~\ref{lem:gaussian_chain} by taking $u_M$ sufficiently large and $u_0$ sufficiently small, respectively; the following two results play the roles of Lemmas~\ref{lem:forward_decay} and~\ref{lem:shorttime}.

\begin{lemma}\label{lem:categorical-initialization}
For every $U>0$,
\begin{equation}
  \KL{\pi_U}{N(0,U\Id_{Sd})}
  \leq\frac d{2U}\,.
  \label{eq:categorical-initialization}
\end{equation}
\end{lemma}

\begin{proof}
By convexity of KL in its first argument, it suffices to compare $N(x,U\Id_{Sd})$ with $N(0,U\Id_{Sd})$ for a fixed $x$ whose blocks are standard basis vectors. As the two covariances agree, the KL is simply $\|x\|_2^2/(2U) = d/(2U)$.
\end{proof}

\begin{lemma}\label{lem:categorical-terminal-kernel}
For $u>0$, let
\begin{equation}
  p^*(u) \triangleq \frac{S-1}{2}\mathrm{erfc}\Bigl(\frac{1}{2\sqrt{u}}\Bigr)\,.
\end{equation}
If $p^*(u)\leq1/2$, then
\begin{equation}
  \EE_{Z_u\sim\pi_u}\KL{\law(A\mid Z_u)}{Q_u(\cdot\mid Z_u)}
  \leq d\bigl(h_2(p^*(u))+p^*(u)\log(S-1)\bigr)\,.
  \label{eq:categorical-terminal-error}
\end{equation}
\end{lemma}

\begin{proof}
The left-hand side is at most $\EE[-\log Q_u(A\mid Z_u)]$, by nonnegativity of entropy. As $Q_u$ is a product kernel, this decomposes into a sum $\sum^d_{i=1} \EE[-\log Q^{(i)}_u(A_i\mid Z_u^{(i)})]$.

Fix any block $i\in[d]$ and $a\in \Sigma$, and consider the contribution of this block conditioned on $A_i = a$, namely
\begin{equation}
  \EE[-\log Q^{(i)}_u(A_i \mid Z_u^{(i)})\mid A_i = a]\,. \label{eq:postprob}
\end{equation}
This has the following interpretation. Nature samples $A'$ uniformly at random from $\Sigma$ and $\gamma\sim N(0,u\Id_S)$, and one observes the vector $z = e_{A'} + \gamma$. Then $Q^{(i)}_u(a\mid z)$ is the observer's posterior probability that $A' = a$. Conditional on $A'=a$, we have $z=e_a+\gamma$, which has the same distribution as $Z_u^{(i)}$ conditional on $A_i=a$, so in Eq.~\eqref{eq:postprob} the expectation is over the randomness of $\gamma$. As a result, Eq.~\eqref{eq:postprob} is independent of $a$. So we may freely replace the expectation over the marginal distribution of $A_i$ under $A\sim q_{\sf pre}$ in $\EE[-\log Q^{(i)}_u(A_i\mid Z^{(i)}_u)]$ with an expectation over any distribution, in particular, over $A_i\sim\Unif(\Sigma)$. Thus, $\EE[-\log Q^{(i)}_u(A_i\mid Z^{(i)}_u)]$ is nothing more than the conditional entropy of a uniformly random element $A'$ of $\Sigma$ conditioned on observing $e_{A'} + \gamma$. Consider any decoder $F: \RR^S \to \Sigma$ that tries to predict the former given the latter, and let $p_{\rm err} \triangleq \Pr{F(e_{A'} + \gamma) \neq A'}$ denote the decoding error.

By Fano's inequality, the conditional entropy is at most
\begin{equation}
  h_2(p_{\rm err}) + p_{\rm err}\log(S-1)\,,
\end{equation}
where $h_2$ denotes binary entropy. Take the decoder $F$ to simply choose the largest coordinate of $e_{A'} + \gamma$. For any realization $A' = a$, the decoding error is the probability that $\gamma_b - \gamma_a > 1$ for some $b\neq a$, which by a union bound is at most $\frac{S-1}{2}\mathrm{erfc}\bigl(1/(2\sqrt{u})\bigr)$. Since $p_{\rm err}\le p^*(u)\le1/2$ and $p\mapsto h_2(p)+p\log(S-1)$ is increasing on $[0,1/2]$, the claimed bound then follows by summing over $i\in[d]$.
\end{proof}

\noindent We are now ready to prove our main result, a DTC-adaptive query complexity for Gaussian diffusion:

\begin{proof}[Proof of Theorem~\ref{thm:gaussian-dtc-sampler}]
It remains to set the noise levels so that the sum in Lemma~\ref{lem:gaussian_chain} telescopes (see Figure~\ref{fig:gaussian-schedule}). Define
\begin{equation}
  L_{\sf G}=\log(64edS/\epsilon),\qquad
  u_0=\frac1{8L_{\sf G}},\qquad
  U=\max\left\{1,\frac{4d}{\epsilon}\right\}\,,
\end{equation}
and
\begin{equation}
  a_{\sf G}=\frac{\epsilon}{4\max\{\overline{\DTC},\epsilon\}},\qquad
  F_{\sf G}(u)=\log u-\frac3u\,,
\end{equation}
noting that $a_{\sf G}\le 1/4$.
Recursively set $u_{j+1}$ by
\begin{equation}
  F_{\sf G}(u_{j+1})-F_{\sf G}(u_j)
  =\min\left\{\log(1+a_{\sf G}),\,F_{\sf G}(U)-F_{\sf G}(u_j)\right\}\,,
  \label{eq:categorical-schedule}
\end{equation}
and let $N_{\sf G}$ be the first index for which $u_j=U$; this is well defined because $F_{\sf G}$ is increasing. In the proof of Theorem~\ref{thm:uniform-dtc-sampler}, the schedule was geometric in the parameter $u=e^t-1$; here, the potential $F_{\sf G}$ replaces $\log u$ as the clock because $F_{\sf G}(v)-F_{\sf G}(u)=\log(v/u) + 3/u - 3/v=\log(1+c_u(v))$. In particular, every step satisfies
\begin{equation}
  c_{u_j}(u_{j+1})= \frac{u_{j+1}}{u_j} e^{3/u_j - 3/u_{j+1}} -1\le a_{\sf G}\,,
\end{equation}
and therefore
\begin{equation}
  \sum^{N_{\sf G}-1}_{j=0}c_{u_j}(u_{j+1})\bigl(\DTC(\pi_{u_j})-\DTC(\pi_{u_{j+1}})\bigr)
  \le a_{\sf G}\,\DTC(\pi_{u_0})
  \le a_{\sf G}\,\DTC(q)
  \le \epsilon/4\,,
\end{equation}
since $\DTC(\pi_u)$ is nonincreasing in $u$ by Lemma~\ref{lem:dtc_decrement}.

\begin{figure}[t]
  \centering
  \definecolor{coordblue}{RGB}{43,82,134}
\definecolor{blockred}{RGB}{158,53,51}
\begin{tikzpicture}[
  x=1cm, y=1cm,
  axis/.style={black!45, line width=0.4pt},
  map/.style={black!25, line width=0.45pt},
  gridpt/.style={draw=coordblue, fill=coordblue!15, line width=0.7pt},
  endpt/.style={draw=blockred, fill=blockred!12, line width=0.7pt},
  steparc/.style={-stealth, black!60, line width=0.55pt,
               shorten <=1pt, shorten >=1pt},
  rowlab/.style={anchor=east, inner sep=1pt},
  idx/.style={black!55, font=\scriptsize, inner sep=1pt},
  tag/.style={black!55, font=\footnotesize},
]
\def\yt{1.8}
\def\ys{0}
\foreach \xt/\xb in {0/0, 0.931/0.269, 1.863/0.568, 2.794/0.904,
                     3.725/1.284, 4.657/1.723, 5.588/2.239, 6.519/2.860,
                     7.450/3.627, 8.382/4.607, 9.313/5.895, 10.244/7.604,
                     11.176/9.775, 12.0/12.0}
  \draw[map] (\xt,\yt-0.09) -- (\xb,\ys+0.09);
\draw[axis] (-0.15,\yt) -- (12.15,\yt);
\draw[axis] (-0.15,\ys) -- (12.15,\ys);
\foreach \xt/\xb in {0.931/0.269, 1.863/0.568, 2.794/0.904, 3.725/1.284,
                     4.657/1.723, 5.588/2.239, 6.519/2.860, 7.450/3.627,
                     8.382/4.607, 9.313/5.895, 10.244/7.604, 11.176/9.775}{
  \fill[gridpt] (\xt,\yt) circle (0.058);
  \fill[gridpt] (\xb,\ys) circle (0.058);
}
\foreach \xt/\xb in {0/0, 12.0/12.0}{
  \fill[endpt] (\xt,\yt) circle (0.058);
  \fill[endpt] (\xb,\ys) circle (0.058);
}
\node[rowlab] at (-0.35,\yt) {$F_{\sf G}(u_j)$};
\node[rowlab] at (-0.35,\ys)
  {\shortstack{$u_j$\\[-1pt]{\scriptsize\color{black!55}(log scale)}}};
\draw[-stealth, black!55, line width=0.5pt] (-1.85,\ys) -- (-1.85,\yt);
\node[black!55, font=\footnotesize, rotate=90, anchor=center]
  at (-2.07,0.9) {$F_{\sf G}(u)=\log u-\tfrac{3}{u}$};
\node[idx] at (0,\yt+0.33) {$F_{\sf G}(u_0)$};
\node[idx] at (12,\yt+0.33) {$F_{\sf G}(U)$};
\node[idx] at (0,\ys-0.30) {$u_0$};
\node[idx] at (12,\ys-0.30) {$U$};
\draw[steparc] (3.725,\yt+0.10) .. controls +(0.15,0.32) and +(-0.15,0.32)
  .. (4.657,\yt+0.10);
\node[tag] at (4.19,\yt+0.62) {$+\log(1+a_{\sf G})$};
\draw[steparc] (7.604,\ys-0.10) .. controls +(0.3,-0.34) and +(-0.3,-0.34)
  .. (9.775,\ys-0.10);
\node[tag] at (8.7,\ys-0.72) {$u_{j+1}\approx(1+a_{\sf G})\,u_j$};
\node[tag] at (2.2,\ys-0.72)
  {$\tfrac{3}{u_j}-\tfrac{3}{u_{j+1}}\approx\log(1+a_{\sf G})$};
\end{tikzpicture}
  \caption{Step schedule for Gaussian diffusion sampler}
  \label{fig:gaussian-schedule}
\end{figure}
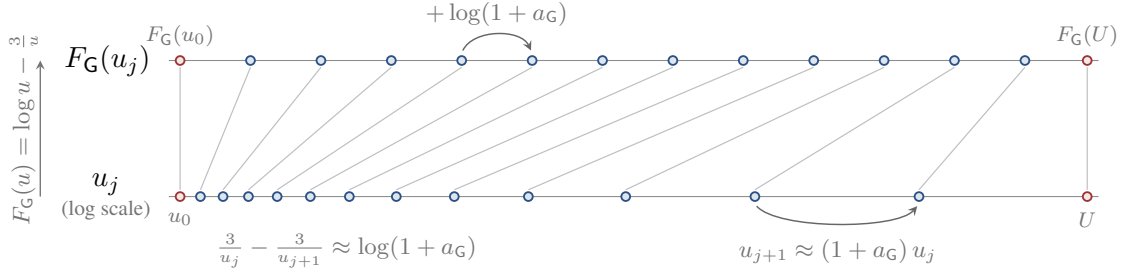

Next, we handle the contribution from score estimation error. We have
\begin{equation}
  \log(1 + c_{u_j}(u_{j+1})) = \log(u_{j+1}/u_j) + 3/u_j - 3/u_{j+1} \le \log(1+a_{\sf G})\,,
\end{equation}
so $c_{u_j}(u_{j+1}) \le 1$ and thus by Eq.~\eqref{eq:categorical-short-step-weight}, 
\begin{equation}
  w_{u_j}(u_{j+1})\le 8\log(1+c_{u_j}(u_{j+1}))\,.
\end{equation}
Furthermore,
\begin{equation}
  \sum^{N_{\sf G}-1}_{j=0} \log(1+c_{u_j}(u_{j+1})) = F_{\sf G}(U) - F_{\sf G}(u_0) \le  \log(U/u_0) + 3/u_0 \le \log(8L_{\sf G}U) + 24 L_{\sf G} \lesssim L_{\sf G}\,. \label{eq:sumcs}
\end{equation}
Therefore,
\begin{equation}
  \sum^{N_{\sf G}-1}_{j=0}w_{u_j}(u_{j+1})\,\epsilon_{\sf gauss}(t(u_{j+1}))^2
  \lesssim \epsilon_{\sf gauss}^2 L_{\sf G}\,.
\end{equation}

Next, we handle the initialization error. By Lemma~\ref{lem:categorical-initialization} and the choice of $U$,
\begin{equation}
  \KL{\pi_U}{N(0,U\Id_{Sd})}
  \le\frac d{2U}
  \le\frac\epsilon8\,.
\end{equation}

Next, for the rounding error, we wish to apply Lemma~\ref{lem:categorical-terminal-kernel}. By standard Gaussian tail bounds, $p^*(u_0)\le Se^{-2L_{\sf G}}\le1/2$. Furthermore, $h_2(p^*(u)) \le p^*(u)\log(e/p^*(u))$. Hence,
\begin{equation}
  \EE_{Z_{u_0}\sim\pi_{u_0}}\KL{\law(A\mid Z_{u_0})}{Q_{u_0}(\cdot\mid Z_{u_0})}
  \le dSe^{-2L_{\sf G}}(1+2L_{\sf G})
  \le dSe^{-L_{\sf G}}
  \le\frac\epsilon4\,.
\end{equation}

Finally, we bound the total number of score oracle queries. By design, we have
\begin{equation}
  N_{\sf G}\leq1+\frac{F_{\sf G}(U)-F_{\sf G}(u_0)}{\log(1+a_{\sf G})}\,.
\end{equation}
Recall from Eq.~\eqref{eq:sumcs} that $F_{\sf G}(U)-F_{\sf G}(u_0)\lesssim L_{\sf G}$. Additionally, $1/\log(1+a_{\sf G})\leq2/a_{\sf G}\lesssim 1+\overline{\DTC}/\epsilon$. Each application of $\approxrevg_{u_{j+1}\to u_j}$ requires only the score vector $\widehat{g}_{u_{j+1}}(y)\in\RR^{Sd}$ at the current iterate $y$, from which Eq.~\eqref{eq:categorical-projected-mean} and the approximate reverse kernel $\approxrevg$ are formed. So the sampler makes $N_{\sf G}$ queries in total, which proves the query bound.

Substituting all of the above bounds into Lemma~\ref{lem:gaussian_chain} gives the claimed KL bound.
\end{proof}

\section{Construction of random empirical measure}
\label{sec:measure}

In the second part of this paper, we turn to proving upper and lower bounds to separate the three paradigms of diffusion language modeling. Here, we first briefly present the family of distributions that we consider, which are based on a construction from recent work~\cite{xun2026query} proving a query complexity lower bound for Gaussian diffusion.

Fix constants $0<\kappa_-<\kappa_+<\log2$. In the subsequent sections, we will specialize to the case of binary alphabet, so that
\begin{equation}
    \Sigma = \{\pm 1\}\,, \qquad  \calX = \{\pm 1\}^d\,, \qquad \beta_t = \frac{1 - e^{-t}}{2}\,.
\end{equation}
Let $\calC\subseteq \calX$ denote the subset of the hypercube defining the codebook, and let $M = |\calC|$ denote the size of the codebook. Throughout, we will take
\begin{equation}
  M = \lceil e^{\kappa d} \rceil
\end{equation}
for constant $\kappa\in[\kappa_-,\kappa_+]$. The compact interval $[\kappa_-,\kappa_+]$ ensures that all the critical values below remain in a fixed compact set.

\begin{definition}\label{def:random_empirical_measure}
  Given a subset $\calC\subseteq\{\pm 1\}^d$, denote the uniform distribution over $\calC$ by
  \begin{equation}
      q^\calC \triangleq \Unif(\calC)\,.
  \end{equation}
  The \emph{random empirical measure} we consider is $q^\calC$ for $\calC$ a uniformly random $M$-element subset of $\{\pm1\}^d$.

  We will often refer to the elements of $\calC$ as \emph{codewords}.
\end{definition}

\noindent In all of our bounds in the sequel, we assume that the sampler:
\begin{itemize}[leftmargin=*,itemsep=0pt]
    \item Knows the interval $[\kappa_-, \kappa_+]$ but does not know $\kappa$
    \item Knows that it has approximate score oracle access to a random empirical measure in the sense of Definition~\ref{def:random_empirical_measure}, but does not know the specific realization of $\calC$
\end{itemize}

\paragraph{Notation for subsequent sections.} In our analysis, constants denoted by $c, C$ may depend on the known parameters $\kappa_-, \kappa_+$, but on no other
parameter; dependence on additional parameters is indicated by a subscript.
For the uniform and Gaussian forward processes, we will overload the notation $q^\calC_t$ to denote the law of $X_t$ under either forward process, with $X_0\sim q^\calC$. 

All three main results below as well as the uniform lower bound in Appendix~\ref{app:uniform-query-lower} use Definition~\ref{def:random_empirical_measure}. The upper bounds we prove for uniform and Gaussian diffusion hold for any $\kappa\in[\kappa_-,\kappa_+]$, whereas the lower bounds for masked and uniform diffusion hold even if $\kappa$ is additionally known to be restricted to the lattice $\{(k/d)\log2:k\in\ZZ\}$, in which case $M=e^{\kappa d}=2^k$.

Define the scalar mutual information for the uniform and Gaussian channels by
\begin{align}
  \uMI(t) &\triangleq \frac{1 + e^{-t}}{2}\log(1 + e^{-t})
  + \frac{1 - e^{-t}}{2}\log(1 - e^{-t}), \\
  \gMI(t) &\triangleq I\!\left(V;
  e^{-t}V+\sqrt{1-e^{-2t}}\,G\right),
  \qquad V\sim\Unif(\{\pm1\}),\quad G\sim N(0,1)\,.
\end{align}
Here and below, $h_2(p)\triangleq -p\log p-(1-p)\log(1-p)$ denotes the binary entropy function (in nats); note that $\uMI(t)=\log 2-h_2(\beta_t)$.

Define critical times $\tcritu$ and $\tcritg$ to be the unique solutions of
\begin{equation}
  d\,\uMI(\tcritu)=\log M
  \qquad\text{and}\qquad
  d\,\gMI(\tcritg)=\log M\,. \label{eq:critical_times}
\end{equation}
Additionally, let $0<\tau_-<\tau_+$ be (known) constants such that $\tcritu, \tcritg\in[\tau_-,\tau_+]$ for every $\kappa \in [\kappa_-, \kappa_+]$ and all sufficiently large $d$, and define
\begin{equation}
    I_0=[\tau_-/2,\tau_++1]\,.\label{eq:known_ts}
\end{equation}
In the proofs of our \emph{upper} bounds, all times queried by our algorithms will lie in $I_0$ defined in Eq.~\eqref{eq:known_ts}, though we place no such constraint in the proofs of our \emph{lower} bounds.

\section{\texorpdfstring{$\widetilde O(\sqrt d)$}{soft-O(sqrt(d))} upper bound for uniform diffusion}
\label{sec:uniform-upper}

In this section, we give an algorithm for approximately sampling from random empirical measures using $\widetilde{O}(\sqrt{d})$ queries to any approximate uniform diffusion score oracle for such a distribution. This query complexity is tight: in Appendix~\ref{app:uniform-query-lower} we prove a matching lower bound.

\begin{theorem}[Uniform diffusion with approximate scores]\label{thm:uniform_random_measure_sampler}
  There are constants $B,C,c_0,d_0>0$, depending only on $[\kappa_-,\kappa_+]$, such that the following holds. Let $d\ge d_0$, let $\epsilon,\delta\in(0,1/4)$, and set
  \begin{equation}
    \Lambda=\log\frac{C d}{\epsilon\delta}\,.
    \label{eq:uniform_sampler_parameter}
  \end{equation}
  Suppose $\Lambda\le c_0d$.
  
  Given query access to any approximate uniform diffusion score oracle $(\widehat{s}_t)$ for $q^\calC$ with score error satisfying
  \begin{equation}
    \sup_{t\in I_0}\epsilon_{\sf unif}(t)\le e^{-B\Lambda}= \left(\frac{\epsilon\delta}{Cd}\right)^B\,,
    \label{eq:uniform_score_accuracy}
  \end{equation}
  there is an algorithm using at most $O(\sqrt{d\Lambda}/\epsilon^2)$
  score oracle queries such that with probability at least $1 - \delta$ over $\calC$ which is sampled uniformly from all $M$-element subsets of $\calX$, the output law of the algorithm $\widehat q^\calC$ satisfies $\TV(\widehat q^\calC,q^\calC)\le\epsilon$.

  In particular, for constant $\epsilon$ and any polynomially small $\delta=d^{-\Theta(1)}$, polynomially small score accuracy suffices, and the algorithm uses $O(\sqrt{d\log d})$ score queries.
\end{theorem}

\noindent For the rest of the section, $\Lambda$ has the value in Eq.~\eqref{eq:uniform_sampler_parameter}. We may decrease $c_0$ and increase $d_0$ later when necessary so that $\log d\le\Lambda\le c_0d$.

\par\vspace{\baselineskip}\noindent In Section~\ref{sec:uniform-critical-window}, we quantify the critical window: above the window, $q^\calC_t$ is close to the uniform distribution, while below the window the original sample can be recovered from its noisy observation. In Section~\ref{sec:locate_window_uniform}, we show how to locate this window using approximate score oracle queries at points sampled from $\Unif(\calX)$. Finally, in Section~\ref{sec:uniform-window-sampling}, we describe how to simulate the reverse process across the window and then decode the codeword, completing the proof of Theorem~\ref{thm:uniform_random_measure_sampler}.

\subsection{Quantifying the critical window}
\label{sec:uniform-critical-window}

In this section, we characterize the critical window of noise levels such that, for higher noise levels $t$, the score of $q^\calC_t$ is close to that of the uniform distribution, and for lower noise levels $t$, given a noisy sample from $q^\calC_t$, it is possible to decode the original sample from $q^\calC$ that gave rise to it with high accuracy.

Throughout, we will denote the likelihood ratio between $q^\calC_t$ and the uniform distribution over $\calX$ by
\begin{equation}
  L_t(\cdot) \triangleq \frac{q^\calC_t(\cdot)}{2^{-d}}\,.
\end{equation}

\par\vspace{\baselineskip}\noindent First, we bound the expected distance between $q^\calC_t$ and the uniform distribution for large $t$.

\begin{lemma}\label{lem:uniform_output_radius}
  For any integer radius $0\le r \le d$,
  \begin{equation}
    \EE_\calC\,\TV(q^\calC_t, \Unif(\calX))\le \frac{1}{2\sqrt{M}}(1 + e^{-t})^{d/2}\cdot \tanh(t/2)^{(r+1)/2} + \Pr{\Bin(d, \beta_t) \le r}\,.
  \end{equation}
\end{lemma}

\noindent We defer the proof to Appendix~\ref{app:deferred-proofs}.

Next, we show that given a noisy sample from $q^\calC_t$ for small $t$, it is possible to recover the original sample from $q^\calC$ with high accuracy.

\begin{lemma}\label{lem:uniform_recovery_radius}
  Given $\calC\subseteq\{\pm 1\}^d$, define the Bayes-optimal recovery error
  \begin{equation}
    \err^{\sf U}_\calC(t) \triangleq \inf_{\widehat{z}(\cdot)} \Pr[z\sim q^\calC, x\sim (K_t^{\sf U})^{\otimes d}(\cdot \mid z)]{\widehat{z}(x) \neq z}\,.
  \end{equation}
  For any integer radius $0 \le r \le d$, 
  \begin{equation}
    \EE_\calC\,\err^{\sf U}_\calC(t) \le M\cdot \Pr{\Bin(d,1/2) \le r} + \Pr{\Bin(d, \beta_t) > r}\,.
  \end{equation}
\end{lemma}

\noindent We defer the proof to Appendix~\ref{app:deferred-proofs}.

Setting parameters appropriately and using the two estimates above, we can thus characterize the critical window for sampling from the random empirical measure using uniform diffusion: above the critical window, the forward process has sufficiently mixed so that the law of $X_t$ is close to uniform, while below the critical window, the Bayes-optimal decoder can recover the original codeword with high probability.

\begin{proposition}[Critical window]\label{prop:uniform_critical_window}
  For every $K>0$ there is $A_K>0$ with the following property. For every $A\ge A_K$, there are $c_{K,A},d_{K,A}>0$, depending only on $A,K$, and $[\kappa_-,\kappa_+]$, such that the following holds whenever $d\ge d_{K,A}$ and $\log d\le\Lambda\le c_{K,A}d$. Set
  \begin{equation}
      w=A\sqrt{\frac\Lambda d}\,.\label{eq:uniform_window_width}
  \end{equation}
  Then we have
  \begin{align}
    \EE_\calC\TV(q^\calC_t,\Unif(\calX))&\le e^{-K\Lambda}, &&t\ge\tcritu+w,\label{eq:uniform_above_window}\\
    \EE_\calC\err^{\sf U}_\calC(t)&\le e^{-K\Lambda}, &&0\le t\le\tcritu-w\,.\label{eq:uniform_below_window}
  \end{align}
\end{proposition}

\begin{proof}
  We first observe that $\err^{\sf U}_\calC(t)$ is nondecreasing in $t$ and, by data processing inequality, $\TV(q^\calC_t,\Unif(\calX))$ is nonincreasing in $t$.  
  Therefore, it suffices to consider the endpoints $t=\tcritu+w$ and $t=\tcritu-w$. By decreasing $c_{K,A}$ we may assume $w\le\min\{\tau_-/2,1\}$; then these points are within $I_0$, so $\beta_t$ is bounded away from $0$ and $1/2$ and
  \begin{equation}
    c w\le |\beta_t-\beta_{\tcritu}|\le Cw,
    \qquad \text{for }t\in\{\tcritu-w,\tcritu+w\}\,,
    \label{eq:uniform_beta_gap}
  \end{equation}
  because $\frac{\D}{\D t} \beta_t=e^{-t}/2$ is bounded above and below by positive constants on $I_0$. Here $c,C>0$ depend only on $[\kappa_-,\kappa_+]$.

  \paragraph{Above the critical window.} Set $t=\tcritu+w$, $\Delta=\beta_t-\beta_{\tcritu}$, and truncation radius $r=\lfloor d(\beta_t+\beta_{\tcritu})/2\rfloor$. The radius is at least $d\Delta/2$ below the mean $d\beta_t$ of $\Bin(d,\beta_t)$. Hoeffding's inequality therefore gives
  \begin{equation}
    \Pr{\Bin(d,\beta_t)\le r}\le e^{-d\Delta^2/2}\,.
    \label{eq:uniform_upper_radius_tail}
  \end{equation}
  It remains to bound the first term in Lemma~\ref{lem:uniform_output_radius}. Eq.~\eqref{eq:critical_times} gives $\log M=d(\log2-h_2(\beta_{\tcritu}))$. Rewrite $1+e^{-t}=2(1-\beta_t)$ and $\tanh(t/2)=\beta_t/(1-\beta_t)$. Since $r+1\ge d\beta_{\tcritu}+d\Delta/2$,
  \begin{align}
    \log\left(\frac{(1+e^{-t})^d\tanh(t/2)^{r+1}}{M}\right) &\le d\log\bigl(2(1-\beta_t)\bigr)+d\left(\beta_{\tcritu}+\frac{\Delta}{2}\right)\log\frac{\beta_t}{1-\beta_t}-\log M\\
    &=-d\,\KL{\Ber(\beta_{\tcritu})}{\Ber(\beta_t)}-\frac{d\Delta}{2}\log\frac{1-\beta_t}{\beta_t}\\
    &\le-cd\Delta\,.
    \label{eq:uniform_upper_radius_likelihood}
  \end{align}
  Lemma~\ref{lem:uniform_output_radius} and Eqs.~\eqref{eq:uniform_upper_radius_tail}--\eqref{eq:uniform_upper_radius_likelihood} now give
  \begin{equation}
    \EE_\calC\TV(q^\calC_t,\Unif(\calX))\le\tfrac12e^{-cd\Delta}+e^{-d\Delta^2/2}\,.
    \label{eq:uniform_upper_boundary_error}
  \end{equation}

  \paragraph{Below the critical window.} Set $t=\tcritu-w$, $\Delta=\beta_{\tcritu}-\beta_t$, and again $r=\lfloor d(\beta_t+\beta_{\tcritu})/2\rfloor$. Since $r+1>d\beta_t+d\Delta/2$, Hoeffding's inequality gives
  \begin{equation}
    \Pr{\Bin(d,\beta_t)>r}\le e^{-d\Delta^2/2}\,.
    \label{eq:uniform_lower_radius_tail}
  \end{equation}
  To control the chance of another point entering the ball, a standard binomial tail bound gives
  \begin{align}
    M\Pr{\Bin(d,1/2)\le r}
    &\le\exp\left(\log M-d\bigl(\log2-h_2((\beta_t+\beta_{\tcritu})/2)\bigr)\right)\nonumber\\
    &=\exp\left(-d\bigl(h_2(\beta_{\tcritu})-h_2((\beta_t+\beta_{\tcritu})/2)\bigr)\right)\le e^{-cd\Delta}\,.
    \label{eq:uniform_lower_radius_count}
  \end{align}
  In the last inequality, we used the fact that $\beta_{\tcritu}-(\beta_t+\beta_{\tcritu})/2=\Delta/2$, and that $h'_2(p)=\log((1-p)/p)$ is bounded away from $0$ for $p$ between $(\beta_t+\beta_{\tcritu})/2$ and $\beta_{\tcritu}$. Combining Eqs.~\eqref{eq:uniform_lower_radius_tail} and~\eqref{eq:uniform_lower_radius_count} with Lemma~\ref{lem:uniform_recovery_radius} yields
  \begin{equation}
    \EE_\calC\err^{\sf U}_\calC(t)\le e^{-cd\Delta}+e^{-d\Delta^2/2}\,.
    \label{eq:uniform_lower_boundary_error}
  \end{equation}

  Finally, Eq.~\eqref{eq:uniform_beta_gap} implies $d\Delta^2\ge cA^2\Lambda$. Since $w\le1$, $d\Delta\ge cdw\ge cdw^2=cA^2\Lambda$. Choosing $A$ large enough in terms of $K$ makes $e^{-cd\Delta}+e^{-d\Delta^2/2}$ at most $e^{-K\Lambda}$. The monotonicity of $\TV(q^\calC_t,\Unif(\calX))$ and $\err^{\sf U}_\calC(t)$ in $t$ then gives the claimed bounds.
\end{proof}

\noindent For the rest of the section, let $w=A\sqrt{\Lambda/d}$ as in Eq.~\eqref{eq:uniform_window_width}, where $A>0$ is a constant to be fixed below. We may decrease $c_0$ and increase $d_0$ later when necessary.

\subsection{Identifying the critical window}
\label{sec:locate_window_uniform}

Having characterized the critical window, we now give an algorithm for detecting its location given approximate uniform diffusion score access to $q^\calC$. 

Given $x\in\calX$, define the denoiser $m_t(x) \triangleq \EE[X_0 \mid X_t = x]$. We will consider the following statistic:
\begin{equation}
  T_t(x) \triangleq \langle x, m_t(x)\rangle - de^{-t} = \frac{1 - e^{-2t}}{2e^{-t}} \sum^d_{i=1} (1 - s_t(x)[i, -x_i])\,.
\end{equation}
Intuitively, $\langle x, m_t(x)\rangle$ corresponds to the overlap between $x$ and its denoising with respect to the true data distribution $q^\calC$, and $d e^{-t}$ is the same but under denoising with respect to the null distribution $\Unif(\calX)$. As $t$ goes to $\infty$, the difference between these tends to zero, but as $t$ goes to $0$, if $x$ is a uniformly random point from $\calX$, then we expect the denoiser under $q^\calC$ to achieve worse overlap than the denoiser under $\Unif(\calX)$. This is formalized in the following:

\begin{proposition}\label{prop:uniform_overlap_separation}
  There are constants $a_0,b_0>0$, depending only on $[\kappa_-,\kappa_+]$, with $b_0>4a_0$, such that for every $K>0$ there is $A_K>0$ with the following property. For every $A\ge A_K$, there are $c_{K,A},d_{K,A}>0$ such that the following holds whenever $d\ge d_{K,A}$ and $\log d\le\Lambda\le c_{K,A}d$:
  \begin{align}
    \Pr[\calC, x\sim\Unif(\calX)]{T_t(x) < -a_0dw} \le e^{-K\Lambda}\,, && t \ge \tcritu + 2w\,,\label{eq:uniform_overlap_above} \\
    \Pr[\calC, x\sim\Unif(\calX)]{T_t(x) > -b_0dw} \le e^{-K\Lambda}\,, && t \le \tcritu - 2w\,,\label{eq:uniform_overlap_below}
  \end{align}
  uniformly over $t\in I_0$.
\end{proposition}

\begin{proof}
  The overlap $\langle x, m_t(x)\rangle$ can be expressed succinctly as $A'_x(\theta(t))$ where
  \begin{equation}
    A_x(\theta) \triangleq \log \EE_{z\sim\Unif(\calC)}[e^{\theta\langle x, z\rangle}]\,,
  \end{equation}
  and $\theta(t) \triangleq \atanh(e^{-t})$. Also define 
  \begin{equation}
    B(\theta) \triangleq d\log\cosh(\theta)\,,
  \end{equation}
  observing that $B'(\theta(t)) = de^{-t}$. We have
  \begin{equation}
    \log L_t(x) = A_x(\theta(t)) - B(\theta(t)) \qquad \text{and} \qquad T_t(x) = A'_x(\theta(t)) - B'(\theta(t))\,.
  \end{equation}

  \paragraph{High-noise regime ($t \ge \tcritu + 2w$).} Choose neighboring noise levels $t$ and $t^+ = t + \eta w$ and define $\theta = \theta(t)$ and $\theta^+ = \theta(t^+)$ so that $0 \le \theta - \theta^+ \asymp \eta w$. The function $A_x(\theta)$ is convex in $\theta$, so
  \begin{equation}
    A'_x(\theta) \ge \frac{A_x(\theta) - A_x(\theta^+)}{\theta - \theta^+}\,,
  \end{equation}
  and thus
  \begin{equation}
    T_t(x) \ge \frac{\log L_t(x) - \log L_{t^+}(x)}{\theta - \theta^+} + \Bigl(\frac{B(\theta) - B(\theta^+)}{\theta - \theta^+} - B'(\theta)\Bigr)\,.
  \end{equation}
  Because $B''(u) = d\,\sech^2(u) \le d$, by Taylor's theorem the second term above is $\ge -C'd\eta w$ for some constant $C' > 0$. On the other hand, since $\EE_{X\sim\Unif(\calX)}|L_t(X)-1|=2\TV(q^\calC_t,\Unif(\calX))$, Proposition~\ref{prop:uniform_critical_window} and Markov's inequality show that $L_t(X),L_{t^+}(X)\in[1/2,3/2]$ except with probability at most $e^{-(K+1)\Lambda}$. On this event, the first term is $\ge -C''/(\eta w)$ for some constant $C'' > 0$. With constant $\eta$, $1/(\eta w) = o(dw)$ since $dw^2 = A^2\Lambda = \Omega(\log d)$. Choosing constant $\eta$ small enough, we can ensure $-C''/(\eta w)-C'd\eta w\ge-a_0dw$ for any prescribed constant $a_0>0$.

  \paragraph{Low-noise regime ($t \le \tcritu - 2w$).} First note that $A'_x(\theta(t))$ is an average of overlaps of the form $\langle x, z\rangle$ for $z\in \calC$. We bound the max over these overlaps by a Chernoff and union bound calculation. By Chernoff, for any $u > 0$,
  \begin{equation}
    \Pr[z\sim\Unif(\calX)]{\langle x,z\rangle \ge du} \le e^{-d\, \KL{\Ber(\frac{1+u}{2})}{\Ber(1/2)}}\,.
  \end{equation}
  So for $u = e^{-\tcritu} + O(\Lambda/d)$, noting that $\KL{\Ber(\frac{1+u}{2})}{\Ber(1/2)} = \uMI(\log(1/u))$ and recalling that $\uMI(\tcritu) = \log(M)/d$, we conclude by a union bound over $\calC$ that with probability at least $1 - e^{-\Theta(\Lambda)}$ over the randomness of $\calC$, 
  \begin{equation}
    \max_{z\in \calC} \langle x,z\rangle \le de^{-\tcritu} + O(\Lambda)\,.
  \end{equation}
  Henceforth condition on this event.
  Then
  \begin{equation}
    T_t(x) \le de^{-\tcritu} - de^{-t} + O(\Lambda)\,.
  \end{equation}
  When $t \le \tcritu - 2w$, we have $de^{-\tcritu} - de^{-t} \le -b'dw$ for some constant $b'>0$. On the other hand, since $\Lambda\le c_{K,A}d$ gives $\Lambda\le(\sqrt{c_{K,A}}/A)\,dw$, the $O(\Lambda)$ term is at most $b''dw$ for some constant $b''\in(0,b')$ once $c_{K,A}$ is small enough. Taking $b_0=b'-b''$, the second claimed bound follows. Take $a_0$ small enough so that $b_0>4a_0$.
\end{proof}

\noindent The quantity $T_t(x)$ is therefore a good statistic for identifying the critical window. In practice, we will not have access to $m_t(x)$ or $s_t(x)$, but we can estimate $T_t(x)$ using estimated scores $\widehat{s}_t(x)$, which we will show are accurate enough to identify the critical window.

We first define the empirical version of $T_t(x)$. Let $\ell_t=\beta_t/(1-\beta_t)$. For a fixed initial point $z$, flipping one coordinate of $x$ changes $(K_t^{\sf U})^{\otimes d}(x\mid z)$ by a factor of either $\ell_t$ or $\ell_t^{-1}$. The score is a weighted average of these ratios, with weights given by the posterior distribution of $z$. So for every $t>0$, $x\in\calX$, and $i\in[d]$,
\begin{equation}
  \ell_t\le s_t(x)[i,-x_i]\le\ell_t^{-1}\,.
  \label{eq:uniform_score_bounds}
\end{equation}
In place of the Bregman projection from Section~\ref{sec:uniform-dtc}, here we simply \emph{clip} the entries of the score vector to this range to get the \emph{clipped approximate score}
\begin{equation}
  \widetilde s_t(x)[i,-x_i] \triangleq \min\bigl(\ell_t^{-1},\max(\ell_t,\widehat s_t(x)[i,-x_i])\bigr)\,.
  \label{eq:uniform_clipped_score}
\end{equation}
Define the following approximation to $T_t$ that the sampler can actually form using approximate score access:
\begin{equation}
  \widehat T_t(x)\triangleq \frac{1-e^{-2t}}{2e^{-t}}\sum_{i=1}^d\bigl(1-\widetilde s_t(x)[i,-x_i]\bigr)\,.
  \label{eq:uniform_empirical_statistic}
\end{equation}
The statistic $\widehat T_t(x)$ can be computed with only one approximate score oracle query at $x$. Here we show that $\widehat T_t(x)$ is a good estimate of $T_t(x)$. We will quantify this relative to the squared Euclidean error of the clipped score:
\begin{equation}
  F_t(x) \triangleq \sum_{i=1}^d\bigl(\widetilde s_t(x)[i,-x_i]-s_t(x)[i,-x_i]\bigr)^2\,.
  \label{eq:uniform_mse}
\end{equation}

\begin{lemma}\label{lem:uniform_score_truncation}
  Uniformly over $t\in I_0$ and $x\in\calX$, we have
  \begin{equation}
    |\widehat T_t(x)-T_t(x)|\le C\sqrt{dF_t(x)}\,,\qquad \EE_{X\sim q^\calC_t}F_t(X)\le C'\epsilon_{\sf unif}(t) \label{eq:uniform_clipped_error}
  \end{equation}
  for constants $C,C'>0$ depending only on $[\kappa_-,\kappa_+]$.
\end{lemma}

\noindent We defer the proof to Appendix~\ref{app:deferred-proofs}.

We record the following standard estimates for Hamming-sphere sizes and binomial point masses, which will be used repeatedly below.

\begin{lemma}\label{lem:hamming_sphere_estimates}
  For $1\le r\le d-1$ and $v=r/d$,
  \begin{equation}
    \frac{e^{d h_2(v)}}{d+1}\le\binom dr\le e^{d h_2(v)}\,.
    \label{eq:uniform_sphere_size}
  \end{equation}
  Consequently, for any $p\in(0,1)$,
  \begin{equation}
    \Pr{\Bin(d,p)=r}\ge\frac1{d+1}\exp\left(-d\,\KL{\Ber(v)}{\Ber(p)}\right)
    \ge\frac1{d+1}\exp\left(-\frac{d(v-p)^2}{p(1-p)}\right)\,.
    \label{eq:uniform_binomial_point_mass}
  \end{equation}
  Moreover, for all $v,p\in(0,1)$,
  \begin{equation}
    \KL{\Ber(v)}{\Ber(p)}\le\frac{(v-p)^2}{p(1-p)}\,.
    \label{eq:uniform_bernoulli_kl}
  \end{equation}
\end{lemma}

\noindent We defer the proof to Appendix~\ref{app:deferred-proofs}.

The following lemma allows us to control the likelihood ratio $L_t(x)$ near the critical time $\tcritu$.

\begin{lemma}\label{lem:uniform_density_lower_bound}
  For every $A,K,L>0$, there are constants $C_{A,K,L},c_{A,K,L},d_{A,K,L}>0$, depending only on $A,K,L$ and $[\kappa_-,\kappa_+]$, such that whenever $d\ge d_{A,K,L}$, $\log d\le\Lambda\le c_{A,K,L}d$, and $|t-\tcritu|\le Lw$, we have
  \begin{equation}
    \Pr[\calC]{L_t(x)<e^{-C_{A,K,L}\Lambda}}\le e^{-K\Lambda}\,,
    \label{eq:uniform_likelihood_ratio_bound}
  \end{equation}
  uniformly over $x\in\calX$.
\end{lemma}

\noindent We defer the proof to Appendix~\ref{app:deferred-proofs}.

We can now combine all of the ingredients above to prove that $\widehat T_t(x)$ can be used to identify the critical window.

\begin{proposition}\label{prop:uniform_score_separation}
  There are constants $0<a<b$, depending only on $[\kappa_-,\kappa_+]$, such that for every $K>0$ there is $A_K>0$ with the following property. For every $A\ge A_K$, there are $B_{K,A},c_{K,A},d_{K,A}>0$ such that the following holds whenever $d\ge d_{K,A}$, $\log d\le\Lambda\le c_{K,A}d$, $B\ge B_{K,A}$, and the score oracle satisfies Eq.~\eqref{eq:uniform_score_accuracy}:
  \begin{align}
    \Pr[\calC,x\sim\Unif(\calX)]{\widehat T_t(x)<-adw}&\le e^{-K\Lambda}\,, &&t\ge\tcritu+2w\,,\label{eq:uniform_high_test}\\
    \Pr[\calC,x\sim\Unif(\calX)]{\widehat T_t(x)>-bdw}&\le e^{-K\Lambda}\,, &&\tcritu-3w\le t\le\tcritu-2w\,.\label{eq:uniform_low_test}
  \end{align}
  uniformly over $t\in I_0$.
\end{proposition}

\begin{proof}
  Let $a_0,b_0$ be the constants in Proposition~\ref{prop:uniform_overlap_separation}, chosen so that $b_0>4a_0$, and set
  \begin{equation}
    a\triangleq 2a_0
    \qquad\text{and}\qquad
    b\triangleq b_0-a_0\,.
    \label{eq:uniform_approximate_thresholds}
  \end{equation}
  Then $0<a<b$.

  By Lemma~\ref{lem:uniform_score_truncation}, the event $|\widehat T_t-T_t|>a_0dw$ implies $F_t>c d w^2$, and its probability under the true noised law $q^\calC_t$ is at most $C e^{-B\Lambda}/(dw^2)$.

  \paragraph{High-noise regime ($t \ge \tcritu + 2w$).} Replacing $q^\calC_t$ by the uniform law $\Unif(\calX)$ adds at most $\TV(q^\calC_t,\Unif(\calX))$. Averaging over $\calC$ and using Proposition~\ref{prop:uniform_critical_window} with exponent $K+2$ therefore gives an exponentially small error probability for a uniform query as well. Therefore, Markov's inequality implies
  \begin{equation}
    \Pr[\calC,x\sim\Unif(\calX)]{|\widehat T_t(x)-T_t(x)|>a_0dw}
    \le \frac{C e^{-B\Lambda}}{dw^2}+e^{-(K+2)\Lambda}\,.
    \label{eq:uniform_high_approximation_error}
  \end{equation}
  Since $dw^2=A^2\Lambda\ge1$, choosing $B$ sufficiently large makes the right-hand side at most $e^{-(K+1)\Lambda}$ for all sufficiently large $d$.

  On the other hand, Proposition~\ref{prop:uniform_overlap_separation} with exponent $K+2$ gives
  \begin{equation}
    \Pr[\calC,x\sim\Unif(\calX)]{T_t(x)<-a_0dw}
    \le e^{-(K+2)\Lambda}\,.
    \label{eq:uniform_high_exact_error}
  \end{equation}
  If neither event in Eqs.~\eqref{eq:uniform_high_approximation_error} and~\eqref{eq:uniform_high_exact_error} occurs, then
  \begin{equation}
    \widehat T_t(x) \ge T_t(x)-a_0dw \ge -2a_0dw=-adw\,.
  \end{equation}
  Consequently,
  \begin{align}
    \Pr[\calC,x\sim\Unif(\calX)]{\widehat T_t(x)<-adw} &\le \Pr[\calC,x\sim\Unif(\calX)]{T_t(x)<-a_0dw}+\Pr[\calC,x\sim\Unif(\calX)]{|\widehat T_t(x)-T_t(x)|>a_0dw} \\
    &\le e^{-(K+2)\Lambda}+e^{-(K+1)\Lambda} \le e^{-K\Lambda}
  \end{align}
  for all sufficiently large $d$, proving Eq.~\eqref{eq:uniform_high_test}.

  \paragraph{Low-noise regime ($t\in[\tcritu-3w,\tcritu-2w]$).} Lemma~\ref{lem:uniform_density_lower_bound} (with $L=3$) implies $L_t(x)\ge e^{-C\Lambda}$, where $C=C_{A,K',3}$, with probability at least $1-e^{-K'\Lambda}$, where $K'$ can be chosen arbitrarily large. On the event $L_t(x)\ge e^{-C\Lambda}$, we have the pointwise bound $2^{-d}\mathbf \Bone[L_t(x)\ge e^{-C\Lambda}]\le e^{C\Lambda}q^\calC_t(x)$. Thus Markov's inequality gives
  \begin{equation}
    \Pr[\calC,x\sim\Unif(\calX)]{|\widehat T_t(x)-T_t(x)|>a_0dw}
    \le e^{-K'\Lambda}+\frac{C' e^{-(B-C)\Lambda}}{dw^2}\,.
    \label{eq:uniform_test_error_transfer}
  \end{equation}
  Choose $K'>K+2$ and then with $A$ fixed, choose $B$ large enough so that the last display is at most $e^{-(K+1)\Lambda}$.

  Proposition~\ref{prop:uniform_overlap_separation} with exponent $K+2$ gives that throughout this interval
  \begin{equation}
    \Pr[\calC,x\sim\Unif(\calX)]{T_t(x)>-b_0dw}
    \le e^{-(K+2)\Lambda}\,,
    \label{eq:uniform_low_exact_error}
  \end{equation}
  since every such $t$ satisfies $t\le\tcritu-2w$. If neither the event in Eq.~\eqref{eq:uniform_test_error_transfer} nor the event in Eq.~\eqref{eq:uniform_low_exact_error} occurs, then
  \begin{equation}
    \widehat T_t(x) \le T_t(x)+a_0dw \le -(b_0-a_0)dw=-bdw\,.
  \end{equation}
  Therefore,
  \begin{align}
    \Pr[\calC,x\sim\Unif(\calX)]{\widehat T_t(x)>-bdw} &\le \Pr[\calC,x\sim\Unif(\calX)]{T_t(x)>-b_0dw}+\Pr[\calC,x\sim\Unif(\calX)]{|\widehat T_t(x)-T_t(x)|>a_0dw} \\
    &\le e^{-(K+2)\Lambda}+e^{-(K+1)\Lambda} \le e^{-K\Lambda}
  \end{align}
  for all sufficiently large $d$. This proves Eq.~\eqref{eq:uniform_low_test}.
\end{proof}

\noindent From now on, fix $A\ge A_{12}$ and choose $B\ge B_{12,A}$. We may decrease $c_0$ and increase $d_0$ so that $c_0\le c_{12,A}$, $d_0\ge d_{12,A}$, and $w\le\min(\tau_-/16,1/16)$. We further decrease $c_0$ and increase $d_0$ below whenever necessary.

\begin{algorithm}
  \caption{Finding the uniform diffusion critical window}
  \label{alg:uniform_window_search}
  \KwInput{$d$, $[\kappa_-,\kappa_+]$, $w$, $\Lambda$, and query access to the score oracle $\widehat s$.}
  \KwOutput{An estimate $\widehat t$ of the critical time $\tcritu$.}
  Let $\mathcal T=\{t_1>t_2>\cdots>t_m\}$ be a decreasing grid covering $[\tau_--3w,\tau_++3w]$ with mesh $w/20$\;
  Set $R\leftarrow 2\lceil C\Lambda\rceil+1$\; 
  \For{$t\in\mathcal T$}{
      \For{$j\leftarrow 1$ \KwTo $R$}{
          Sample $X^{(j)}\sim\Unif(\calX)$\;
          Query $\widehat s_t(X^{(j)})$ and compute $\widehat T_t(X^{(j)})$ using Eqs.~\eqref{eq:uniform_clipped_score}--\eqref{eq:uniform_empirical_statistic}\;
      }
      \If{$\operatorname{median}\bigl\{\widehat T_t(X^{(1)}),\ldots,\widehat T_t(X^{(R)})\bigr\}<-\frac{a+b}{2}dw$}{
          \Return{$t$}
      }
  }
  \Return{$\tau_-$}\;
\end{algorithm}

We can now conclude the proof of the main result of this subsection, namely an algorithm for estimating the critical noise level $\tcritu$, essentially by scanning from high to low noise levels and stopping when the median of $\widehat{T}_t$ over enough random inputs becomes sufficiently negative (see Algorithm~\ref{alg:uniform_window_search} for details).

\begin{proposition}\label{prop:uniform_window_search}
  Under the hypotheses of Theorem~\ref{thm:uniform_random_measure_sampler}, Algorithm~\ref{alg:uniform_window_search} uses $O(\Lambda/w)$ score oracle queries. For at least a $1-\delta/4$ fraction of $\calC$, it satisfies
  \begin{equation}
    \Pr{\tcritu-3w\le\widehat t\le\tcritu+2w\mid \calC}\ge1-\epsilon/8\,.
    \label{eq:uniform_window_search_success}
  \end{equation}
\end{proposition}

\noindent We defer the proof to Appendix~\ref{app:deferred-proofs}.

\subsection{Sampling after locating the window}
\label{sec:uniform-window-sampling}

Having found the window, we initialize right at the upper edge of the window, using the uniform distribution, and approximately simulate the reverse process until we reach its lower edge. Our algorithm for simulation is a naive discretization; unlike the DTC-adaptive sampler in Section~\ref{sec:uniform-dtc}, we do not need to carefully tune the step sizes, and we use the clipped score in lieu of the Bregman projection.

First, we specialize the first part of Lemma~\ref{lem:uniform_score_to_marginals} to the binary alphabet setting that we consider here. 
For $0<s<t$, write $h=t-s$ and define the affine function
\begin{equation}
  p_h(r)=\frac{\beta_h}{\alpha_h}\bigl((1-\beta_h)r-\beta_h\bigr)\,.
  \label{eq:uniform_reverse_flip_map}
\end{equation}
Lemma~\ref{lem:uniform_score_to_marginals} implies that $p_h(s_t(y)[i,-y_i])$ is the probability that $(X_s)_i\neq y_i$ conditional on $X_t=y$. As in Section~\ref{sec:uniform-dtc}, in a single step of our sampler it will flip coordinate $i$ with probability
\begin{equation}
  \widehat p^i_{t\to s}(y)=p_h\bigl(\widetilde s_t(y)[i,-y_i]\bigr)\,,
  \label{eq:uniform_approx_flip_probability}
\end{equation}
independently over $i$. The only distinctions from Section~\ref{sec:uniform-dtc} are that here we use the clipped score $\widetilde s_t$, and because the alphabet is binary, we only have to specify a single probability for each coordinate. As before, we will denote the resulting approximate product reverse kernel by $\approxrevu_{t\to s}$.

\par\vspace{\baselineskip}\noindent Next we bound the sampling error coming from using $\widetilde s_t$ instead of $s_t$ in this product reverse kernel.

\begin{lemma}\label{lem:uniform_reverse_step}
  There are constants $C,h_0>0$, depending only on $[\kappa_-,\kappa_+]$, such that whenever $s,t\in I_0$ and $0<h=t-s\le h_0$, the probabilities in Eq.~\eqref{eq:uniform_approx_flip_probability} belong to $(0,1)$. Denoting $P^i_{t\to s}(y,\cdot)=\law(X_s^i\mid X_t=y)$ and $\widehat P^i_{t\to s}(y,\cdot)$ the Bernoulli law defined by Eq.~\eqref{eq:uniform_approx_flip_probability}, we have
  \begin{equation}
    \EE_{Y\sim q^\calC_t}\sum_{i=1}^d\KL{P^i_{t\to s}(Y)}{\widehat P^i_{t\to s}(Y)}\le Ch\epsilon_{\sf unif}(t)\,.
    \label{eq:uniform_reverse_step_kl}
  \end{equation}
\end{lemma}

\noindent We defer the proof to Appendix~\ref{app:deferred-proofs}.

Finally, we control the overall error incurred by our sampler over the course of crossing the critical window. Here we re-use some of the basic chain rule calculations from Section~\ref{sec:uniform-dtc}, but because we take steps of size $\asymp\epsilon^2/d$, we do not need to tune the step sizes so as to take advantage of telescoping.

\begin{proposition}\label{prop:uniform_window_traversal}
  Fix $t_-<t_+$ in $I_0$, let $\Delta=t_+-t_-$, and let $r_j=t_-+j\Delta/N$ for $0\le j\le N$. Suppose $h=\Delta/N\le h_0$, and let $\widehat R$ be the composition of the approximate product reverse kernels from $r_N$ down to $r_0$. If $\epsilon_{\sf unif}(r_j)\le\eta$ for $1\le j\le N$, then any initial distribution $\rho$ satisfies
  \begin{equation}
  \TV(\rho\widehat R,q^\calC_{t_-})\le\TV(\rho,q^\calC_{t_+})+C\sqrt{\frac{d\Delta}{N}+\Delta\eta}\,.
  \label{eq:uniform_traversal_tv}
  \end{equation}
\end{proposition}

\begin{proof}
  First initialize from $q^\calC_{t_+}$. Let $P$ be the true reverse path law on $(X_{r_N},\ldots,X_{r_0})$ and let $\widehat P$ be the path law obtained from the approximate product kernels. Write $R_j(y,\cdot)=\law(X_{r_j}\mid X_{r_{j+1}}=y)$ and $\widehat R_j(y,\cdot)=\prod_i\widehat P^i_{r_{j+1}\to r_j}(y,\cdot)$. For every $y$, the identity
  \begin{equation}
    \KL{R_j(y,\cdot)}{\widehat R_j(y,\cdot)}=\TC(X_{r_j}\mid X_{r_{j+1}}=y)+\sum_{i=1}^d\KL{P^i_{r_{j+1}\to r_j}(y,\cdot)}{\widehat P^i_{r_{j+1}\to r_j}(y,\cdot)}
  \end{equation}
  separates discretization error and score approximation error. The KL chain rule, Lemma~\ref{lem:uniform-tc-dtc} (with $s=r_j\ge t_-$), Lemma~\ref{lem:uniform_reverse_step}, and telescoping therefore give
  \begin{align}
    \KL{P}{\widehat P}
    &\le \frac{e^{2h}-1}{1-e^{-t_-}}\sum_{j=0}^{N-1}\bigl(\DTC(q^\calC_{r_j})-\DTC(q^\calC_{r_{j+1}})\bigr)+\sum_{j=0}^{N-1}Ch\epsilon_{\sf unif}(r_{j+1})\\
    &\le C\frac{\Delta}{N}\bigl(\DTC(q^\calC_{t_-})-\DTC(q^\calC_{t_+})\bigr)+C\frac{\Delta}{N}\sum_{j=0}^{N-1}\epsilon_{\sf unif}(r_{j+1})\\
    &\le C\frac{d\Delta}{N}+C\Delta\eta\,,
  \end{align}
  where $\DTC(q^\calC_{t_-})\le H(q^\calC_{t_-})\le d\log2$. Data processing inequality and Pinsker's inequality imply Eq.~\eqref{eq:uniform_traversal_tv} when the initial law is $q^\calC_{t_+}$. Replacing it by $\rho$ adds $\TV(\rho\widehat R,q^\calC_{t_+}\widehat R)\le\TV(\rho,q^\calC_{t_+})$.
\end{proof}

\noindent Below the window, the posterior on the codeword that would have generated the noisy iterate is typically concentrated on a single point of $\calC$. We can recover that point by taking the signs of the estimated posterior expectation. Use Lemma~\ref{lem:uniform_score_to_marginals} once more, now for the transition from $t$ to $0$, and define
\begin{equation}
  \widehat m_t(y)_i=y_i\left[1-2p_t\bigl(\widetilde s_t(y)[i,-y_i]\bigr)\right],\qquad
  \widehat z_t(y)=\bigl(\sgn(\widehat m_t(y)_1),\ldots,\sgn(\widehat m_t(y)_d)\bigr)\,,
  \label{eq:uniform_posterior_rounding}
\end{equation}
with $p_t$ given by Eq.~\eqref{eq:uniform_reverse_flip_map} with $h=t$.

\begin{lemma}\label{lem:uniform_posterior_rounding}
  For $t\in I_0$, $X_0\sim q^\calC$, and its forward observation $X_t$, the rule in Eq.~\eqref{eq:uniform_posterior_rounding} satisfies
  \begin{equation}
    \Pr{\widehat z_t(X_t)\neq X_0\mid \calC}\le5\err^{\sf U}_\calC(t)+C\epsilon_{\sf unif}(t)\,.
    \label{eq:uniform_posterior_rounding_bound}
  \end{equation}
  In particular, the law of $\widehat z_t(X_t)$ is within this distance of $q^\calC$ in total variation.
\end{lemma}

\begin{proof}
  Define a Bayes-optimal decoder (deterministic function) $z_*$: for each $y$, $z_*(y)$ is a point with largest posterior probability given $X_t=y$. Let $\pi_*(y)$ denote the corresponding posterior mass. The Bayes error is $\EE[1-\pi_*(X_t)\mid \calC]=\err^{\sf U}_\calC(t)$. Therefore by Markov's inequality,
  \begin{equation}
    \Pr{\pi_*(X_t)<3/4\mid \calC}\le4\err^{\sf U}_\calC(t)\,.
  \end{equation}
  On the complementary event, every coordinate of the exact posterior expectation has the sign of $z_*(X_t)$ and magnitude at least $1/2$.

  The map from a score to the corresponding posterior-mean coordinate in the first part of Eq.~\eqref{eq:uniform_posterior_rounding} is affine with bounded slope, since $I_0$ is compact. Lemma~\ref{lem:uniform_score_truncation} consequently gives
  \begin{equation}
  \EE_{X_t\sim q^\calC_t}\|\widehat m_t(X_t)-m_t(X_t)\|_2^2\le C\epsilon_{\sf unif}(t)\,.
  \end{equation}
  When $\pi_*(X_t)\ge3/4$, a wrong sign requires the error in at least one coordinate to have magnitude at least $1/2$. Markov's inequality bounds the probability of this event by $C\epsilon_{\sf unif}(t)$. Finally, $z_*(X_t)\neq X_0$ with probability $\err^{\sf U}_\calC(t)$. These three possible failures give Eq.~\eqref{eq:uniform_posterior_rounding_bound}. The last assertion follows by coupling $\widehat z_t(X_t)$ with $X_0$.
\end{proof}

\noindent We are now ready to prove the main result of the section.

\begin{algorithm}
  \caption{Sampling from the random empirical measure with uniform diffusion}
  \label{alg:uniform_random_measure_sampler}

  \KwInput{$d$, $[\kappa_-,\kappa_+]$, target accuracy $\epsilon$, failure probability $\delta$, and a score oracle satisfying Eq.~\eqref{eq:uniform_score_accuracy}.}
  \KwOutput{A sample from the approximate output law $\widehat q^\calC$.}
  Set $\Lambda$ as in Eq.~\eqref{eq:uniform_sampler_parameter} and $w$ as in Eq.~\eqref{eq:uniform_window_width}\;
  Run Algorithm~\ref{alg:uniform_window_search} to obtain $\widehat t$\;
  Set $t_+\leftarrow \widehat t+4w$, $t_-\leftarrow \widehat t-4w$, $N\leftarrow \left\lceil \frac{Cdw}{\epsilon^2}\right\rceil$, $h\leftarrow \frac{8w}{N}$\;
  Set $r_j\leftarrow t_-+jh,\qquad 0\le j\le N$\;
  Sample $Y_N\sim\Unif(\calX)$ independently of Algorithm~\ref{alg:uniform_window_search}\;
  \For{$j\leftarrow N$ \KwTo $1$}{
      Query $\widehat s_{r_j}(Y_j)$ and compute $\widetilde s_{r_j}(Y_j)$ using Eq.~\eqref{eq:uniform_clipped_score}\;
      \For{$i\leftarrow 1$ \KwTo $d$}{
          Set $\widehat p_{j,i}\leftarrow p_h\!\left(\widetilde s_{r_j}(Y_j)[i,-(Y_j)_i]\right)$\;
          Independently draw $B_{j,i}\sim\Ber(\widehat p_{j,i})$ and set $(Y_{j-1})_i\leftarrow(-1)^{B_{j,i}}(Y_j)_i$\;
      }
  }
  Query $\widehat s_{t_-}(Y_0)$ and compute $\widehat z_{t_-}(Y_0)$ using Eqs.~\eqref{eq:uniform_clipped_score}--\eqref{eq:uniform_posterior_rounding}\;
  \Return{$\widehat z_{t_-}(Y_0)$}\;
\end{algorithm}

\begin{proof}[Proof of Theorem~\ref{thm:uniform_random_measure_sampler}]
We analyze Algorithm~\ref{alg:uniform_random_measure_sampler}. Proposition~\ref{prop:uniform_critical_window}, with exponent $12$, and Markov's inequality imply
\begin{equation}
  \TV(q^\calC_{t_1},\Unif(\calX))\le\epsilon/16,
  \qquad\err^{\sf U}_\calC(t_2)\le\epsilon/80\,,
  \label{eq:uniform_good_endpoints}
\end{equation}
uniformly over $t_1\ge\tcritu+w$ and $t_2\le\tcritu-w$, outside an event over $\calC$ of probability at most $\delta/2$. Proposition~\ref{prop:uniform_window_search} removes a further event over $\calC$ of probability at most $\delta/4$ where the search of the critical window fails.

Now fix a $\calC$ outside these failure events. Conditional on this $\calC$, the window search succeeds with probability at least $1-\epsilon/8$.

Now suppose Algorithm~\ref{alg:uniform_window_search} succeeds, and condition on its returned value $\widehat t$. The successful search ensures that
\begin{equation}
  t_+\ge\tcritu+w,\qquad t_-\le\tcritu-w,\qquad t_+-t_-=8w\,.
\end{equation}
Both $t_+$ and $t_-$ belong to $I_0$ by the choice of $w$.

Now we proceed to analyze the law of $Y_0$ generated by applying approximate reverse kernels to the uniform sample $Y_N$ in Algorithm~\ref{alg:uniform_random_measure_sampler}. Since $N\ge Cdw/\epsilon^2$ and $\epsilon\le1/4$, we have $h=\frac{8w}{N}\le\frac{8\epsilon^2}{Cd}\le\frac{1}{2Cd}$. Thus, by increasing $d_0$ if necessary, we have $h\le h_0$, as required by Proposition~\ref{prop:uniform_window_traversal}. Therefore, Proposition~\ref{prop:uniform_window_traversal} gives
\begin{equation}
  \TV(\law(Y_0\mid \calC,\widehat t),q^\calC_{t_-})\le\epsilon/16+C\sqrt{\frac{8dw}{N}+8w e^{-B\Lambda}}\le\frac{3\epsilon}{16}\,,
\label{eq:uniform_reverse_error}
\end{equation}
after choosing $B$ and the constant in $N$ sufficiently large.

Finally consider the final recovery step. If the input law were exactly $q^\calC_{t_-}$, Lemma~\ref{lem:uniform_posterior_rounding} and Eq.~\eqref{eq:uniform_posterior_rounding_bound} would bound the distance of its output law from $q^\calC$ by $5\epsilon/80+C e^{-B\Lambda}\le\epsilon/8$. Combining this with Eq.~\eqref{eq:uniform_reverse_error} (after applying the same recovery rule) gives
\begin{equation}
  \TV(\law(\widehat z_{t_-}(Y_0)\mid \calC),q^\calC)\le\frac{5\epsilon}{16}\,,
\end{equation}
uniformly over every possible successful value of $\widehat t$.

Algorithm~\ref{alg:uniform_window_search} fails with probability at most $\epsilon/8$. Thus $\TV(\widehat q^\calC,q^\calC)\le7\epsilon/16<\epsilon$. The total failure probability over $\calC$ is at most $3\delta/4<\delta$.

Finally, the window search stage uses $O(\Lambda/w)$ queries, the reverse simulation stage uses $N=O(dw/\epsilon^2+1)$, and the final recovery stage uses one. Substituting $w=A\sqrt{\Lambda/d}$ gives the claimed query complexity.
\end{proof}

\section{\texorpdfstring{$\widetilde O(\sqrt d)$}{soft-O(sqrt(d))} upper bound for Gaussian diffusion}
\label{sec:gaussian-upper}

In this section, we give an algorithm for approximately sampling from random empirical measures using $\widetilde O(\sqrt d)$ queries to any sufficiently accurate Gaussian diffusion score oracle for such a distribution. This query complexity is also tight: a matching lower bound was proved in~\cite{xun2026query}.

\begin{theorem}[Gaussian diffusion with approximate scores]\label{thm:gaussian_random_measure_sampler}
  There are constants $B,C,c_0,d_0>0$, depending only on $[\kappa_-,\kappa_+]$, such that the following holds. Let $d\ge d_0$, let $\epsilon,\delta\in(0,1/4)$, and set
  \begin{equation}
    \Lambda=\log\frac{Cd}{\epsilon\delta}\,.
    \label{eq:gaussian_sampler_parameter}
  \end{equation}
  Suppose $\Lambda\le c_0d$.
  
  Given query access to any approximate Gaussian diffusion score oracle $(\widehat s_t)$ for $q^\calC$ with score error satisfying
  \begin{equation}
    \sup_{t\in I_0}\epsilon_{\sf gauss}(t)\le e^{-B\Lambda}= \left(\frac{\epsilon\delta}{Cd}\right)^B\,,
    \label{eq:gaussian_score_accuracy}
  \end{equation}
  there is an algorithm using at most $O(\sqrt{d\Lambda}/\epsilon^2)$ score oracle queries such that with probability at least $1-\delta$ over $\calC$ sampled uniformly from all $M$-element subsets of $\calX$, the output law of the algorithm $\widehat q^\calC$ satisfies $\TV(\widehat q^\calC,q^\calC)\le\epsilon$.

  In particular, for constant $\epsilon$ and any polynomially small $\delta=d^{-\Theta(1)}$, polynomially small score accuracy suffices, and the algorithm uses $O(\sqrt{d\log d})$ score queries.
\end{theorem}

\noindent For the rest of the section, $\Lambda$ has the value in Eq.~\eqref{eq:gaussian_sampler_parameter}. We may decrease $c_0$ and increase $d_0$ later when necessary so that $\log d\le\Lambda\le c_0d$.

Define the \emph{reference distribution}
\begin{equation}
  \nu_t\triangleq\law(e^{-t}U+\sigma_tG),\qquad U\sim\Unif(\calX),\quad G\sim N(0,I_d)\,.
  \label{eq:gaussian_reference_law}
\end{equation}
Thus $\nu_t$ is obtained by applying the same forward process to the uniform distribution on the whole cube. The algorithm can of course sample from $\nu_t$ without knowing $\calC$.

\par\vspace{\baselineskip}\noindent In Section~\ref{sec:gaussian-critical-window}, we quantify the critical window: above the window, $q^\calC_t$ is close to the reference distribution $\nu_t$, while below the window the original sample can be recovered from its noisy observation. In Section~\ref{sec:locate-gaussian-window}, we show how to locate this window using approximate score oracle queries at points sampled from $\nu_t$. Finally, in Section~\ref{sec:gaussian-window-sampling}, we describe how to simulate the reverse process across the window and then recover the original point, completing the proof of Theorem~\ref{thm:gaussian_random_measure_sampler}.

\subsection{Quantifying the critical window}
\label{sec:gaussian-critical-window}

We begin by comparing the density of an observation $X_t$ started from one point $Y\in\calX$ with the density obtained from a uniform initial point on the whole cube. Normalize the observation by its noise standard deviation:
\begin{equation}
  Z=X_t/\sigma_t=\lambda_tY+G,\qquad \lambda_t=e^{-t}/\sigma_t\,.
  \label{eq:gaussian_normalized_channel}
\end{equation}
For a general parameter $u>0$, write $p_u$ for the density of $uU+G$ with $U$ uniform on $\calX$. Write $\phi_d$ for the standard Gaussian density, so
\begin{equation}
  p_u(z)=\EE_{Y\sim\Unif(\calX)}\phi_d(z-uY)
  =\phi_d(z)e^{-u^2d/2}\prod_{i=1}^d\cosh(uz_i)\,.
  \label{eq:gaussian_reference_density}
\end{equation}
Consequently, define the log-likelihood ratio of a point $y\in\calX$ given an observation $z$ as
\begin{equation}
  \imath_t(y;z)=\log\frac{\phi_d(z-\lambda_ty)}{p_{\lambda_t}(z)}
  =\lambda_t\langle z,y\rangle-\sum_{i=1}^d\log\cosh(\lambda_tz_i)\,.
  \label{eq:gaussian_information_density}
\end{equation}
Intuitively, $\imath_t(y;z)$ measures how $y$ stands out from the uniform distribution given the observation $z$. Under the joint distribution of $Y\sim\Unif(\calX)$ and $Z=\lambda_tY+G$, this is a sum of $d$ independent terms $\lambda_tz_iy_i-\log\cosh(\lambda_tz_i)$, each with mean $\gMI(t)$. Below we will show that it concentrates around its mean, and that the mean is strictly decreasing with a controlled speed.

\begin{lemma}\label{lem:gaussian_information_concentration}
  Uniformly over $t\in I_0$ and $v\ge0$,
  \begin{equation}
    \Pr[Y\sim\Unif(\calX),\,Z=\lambda_tY+G]{|\imath_t(Y;Z)-d\gMI(t)|\ge v}
    \le2\exp\left(-cv^2/d\right)\,.
    \label{eq:gaussian_information_concentration}
  \end{equation}
  Moreover, there are constants $0<c_I<C_I$ such that
  \begin{equation}
    c_I\le-\gMI'(t)\le C_I,\qquad t\in I_0\,.
    \label{eq:gaussian_information_derivative}
  \end{equation}
\end{lemma}

\begin{proof}
  For the concentration bound, write
  \begin{equation}
    \imath_t(Y;Z)=\sum_{i=1}^d \xi_{t,i}\,,\qquad\xi_{t,i}=\lambda_t Z_iY_i-\log\cosh(\lambda_t Z_i)\,.
  \end{equation}
  Let $H_i=Y_iG_i$. Then $H_1,\ldots,H_d$ are independent standard Gaussians, and since $Z_iY_i=\lambda_t+H_i$ and $\cosh$ is even,
  \begin{equation}
    \xi_{t,i}=f_t(H_i)\,,\qquad f_t(h)=\lambda_t(\lambda_t+h)-\log\cosh\bigl(\lambda_t(\lambda_t+h)\bigr)\,.
  \end{equation}
  Moreover,
  \begin{equation}
    f_t'(h)=\lambda_t\left[1-\tanh\bigl(\lambda_t(\lambda_t+h)\bigr)\right]\,.
  \end{equation}
  Since $\lambda_t$ is bounded for $t\in I_0$,
  $|f_t'(h)|\le C$ uniformly in $t$ and $h$. Hence the function $(h_1,\ldots,h_d)\longmapsto\sum_{i=1}^d f_t(h_i)$ is $C\sqrt d$-Lipschitz. Gaussian concentration gives Eq.~\eqref{eq:gaussian_information_concentration}.

  To check the derivative, consider one coordinate. Write $V\sim\Unif(\{-1,1\}),W\sim N(0,1)$ independent, and $u=\lambda_t$. Then
  \begin{equation}
    \gMI(t)=I(V;uV+W)=u^2-\EE_W\log\cosh(u^2+uW)\,.
  \end{equation}
  Differentiate in $u$ and use Gaussian integration by parts,
  \begin{equation}
    \frac{\D}{\D u}I(V;uV+W)=2u-2u\,\EE\tanh H-u\,\EE\sech^2 H,\qquad H=u^2+uW\,.
  \end{equation}
  Since $\EE[V\mid uV+W]=\tanh(u(uV+W))$,
  \begin{equation}
    \EE\tanh H=\EE[V\,\EE[V\mid uV+W]]=\EE[\EE[V\mid uV+W]^2]=\EE\tanh^2 H\,.
  \end{equation}
  Substituting this identity into the derivative, we get
  \begin{equation}
    \frac{\D}{\D u}I(V;uV+W)=u\,\EE_W\sech^2(u^2+uW)\,.
  \end{equation}
  This derivative is continuous and strictly positive for $u>0$. Both $\lambda_t$ and $|\lambda'_t|$ are bounded above and away from zero on $I_0$, and $\lambda'_t<0$. The chain rule now gives Eq.~\eqref{eq:gaussian_information_derivative}.
\end{proof}

\noindent For a fixed $\calC$, denote the smallest probability of failing to recover $Y\sim q^\calC$ from $\lambda_tY+G$ by $\err^{\sf G}_\calC(t)$:
\begin{equation}
  \err^{\sf G}_\calC(t)\triangleq\inf_{\widehat Y}\Pr[Y\sim q^\calC]{\widehat Y(\lambda_tY+G)\neq Y\mid \calC}\,.
  \label{eq:gaussian_recovery_error}
\end{equation}
The following lemma quantifies our intuition for the log-likelihood ratio $\imath_t(Y;Z)$.

\begin{lemma}\label{lem:gaussian_finite_sample_bounds}
  For every $t>0$ and $\gamma>0$,
  \begin{align}
    \EE_\calC\TV(q^\calC_t,\nu_t)&\le\Pr{\imath_t(Y;Z)>\log M-\gamma}+\tfrac12e^{-\gamma/2},\label{eq:gaussian_output_bound}\\
    \EE_\calC\err^{\sf G}_\calC(t)&\le\Pr{\imath_t(Y;Z)\le\log M+\gamma}+e^{-\gamma}\,,\label{eq:gaussian_recovery_bound}
  \end{align}
  where both probabilities are taken with respect to $Y\sim\Unif(\calX)$ and $Z=\lambda_tY+G$.
\end{lemma}

\begin{proof}
  First consider bounding $\EE_\calC\TV(q^\calC_t,\nu_t)$. Fix $t$, and write $L_y(z)=\phi_d(z-\lambda_ty)/p_{\lambda_t}(z)$. For every $z$, the average of $L_y(z)$ over the whole cube is $1$. We have
  \begin{equation}
    \TV(q_t^\calC,\nu_t) = \frac12\int p_{\lambda_t}(z) \left| \frac1M\sum_{y\in \calC}L_y(z)-1 \right|\d z\,,
  \end{equation}
  and $M^{-1}\sum_{y\in \calC}L_y(z)$ is the empirical average of $L_y(z)$. To control its fluctuation, truncate the likelihood ratio at $R=Me^{-\gamma}$ and define
  \begin{equation}
    \widetilde L_y(z)=L_y(z)\bone{L_y(z)\le R},\qquad
    \widetilde L^c_y(z)=L_y(z)\bone{L_y(z)>R},\qquad
    \mu(z)=\EE_{Y\sim\Unif(\calX)}\widetilde L_Y(z)\,.
  \end{equation}
  We have the following decomposition:
  \begin{align}
    2\EE_\calC\TV(q^\calC_t,\nu_t)&\le\underbrace{\EE_\calC\int p_{\lambda_t}(z)\left|\frac1M\sum_{y\in \calC}\widetilde L_y(z)-\mu(z)\right|\d z}_{\circled{1}} + \underbrace{\EE_\calC\int p_{\lambda_t}(z)\Bigl(\frac1M\sum_{y\in \calC}\widetilde L^c_y(z)\Bigr)\d z}_{\circled{2}} \\
     &\qquad\qquad + \underbrace{ \EE_{Y\sim\Unif(\calX)} \int p_{\lambda_t}(z)L_Y(z)\bone{L_Y(z)>R} \d z}_{\circled{3}}\,.
  \end{align}
  Fix $z$ now. For $\circled{1}$, by Cauchy--Schwarz and the variance bound for sampling without replacement,
  \begin{equation}
    \EE_\calC\left|\frac1M\sum_{y\in \calC}\widetilde L_y(z)-\mu(z)\right|
    \le\sqrt{\frac1M\EE_{Y\sim\Unif(\calX)}\widetilde L_Y(z)^2}
    \le\sqrt{\frac RM}=e^{-\gamma/2}\,.
  \end{equation}
  Here we used $\widetilde L_Y^2\le R\widetilde L_Y$ and $\mu(z)\le1$. Integrating over $p_{\lambda_t}$ gives the same bound. Finally,
  \begin{equation}
    \circled{2}=\circled{3}=\EE_{Y\sim\Unif(\calX)}\int p_{\lambda_t}(z)L_Y(z)\bone{L_Y(z)>R}\,\d z
    =\Pr{\imath_t(Y;Z)>\log M-\gamma}\,,
  \end{equation}
  and Eq.~\eqref{eq:gaussian_output_bound} follows.

  Next we bound the Bayes recovery error. Given $\calC$ and $z$, return the unique point $y\in \calC$ with $L_y(z)>Me^\gamma$, if there is one. This rule succeeds if the initial point passes the test and all other points fail it. The probability that the initial point fails is the first term in Eq.~\eqref{eq:gaussian_recovery_bound}. We will bound the probability that any other point passes the test.

  Fix the initial point $y$ and the observation $z$, and consider the randomness of $\calC$. The remaining $M-1$ points are a uniform subset of $\calX\setminus\{y\}$. Since $\EE_{Y'\sim\Unif(\calX)}L_{Y'}(z)=1$, by Markov's inequality we have
  \begin{equation}
    \Pr[Y'\sim\Unif(\calX)]{L_{Y'}(z)>Me^{\gamma}}\le M^{-1}e^{-\gamma}\,.
  \end{equation}
  Removing $y$ from the sampling population increases this bound by at most a factor of $2^d/(2^d-1)<M/(M-1)$. A union bound over the remaining points gives the bound $e^{-\gamma}$, and Eq.~\eqref{eq:gaussian_recovery_bound} follows.
\end{proof}

\noindent Setting the parameters appropriately in the two estimates above characterizes the Gaussian critical window: above it the noised empirical measure is close to the reference distribution, while below it the original point can be recovered with high probability.

\begin{proposition}[Critical window]\label{prop:gaussian_critical_window}
  For every $K>0$ there is $A_K>0$ with the following property. For every $A\ge A_K$, there are $c_{K,A},d_{K,A}>0$, depending only on $A,K$, and $[\kappa_-,\kappa_+]$, such that the following holds whenever $d\ge d_{K,A}$ and $\log d\le\Lambda\le c_{K,A}d$. Set
  \begin{equation}
    w=A\sqrt{\frac{\Lambda}{d}}\,.
    \label{eq:gaussian_window_width}
  \end{equation}
  Then
  \begin{align}
    \EE_\calC\TV(q^\calC_t,\nu_t)&\le e^{-K\Lambda}\,,&&t\ge\tcritg+w\,,\label{eq:gaussian_above_window}\\
    \EE_\calC\err^{\sf G}_\calC(t)&\le e^{-K\Lambda}\,,&&0\le t\le\tcritg-w\,.\label{eq:gaussian_below_window}
  \end{align}
\end{proposition}

\begin{proof}
  By decreasing $c_{K,A}$ if necessary, we may assume $w\le\min\{\tau_-/2,1\}$, so that $\tcritg\pm w\in I_0$. Fix $\calC$. Applying an additional forward step to both distributions cannot increase $\TV(q^\calC_t,\nu_t)$. Also, an earlier observation can be used to simulate a later one, so $\err^{\sf G}_\calC(t)$ cannot decrease with $t$. It is therefore enough to prove the claims at $\tcritg+w$ and $\tcritg-w$.

  \paragraph{Above the critical window.} Set $t=\tcritg+w$. Since $d\,\gMI(\tcritg)=\log M$, Eq.~\eqref{eq:gaussian_information_derivative} gives $d\,\gMI(t)\le\log M-c_Idw$. Take $\gamma=c_Idw/2$ in Eq.~\eqref{eq:gaussian_output_bound}. The event $\imath_t(Y;Z)>\log M-\gamma$ requires a deviation for $\imath_t(Y;Z)$ of at least $c_Idw/2$ above the mean $d\,\gMI(t)$. Lemma~\ref{lem:gaussian_information_concentration} bounds its probability by $2e^{-cdw^2}$. The remaining term is at most $e^{-cdw}\lesssim e^{-cdw^2}$.

  \paragraph{Below the critical window.} Set $t=\tcritg-w$. Now $d\,\gMI(t)\ge\log M+c_Idw$. With the same choice $\gamma=c_Idw/2$ in Eq.~\eqref{eq:gaussian_recovery_bound}, the event $\imath_t(Y;Z)\le\log M+\gamma$ requires a deviation of at least $c_Idw/2$ below the mean $d\,\gMI(t)$. Its probability is again at most $2e^{-cdw^2}$, and the other term is at most $e^{-cdw}\lesssim e^{-cdw^2}$.

  \par\vspace{\baselineskip}\noindent Finally, since $dw^2=A^2\Lambda$, increasing $A$ gives the desired exponent. The monotonicity noted at the beginning extends the bounds to the stated time ranges.
\end{proof}

\noindent For the rest of the section, let $w=A\sqrt{\Lambda/d}$ as in Eq.~\eqref{eq:gaussian_window_width}, where $A>0$ is a constant to be fixed below. We may further decrease $c_0$ and increase $d_0$ later when necessary.

\subsection{Identifying the critical window}
\label{sec:locate-gaussian-window}

Having characterized the critical window, we now show how to identify its location using approximate Gaussian score oracle queries at samples from the reference distribution.

Define the density ratio of $q^\calC_t$ to $\nu_t$ as
\begin{equation}
  L_t(x)\triangleq\frac{q^\calC_t(x)}{\nu_t(x)}\,.
  \label{eq:gaussian_density_ratio}
\end{equation}

Given $x\in\RR^d$, define the posterior means under $q^\calC$ and under the uniform distribution on the whole cube by
\begin{equation}
  m_t(x)\triangleq\EE[X_0\mid X_t=x],\qquad
  m^0_t(x)_i\triangleq\tanh\left(\frac{e^{-t}x_i}{\sigma_t^2}\right),\quad i\in[d]\,.
  \label{eq:gaussian_reference_mean}
\end{equation}
For $z=x/\sigma_t$, consider the difference between the two expected overlaps $\langle z,m_t(x)\rangle$ and $\langle z,m^0_t(x)\rangle$:
\begin{equation}
  T_t(x)=\langle z,m_t(x)-m^0_t(x)\rangle\,.
  \label{eq:gaussian_test_quantity}
\end{equation}
Above the window, both distributions $q^\calC_t$ and $\nu_t$ assign nearly the same density to most observations drawn from $\nu_t$, and this gap is small. Below the window, a sample from $\nu_t$ is unlikely to be close to any point of $\calC$ since $\calC$ is exponentially sparse, so the gap is large. These two behaviors can be used to identify the critical window. We prove these two assertions before considering estimated scores.

As in the uniform diffusion proof, it is useful to express both the density ratio and the overlap difference through two log averages. For fixed $\calC,z$, define
\begin{equation}
  A_z(r)=\log\EE_{Y\sim q^\calC}e^{r\langle z,Y\rangle},\qquad
  B_z(r)=\log\EE_{Y\sim\Unif(\calX)}e^{r\langle z,Y\rangle}
  =\sum_{i=1}^d\log\cosh(rz_i)\,.
  \label{eq:gaussian_log_averages}
\end{equation}
At $u=\lambda_t$ and $z=x/\sigma_t$, Bayes' rule gives
\begin{equation}
  \log L_t(x)=A_z(u)-B_z(u),
  \qquad
  T_t(x)=A'_z(u)-B'_z(u)\,.
  \label{eq:gaussian_overlap_derivative}
\end{equation}
Thus $A_z-B_z$ compares the two densities, while its derivative is exactly the quantity used by our test.

Let $\pi_r(\cdot\mid z)$ be the conditional distribution of $U\sim\Unif(\calX)$ given $rU+G=z$. Then
\begin{equation}
  \pi_r(y\mid z)=2^{-d}e^{r\langle z,y\rangle-B_z(r)}
  =\prod_{i=1}^d\frac{e^{rz_i y_i}}{2\cosh(rz_i)}\,.
  \label{eq:gaussian_cube_posterior}
\end{equation}
Define the \emph{posterior information}
\begin{equation}
  \mathcal I_z(r)
  \triangleq\KL{\pi_r(\cdot\mid z)}{\Unif(\calX)}
  =d\log 2-H(\pi_r(\cdot\mid z))
  =rB'_z(r)-B_z(r)\,.
  \label{eq:gaussian_posterior_information}
\end{equation}
Indeed, Eq.~\eqref{eq:gaussian_cube_posterior} gives $\log(\pi_r(y\mid z)/2^{-d})=r\langle z,y\rangle-B_z(r)$, and taking the posterior expectation gives the last identity in Eq.~\eqref{eq:gaussian_posterior_information}. Thus $\mathcal I_z(r)$ is the information about a uniformly random cube point contained in the observation $z$. The exponential of the posterior entropy is exactly $2^de^{-\mathcal I_z(r)}$. When $r=\lambda_t$ and $Z\sim p_{\lambda_t}$, taking expectation of $\mathcal I_Z(\lambda_t)$ over $Z$ gives the mutual information $d\,\gMI(t)$. At the critical time $\tcritg$ this equals $\log M$, so the entropy scale equals $2^d/M$, the reciprocal of the fraction of the cube contained in $\calC$.

The derivatives of these quantities also have direct meanings. Under $\pi_r(\cdot\mid z)$, the mean and variance of the overlap $\langle z,U\rangle$ are $B'_z(r)$ and $B''_z(r)$, and hence
\begin{equation}
  \mathcal I'_z(r)=rB''_z(r)
  =r\Var[U\sim\pi_r(\cdot\mid z)]{\langle z,U\rangle}\,.
  \label{eq:gaussian_information_derivative_fixed_observation}
\end{equation}
In particular, the posterior information increases with the signal strength $r$.

\par\vspace{\baselineskip}\noindent We need two elementary facts about the reference law. The first ensures enough overlap between $p_u$ and $p_v$ if $u$ and $v$ are close.

\begin{lemma}\label{lem:gaussian_reference_comparison}
  For $u,v>0$ and every measurable set $A\subseteq\RR^d$,
  \begin{equation}
    p_u(A)\le e^{d(u-v)^2/2}\sqrt{p_v(A)}\,.
    \label{eq:gaussian_reference_comparison}
  \end{equation}
  More generally, let $\calC$ be a random subset independent of $(U,G)$, with the same law under parameters $u$ and $v$. For every measurable event $\mathcal E$ determined by $(\calC,Z)$,
  \begin{equation}
    \Pr[\calC,Z\sim p_u]{\mathcal E}
    \le
    e^{d(u-v)^2/2}
    \sqrt{\Pr[\calC,Z\sim p_v]{\mathcal E}}\,.
    \label{eq:gaussian_reference_comparison_random_subset}
  \end{equation}
\end{lemma}

\noindent We defer the proof to Appendix~\ref{app:deferred-proofs}.

The second fact says that, for a typical sample from $p_{\lambda_t}$, the posterior information is close to its mean and changes at a speed of order $d$.

\begin{lemma}\label{lem:gaussian_reference_observation}
  For every $K>0$, there are constants $C_K,c_K,d_K>0$, depending only on $[\kappa_-,\kappa_+]$ and $K$, such that the following holds whenever $d\ge d_K$, $\log d\le\Lambda\le c_Kd$, and $t\in I_0$. With probability at least $1-e^{-K\Lambda}$ for $Z\sim p_{\lambda_t}$,
  \begin{align}
    &\|Z\|_2^2\le Cd,\label{eq:gaussian_reference_norm}\\
    &cd\le B''_Z(r)\le Cd,\label{eq:gaussian_reference_variance}\\
    &\left|\mathcal I_Z(\lambda_t)-d\gMI(t)\right|\le C_K\sqrt{d\Lambda}\,.
    \label{eq:gaussian_posterior_information_concentration}
  \end{align}
  The bounds on $B''_Z$ hold uniformly for $r\in[\lambda_{\tau_++1}/2,2\lambda_{\tau_-/2}]$, and $c,C$ depend only on $[\kappa_-,\kappa_+]$.
\end{lemma}

\noindent We defer the proof to Appendix~\ref{app:deferred-proofs}.

We can now use the above ingredients to prove that the statistic $T_t$ can be used to distinguish whether one is above or below the critical window:

\begin{proposition}\label{prop:gaussian_overlap_separation}
  There are constants $a_0,b_0>0$, depending only on $[\kappa_-,\kappa_+]$, with $b_0>4a_0$, such that for every $K>0$ there is $A_K>0$ with the following property. For every $A\ge A_K$, there are $c_{K,A},d_{K,A}>0$, depending only on $A,K$, and $[\kappa_-,\kappa_+]$, such that the following holds whenever $d\ge d_{K,A}$ and $\log d\le\Lambda\le c_{K,A}d$:
  \begin{align}
    \Pr[\calC,X\sim\nu_t]{T_t(X)<-a_0dw}&\le e^{-K\Lambda},&&t\ge\tcritg+2w,\label{eq:gaussian_overlap_above}\\
    \Pr[\calC,X\sim\nu_t]{T_t(X)>-b_0dw}&\le e^{-K\Lambda},&&t\le\tcritg-2w\,,\label{eq:gaussian_overlap_below}
  \end{align}
  uniformly over $t\in I_0$.
\end{proposition}

\begin{proof}
  Write $u=\lambda_t$ and $z=X/\sigma_t$.

  \paragraph{High-noise regime ($t\ge\tcritg+2w$).} Choose a small constant $\theta>0$, set $t^+=t+\theta w$, and write $v=\lambda_{t^+}$. Since $-\lambda'_t$ is bounded above and below on the relevant compact interval, $u-v\asymp\theta w$. In particular, $t$ and $t^+$ are both above $\tcritg+w$. By Eq.~\eqref{eq:gaussian_overlap_derivative}, define the two density ratios, as functions of the common normalized observation $z$, by
  \begin{equation}
    L_u(z)=e^{A_z(u)-B_z(u)},\qquad L_v(z)=e^{A_z(v)-B_z(v)}\,.
  \end{equation}

  Proposition~\ref{prop:gaussian_critical_window} and Markov's inequality give
  \begin{equation}
    \Pr[\calC,Z\sim p_u]{|L_u(Z)-1|>1/2}
    +\Pr[\calC,Z\sim p_v]{|L_v(Z)-1|>1/2}
    \le Ce^{-cdw^2}\,.
  \end{equation}
  Lemma~\ref{lem:gaussian_reference_comparison} transfers the bound for $L_v$ from $p_v$ to $p_u$. Since $d(u-v)^2\le C\theta^2dw^2$, we may choose $\theta$ small enough that, under $\calC\sim\Unif\bigl(\binom{\calX}{M}\bigr)$ and $Z\sim p_u$, both $L_u(Z)$ and $L_v(Z)$ belong to $[1/2,3/2]$ except with probability $Ce^{-cdw^2}$. Lemma~\ref{lem:gaussian_reference_observation} gives $B''_Z(r)\le Cd$ uniformly for $r\in[v,u]$ outside another event of probability $e^{-K\Lambda}$, for sufficiently large $d$.

  Fix $\calC,z$ for which these three conclusions hold. Since $A_z$ is convex,
  \begin{align}
    T_t(x)&=A'_z(u)-B'_z(u)\notag\\
    &\ge\frac{A_z(u)-A_z(v)}{u-v}-B'_z(u)\notag\\
    &=\frac{\log L_u(z)-\log L_v(z)}{u-v}
      +\frac{B_z(u)-B_z(v)}{u-v}-B'_z(u)\notag\\
    &\ge-\frac{C\log3}{\theta w}-Cd\theta w\,.
    \label{eq:gaussian_overlap_above_estimate}
  \end{align}
  The last step uses $u-v\asymp\theta w$, Eq.~\eqref{eq:gaussian_reference_variance}, and Taylor's formula. For any prescribed $a_0>0$, further choose $\theta$ small enough so that the last term is at least $-a_0dw/2$, and then choose $A$ large enough so that the first term is at least $-a_0dw/2$, using $dw^2=A^2\Lambda$. Increasing $A$ also makes the total exceptional probability at most $e^{-K\Lambda}$.

  \paragraph{Low-noise regime ($t\le\tcritg-2w$).} Apply Lemma~\ref{lem:gaussian_reference_observation} with exponent $K+1$ to $Z\sim p_u$, and fix an observation $z$ satisfying Eq.~\eqref{eq:gaussian_posterior_information_concentration}. Since $\gMI(t)-\gMI(\tcritg)\ge2c_Iw$ and $d\gMI(\tcritg)=\log M$, choosing $A$ sufficiently large gives
  \begin{equation}
    \mathcal I_z(u)\ge\log M+c_Idw\,.
    \label{eq:gaussian_information_below_window}
  \end{equation}
  We show directly that $\calC$ is unlikely to contain even one point $y$ whose overlap $\langle z,y\rangle$ reaches the reference posterior mean $B'_z(u)$ up to a small gap $b_0dw$. Choose $b_0>0$ so that $ub_0\le c_I/2$ uniformly in $u$. By Markov's inequality and Eq.~\eqref{eq:gaussian_posterior_information},
  \begin{align}
    \Pr[Y\sim\Unif(\calX)]{\langle z,Y\rangle\ge B'_z(u)-b_0dw}
    &\le e^{-u(B'_z(u)-b_0dw)}\EE_{Y\sim\Unif(\calX)}e^{u\langle z,Y\rangle}\notag\\
    &=\exp\{-\mathcal I_z(u)+ub_0dw\}\notag\\
    &\le M^{-1}e^{-c_Idw/2}\,.
    \label{eq:gaussian_few_large_overlaps}
  \end{align}
  By a union bound over the $M$ points in $\calC$, except with probability $e^{-c_Idw/2}$, every $y\in \calC$ has overlap less than $B'_z(u)-b_0dw$. Their posterior-weighted average is $A'_z(u)$, so on this event
  \begin{equation}
    T_t(x)=A'_z(u)-B'_z(u)\le-b_0dw\,.
  \end{equation}
  Adding the probability $e^{-(K+1)\Lambda}$ that Eq.~\eqref{eq:gaussian_information_below_window} fails, and choosing $A$ large enough that $e^{-c_Idw/2}\le e^{-(K+1)\Lambda}$, which is possible since $dw\ge dw^2=A^2\Lambda$, we obtain Eq.~\eqref{eq:gaussian_overlap_below}. Finally, choose $a_0$ in the first part small enough that $b_0>4a_0$.
\end{proof}

\noindent We now replace the posterior mean by the estimate obtained from the score. Define its clipped version coordinatewise by
\begin{equation}
  \widetilde m_t(x)_i
  \triangleq\min\left\{1,\max\left\{-1,e^tx_i+e^t\sigma_t^2\widehat s_t(x)_i\right\}\right\},
  \qquad i\in[d]\,.
  \label{eq:gaussian_clipped_mean}
\end{equation}
Using this vector, define the clipped approximate score and the statistic available to the algorithm by
\begin{equation}
  \widetilde s_t(x)\triangleq\frac{e^{-t}\widetilde m_t(x)-x}{\sigma_t^2},
  \qquad
  \widehat T_t(x)\triangleq\langle x/\sigma_t,\widetilde m_t(x)-m^0_t(x)\rangle\,.
  \label{eq:gaussian_clipped_score_and_statistic}
\end{equation}
Both quantities can be computed from one approximate score oracle query at $x$.

\par\vspace{\baselineskip}\noindent The following lemma shows that the estimated overlap difference $\widehat T_t$ is indeed a good approximation of the true overlap difference $T_t$.

\begin{lemma}\label{lem:gaussian_mean_error}
  For every fixed $\calC$ and $t\in I_0$,
  \begin{equation}
    \EE_{q^\calC_t}\|\widetilde m_t-m_t\|_2^2\le C\epsilon_{\sf gauss}(t)^2,\qquad
    \EE_{q^\calC_t}\|\widetilde s_t-s_t\|_2^2\le C\epsilon_{\sf gauss}(t)^2\,.
    \label{eq:gaussian_clipping_error}
  \end{equation}
  Moreover, whenever $\|x/\sigma_t\|_2\le C_1\sqrt d$,
  \begin{equation}
    |\widehat T_t(x)-T_t(x)|\le C_1\sqrt d\,\|\widetilde m_t(x)-m_t(x)\|_2\,.
    \label{eq:gaussian_overlap_error}
  \end{equation}
\end{lemma}

\noindent We defer the proof to Appendix~\ref{app:deferred-proofs}.

The score assumption controls errors on observations from $q^\calC_t$, while our queries use $\nu_t$. Above the window, total variation already compares the two distributions. Near the window, we need the following weaker assertion: a sample from $\nu_t$ is unlikely to have extremely small density under $q^\calC_t$.

\begin{lemma}\label{lem:gaussian_density_near_transition}
  For every $A,L,K>0$, there are constants $C_{A,L,K},c_{A,L,K},d_{A,L,K}>0$, depending only on $A,L,K$ and $[\kappa_-,\kappa_+]$, such that whenever $d\ge d_{A,L,K}$ and $\log d\le\Lambda\le c_{A,L,K}d$,
  \begin{equation}
    \Pr[\calC,X\sim\nu_t]{L_t(X)<e^{-C_{A,L,K}\Lambda}}\le e^{-K\Lambda}\,,
    \label{eq:gaussian_density_near_transition}
  \end{equation}
  uniformly over $|t - \tcritg| \le Lw$.
\end{lemma}

\noindent We defer the proof to Appendix~\ref{app:deferred-proofs}.

The above lemmas allow us to use the estimated overlap difference $\widehat T_t$ to identify the critical window.

\begin{proposition}\label{prop:gaussian_score_separation}
  There are constants $0<a<b$, depending only on $[\kappa_-,\kappa_+]$, such that for every $K>0$ there is $A_K>0$ with the following property. For every $A\ge A_K$, there are $B_{K,A},c_{K,A},d_{K,A}>0$ such that the following holds whenever $d\ge d_{K,A}$, $\log d\le\Lambda\le c_{K,A}d$, $B\ge B_{K,A}$, and the score oracle satisfies Eq.~\eqref{eq:gaussian_score_accuracy}:
  \begin{align}
    \Pr[\calC,X\sim\nu_t]{\widehat T_t(X)<-adw}&\le e^{-K\Lambda},&&t\ge\tcritg+2w,\label{eq:gaussian_score_above}\\
    \Pr[\calC,X\sim\nu_t]{\widehat T_t(X)>-bdw}&\le e^{-K\Lambda},&&\tcritg-3w\le t\le\tcritg-2w\,,\label{eq:gaussian_score_below}
  \end{align}
  uniformly over $t\in I_0$.
\end{proposition}

\begin{proof}
  Take $a_0,b_0$ from Proposition~\ref{prop:gaussian_overlap_separation}, and put $a=2a_0$ and $b=b_0-a_0$. These satisfy $a<b$. We will show that replacing $T_t$ by $\widehat T_t$ changes it by at most $a_0dw$ with sufficiently high probability in each of the two ranges.

  Fix $\calC$ and $t$, and let $c>0$ be a small constant to be fixed below. By Lemma~\ref{lem:gaussian_mean_error} and Markov's inequality,
  \begin{equation}
    \Pr[X\sim q^\calC_t]{\|\widetilde m_t(X)-m_t(X)\|_2>c\sqrt d\,w}
    \le\frac{Ce^{-2B\Lambda}}{dw^2}\,.
    \label{eq:gaussian_mean_error_probability}
  \end{equation}
  Choose $c$ small enough that Eq.~\eqref{eq:gaussian_overlap_error} bounds the overlap error by $a_0dw$ whenever this event does not occur and $\|X/\sigma_t\|_2\le C\sqrt d$. The latter condition fails under $\nu_t$ with probability at most $e^{-(K+3)\Lambda}$, by Lemma~\ref{lem:gaussian_reference_observation}.

  \paragraph{High-noise regime ($t\ge\tcritg+2w$).} Replacing $q^\calC_t$ by $\nu_t$ in Eq.~\eqref{eq:gaussian_mean_error_probability} adds at most $\TV(q^\calC_t,\nu_t)$. Average over $\calC$ and apply Proposition~\ref{prop:gaussian_critical_window} with exponent $K+3$. Together with the $\|X/\sigma_t\|_2\le C\sqrt d$ probability bound, this gives
  \begin{equation}
    \Pr[\calC,X\sim\nu_t]{|\widehat T_t(X)-T_t(X)|>a_0dw}
    \le\frac{Ce^{-2B\Lambda}}{dw^2}+2e^{-(K+3)\Lambda}\,.
  \end{equation}
  Increasing $B$ makes this at most $e^{-(K+1)\Lambda}$. Outside this event and the event in Eq.~\eqref{eq:gaussian_overlap_above}, we have $\widehat T_t\ge-2a_0dw=-adw$.

  \paragraph{Low-noise regime ($\tcritg-3w\le t\le\tcritg-2w$).} We apply Lemma~\ref{lem:gaussian_density_near_transition} with $L=3$ and exponent $K+3$, and denote its density exponent by $C_1$. On the event $L_t(X)=q^\calC_t(X)/\nu_t(X)\ge e^{-C_1\Lambda}$, integration under $\nu_t$ is bounded by $e^{C_1\Lambda}$ times integration under $q^\calC_t$. Then following the same steps as in the high-noise regime, we obtain
  \begin{equation}
    \Pr[\calC,X\sim\nu_t]{|\widehat T_t(X)-T_t(X)|>a_0dw}
    \le\frac{Ce^{-(2B-C_1)\Lambda}}{dw^2}+2e^{-(K+3)\Lambda}\,.
  \end{equation}
  Choose $B$ after $C_1$ so that this is at most $e^{-(K+1)\Lambda}$. Outside this event and the event in Eq.~\eqref{eq:gaussian_overlap_below}, we have $\widehat T_t\le-(b_0-a_0)dw=-bdw$.
  
  \par\vspace{\baselineskip}\noindent Finally, applying Proposition~\ref{prop:gaussian_overlap_separation} with exponent $K+2$ and adding the error probabilities proves both assertions. Here $A_K$ is the larger of the thresholds required by Proposition~\ref{prop:gaussian_critical_window} at exponent $K+3$ and by Proposition~\ref{prop:gaussian_overlap_separation} at exponent $K+2$. For every fixed $A\ge A_K$, apply Lemma~\ref{lem:gaussian_density_near_transition} with $L=3$ and exponent $K+3$, choose $B_{K,A}$ sufficiently large in terms of its density exponent, and then choose $c_{K,A}$ sufficiently small and $d_{K,A}$ sufficiently large so that all three results apply.
\end{proof}

\noindent From now on, fix $A\ge A_{12}$ and choose $B\ge B_{12,A}$. We may decrease $c_0$ and increase $d_0$ so that $c_0\le c_{12,A}$, $d_0\ge d_{12,A}$, and $w\le\min{\tau_-/20,1/20}$. We further decrease $c_0$ and increase $d_0$ below whenever necessary.

The two bounds tell us how to search. We test decreasing times using samples from the reference distribution at each time. The test should not stop above $\tcritg+2w$, and it should stop by the time it reaches $[\tcritg-3w,\tcritg-2w]$. Its behavior elsewhere is irrelevant.

\begin{algorithm}
  \caption{Finding the Gaussian diffusion critical window}
  \label{alg:gaussian_window_search}
  \KwInput{$d$, $[\kappa_-,\kappa_+]$, $w$, $\Lambda$, and query access to the score oracle $\widehat s$.}
  \KwOutput{An estimate $\widehat t$ of the critical time $\tcritg$.}
  Let $\mathcal T=\{t_1>t_2>\cdots>t_m\}$ be a decreasing grid covering $[\tau_--3w,\tau_++3w]$ with mesh $w/20$\;
  Set $R\leftarrow2\lceil C\Lambda\rceil+1$\;
  \For{$t\in\mathcal T$}{
    \For{$j\leftarrow1$ \KwTo $R$}{
      Draw independent $U^{(j)}\sim\Unif(\calX)$ and $G^{(j)}\sim N(0,I_d)$\;
      Set $X^{(j)}\leftarrow e^{-t}U^{(j)}+\sigma_tG^{(j)}$\;
      Query $\widehat s_t(X^{(j)})$ and compute $\widehat T_t(X^{(j)})$ using Eqs.~\eqref{eq:gaussian_clipped_mean}--\eqref{eq:gaussian_clipped_score_and_statistic}\;
    }
    \If{$\operatorname{median}\bigl\{\widehat T_t(X^{(1)}),\ldots,\widehat T_t(X^{(R)})\bigr\}<-(a+b)dw/2$}{
      \Return{$t$}\;
    }
  }
  \Return{$\tau_-$}\;
\end{algorithm}

\noindent We can now conclude the main result of this subsection: scanning from high to low noise and stopping when the median statistic becomes sufficiently negative locates the critical window.

\begin{proposition}\label{prop:gaussian_window_search}
  Under the hypotheses of Theorem~\ref{thm:gaussian_random_measure_sampler}, Algorithm~\ref{alg:gaussian_window_search} uses $O(\Lambda/w)$ score oracle queries. For at least a $1-\delta/4$ fraction of $\calC$,
  \begin{equation}
    \Pr{\tcritg-3w\le\widehat t\le\tcritg+2w\mid \calC}\ge1-\epsilon/8\,.
    \label{eq:gaussian_window_search_success}
  \end{equation}
\end{proposition}

\noindent We defer the proof to Appendix~\ref{app:deferred-proofs}.

\subsection{Sampling after locating the window}
\label{sec:gaussian-window-sampling}

Once the search has returned $\widehat t$, we will start reverse simulation at $t_+=\widehat t+4w$ and stop at $t_-=\widehat t-4w$. At each step, we hold the estimated score fixed and solve the remaining linear SDE exactly. If a step goes from time $r$ to time $r-h$, its update is
\begin{equation}
  Y_{\rm next}=e^hY+2(e^h-1)\widetilde s_r(Y)+\sqrt{e^{2h}-1}\,G,
  \qquad G\sim N(0,I_d)\,.
  \label{eq:gaussian_reverse_step}
\end{equation}
We need to bound the error from holding the score fixed for time $h$. A short interval alone is not enough: near the critical time, the posterior can change rapidly. The useful fact is that the total increase of its variance is at most $d$. We will charge the error on each step to the increase during that step and then sum these increases.

For a fixed $\calC$, write
\begin{equation}
  \Sigma_t(x)=\Cov(X_0\mid X_t=x),\qquad
  V(t)=\EE_{X_t\sim q^\calC_t}\Tr\Sigma_t(X_t)\,.
  \label{eq:gaussian_posterior_variance}
\end{equation}
Thus $V(t)$ is the mean squared error of the posterior mean. The following identity makes the preceding plan precise.

\begin{lemma}\label{lem:gaussian_posterior_variance}
  Fix $\calC$. For $t>0$,
  \begin{equation}
    V'(t)=2e^{-2t}\sigma_t^{-4}\EE_{q^\calC_t}\Tr(\Sigma_t^2)\ge0,
    \qquad 0\le V(t)\le d\,.
    \label{eq:gaussian_posterior_variance_evolution}
  \end{equation}
\end{lemma}

\noindent We defer the proof to Appendix~\ref{app:deferred-proofs}.

We next bound the change of the score along one exact reverse step. Stating this separately keeps the time-discretization calculation apart from the score-estimation calculation.

\begin{lemma}\label{lem:gaussian_score_change}
  Fix $\calC$ and $r-h,r\in I_0$, where $0<h\le1$. Let $(Y_u)_{0\le u\le h}$ follow the exact reverse process from $q^\calC_r$ to $q^\calC_{r-h}$. Then
  \begin{equation}
    \int_0^h\EE\|s_{r-u}(Y_u)-s_r(Y_0)\|_2^2\,\d u
    \le C\left(dh^2+h(V(r)-V(r-h))\right)\,.
    \label{eq:gaussian_score_change}
  \end{equation}
\end{lemma}

\noindent We defer the proof to Appendix~\ref{app:deferred-proofs}.

We now combine the preceding score-change estimate with Girsanov's theorem on each grid interval, and then sum the errors over the grid.

\begin{proposition}\label{prop:gaussian_short_reverse_interval}
  Fix $\calC$ and a grid $t_-=r_0<r_1<\cdots<r_N=t_+$ in $I_0$. Put $h_j=r_j-r_{j-1}$, and suppose $\max_jh_j\le h_0$, for a sufficiently small constant $h_0>0$ depending only on $I_0$. Initialize Eq.~\eqref{eq:gaussian_reverse_step} from $q^\calC_{t_+}$, and let $\widehat q^\calC_{t_-}$ be its endpoint law. Then
  \begin{equation}
    \KL{q^\calC_{t_-}}{\widehat q^\calC_{t_-}}
    \le C\sum_{j=1}^N\bigg(dh_j^2+h_j(V(r_j)-V(r_{j-1}))
    +h_j\EE_{q^\calC_{r_j}}\|\widetilde s_{r_j}-s_{r_j}\|_2^2\bigg)\,.
    \label{eq:gaussian_short_reverse_interval}
  \end{equation}
  In particular, on an equal grid with $\Delta=t_+-t_-$,
  \begin{equation}
    \KL{q^\calC_{t_-}}{\widehat q^\calC_{t_-}}
    \le C\left(\frac{d\Delta^2}{N}+\frac{\Delta}{N}(V(t_+)-V(t_-))
    +\sum_{j=1}^N\frac{\Delta}{N}\EE_{q^\calC_{r_j}}\|\widetilde s_{r_j}-s_{r_j}\|_2^2\right)\,.
    \label{eq:gaussian_equal_grid_error}
  \end{equation}
\end{proposition}

\begin{proof}
  Fix $\calC$ throughout. We compare the conditional path laws on each interval before summing over the grid.

  \paragraph{One grid interval.} Fix a grid interval $[r-h,r]$ and condition on $Y_0=y$. Let $\mathsf P_{r,h}^y$ denote the path law of the exact reverse process
  \begin{equation}
    \d Y_u=\bigl(Y_u+2s_{r-u}(Y_u)\bigr)\,\d u+\sqrt2\,\d B_u,\qquad Y_0=y\,,
  \end{equation}
  and let $\widetilde{\mathsf P}_{r,h}^y$ denote the path law of
  \begin{equation}
    \d \widetilde Y_u=\bigl(\widetilde Y_u+2\widetilde s_r(y)\bigr)\,\d u+\sqrt2\,\d B_u,\qquad \widetilde Y_0=y\,.
  \end{equation}
  The time-$h$ endpoint of the latter process is exactly the update in Eq.~\eqref{eq:gaussian_reverse_step}.

  By Girsanov's theorem, we have
  \begin{equation}
    \KL{\mathsf P_{r,h}^y}{\widetilde{\mathsf P}_{r,h}^y}
    =\frac14\EE_{\mathsf P_{r,h}^y}\int_0^h\left\|2s_{r-u}(Y_u)-2\widetilde s_r(y)\right\|_2^2\,\d u
    =\int_0^h\EE\!\left[\left\|s_{r-u}(Y_u)-\widetilde s_r(y)\right\|_2^2\,\middle|\,Y_0=y\right]\d u\,.
  \end{equation}
  Averaging over $Y_0\sim q_r^\calC$ therefore gives
  \begin{equation}
    \EE_{Y_0\sim q_r^\calC}\KL{\mathsf P_{r,h}^{Y_0}}{\widetilde{\mathsf P}_{r,h}^{Y_0}}
    =\int_0^h\EE\left\|s_{r-u}(Y_u)-\widetilde s_r(Y_0)\right\|_2^2\,\d u\,.
    \label{eq:gaussian_step_path_entropy}
  \end{equation}

  \paragraph{The change-of-measure condition.} We briefly verify Novikov's condition
  \begin{equation}
    \EE\exp\left(\int^h_0\|s_{r-u}(Y_u)-\widetilde s_r(Y_0)\|_2^2\,\d u\right)<\infty
    \label{eq:gaussian_novikov_condition}
  \end{equation}
  used above in Girsanov's theorem. By Tweedie's formula and the clipping in Eqs.~\eqref{eq:gaussian_clipped_mean}--\eqref{eq:gaussian_clipped_score_and_statistic},
  \begin{equation}
    \|s_t(x)\|_2+\|\widetilde s_t(x)\|_2
    \le C(\|x\|_2+\sqrt d),
    \qquad t\in I_0\,.
  \end{equation}
  Hence, along the exact reverse process,
  \begin{equation}
    \|s_{r-u}(Y_u)-\widetilde s_r(Y_0)\|_2^2
    \le C\bigl(\|Y_u\|_2^2+\|Y_0\|_2^2+d\bigr)\,.
    \label{eq:gaussian_novikov_growth}
  \end{equation}
  Under this process $Y_u\sim q^\calC_{r-u}$, and every $q^\calC_t$, $t\in I_0$, is a mixture of Gaussians
  $N(e^{-t}y,\sigma_t^2I_d)$, $y\in\calC$.
  Since their variances are uniformly bounded and $\|e^{-t}y\|_2\le\sqrt d$, there is a constant $c>0$, depending only on $I_0$, such that
  \begin{equation}
    \sup_{\calC,t\in I_0}\EE_{X\sim q^\calC_t}e^{c\|X\|_2^2}<\infty\,.
    \label{eq:gaussian_quadratic_exponential_moment}
  \end{equation}
  Indeed, this follows from the quadratic exponential moment of a Gaussian, which is finite for $c<\frac{1}{2\sup_{t\in I_0}\sigma_t^2}$ after decreasing $c$ if necessary.

  Moreover, Jensen's inequality gives
  \begin{equation}
    \exp\left(C\int_0^h\|Y_u\|_2^2\,\d u\right)
    \le \frac1h\int_0^h e^{Ch\|Y_u\|_2^2}\,\d u\,.
  \end{equation}
  Combining this with Eq.~\eqref{eq:gaussian_quadratic_exponential_moment} and Cauchy--Schwarz, we get
  \begin{equation}
    \EE\exp\left(
    C\int_0^h
    \bigl(\|Y_u\|_2^2+\|Y_0\|_2^2+d\bigr)\,\d u
    \right)<\infty
  \end{equation}
  whenever $h\le h_0$ for a sufficiently small constant $h_0>0$ depending only on $I_0$. By Eq.~\eqref{eq:gaussian_novikov_growth}, Novikov's condition follows.

  \paragraph{Bounding the one-step error.} We now return to Eq.~\eqref{eq:gaussian_step_path_entropy}. Adding and subtracting $s_r(Y_0)$, we separate the error caused by freezing the exact score from the error in the score estimate:
  \begin{equation}
    \int_0^h\EE\left\|s_{r-u}(Y_u)-\widetilde s_r(Y_0)\right\|_2^2\,\d u
    \le 2\int_0^h\EE\left\|s_{r-u}(Y_u)-s_r(Y_0)\right\|_2^2\,\d u
    +2h\,\EE_{q_r^\calC}\left\|s_r-\widetilde s_r\right\|_2^2\,.
  \end{equation}
  By Lemma~\ref{lem:gaussian_score_change},
  \begin{equation}
    \int_0^h\EE\left\|s_{r-u}(Y_u)-s_r(Y_0)\right\|_2^2\,\d u
    \le C\bigl(dh^2+h(V(r)-V(r-h))\bigr)\,.
  \end{equation}
  Therefore we get the following bound on the KL divergence between the exact and approximate path laws on a short interval:
  \begin{equation}
    \EE_{Y_0\sim q_r^\calC}\KL{\mathsf P_{r,h}^{Y_0}}{\widetilde{\mathsf P}_{r,h}^{Y_0}}
    \le C\bigl(dh^2+h(V(r)-V(r-h))\bigr)+2h\,\EE_{q_r^\calC}\left\|s_r-\widetilde s_r\right\|_2^2\,.
  \end{equation}

  \paragraph{Summing over the whole grid.} Let $\mathsf P^{\rm path}$ and $\widetilde{\mathsf P}^{\rm path}$ denote the path laws obtained by concatenating, respectively, the exact and approximate reverse processes over the grid intervals. The chain rule implies
  \begin{equation}
    \KL{\mathsf P^{\rm path}}{\widetilde{\mathsf P}^{\rm path}}
    =\sum_{j=1}^{N}\EE_{\mathsf P^{\rm path}}
    \KL{\mathsf P_{r_j,h_j}^{Y_{r_j}}}{\widetilde{\mathsf P}_{r_j,h_j}^{Y_{r_j}}}\,.
  \end{equation}
  Under the exact reverse law, $Y_{r_j}\sim q_{r_j}^\calC$. Applying the preceding one-step bound on each interval therefore yields
  \begin{equation}
    \KL{\mathsf P^{\rm path}}{\widetilde{\mathsf P}^{\rm path}}
    \le C\sum_{j=1}^{N}\!\left[d h_j^2+h_j\bigl(V(r_j)-V(r_{j-1})\bigr)\right]
    +2\sum_{j=1}^{N}h_j\,\EE_{q_{r_j}^\calC}\left\|s_{r_j}-\widetilde s_{r_j}\right\|_2^2\,.
  \end{equation}
  The endpoint is a measurable function of the full path, so the data processing inequality gives the same upper bound for $\KL{q^\calC_{t_-}}{\widehat q^\calC_{t_-}}$. This proves Eq.~\eqref{eq:gaussian_short_reverse_interval}.

  For an equal grid, $h_j=\Delta/N$ with $\Delta=t_+-t_-$. Hence
  \begin{equation}
    \sum_{j=1}^{N}d h_j^2=\frac{d\Delta^2}{N},
    \qquad
    \sum_{j=1}^{N}h_j\bigl(V(r_j)-V(r_{j-1})\bigr)
    =\frac{\Delta}{N}\bigl(V(t_+)-V(t_-)\bigr)\,.
  \end{equation}
  Substituting these identities into the preceding bound gives Eq.~\eqref{eq:gaussian_equal_grid_error}.
\end{proof}

\noindent At the lower endpoint, we return the coordinatewise signs of $\widetilde m_t(x)$. This works for a reason stronger than a coordinatewise error estimate: most approximate samples below the window place almost all posterior mass on a single point of $\calC$.

\begin{lemma}\label{lem:gaussian_posterior_rounding}
  Fix $\calC$ and define $D_t(x)\triangleq\sgn(\widetilde m_t(x))$. If $X_0\sim q^\calC$ and $X_t$ is obtained from $X_0$ through the Gaussian diffusion process, then
  \begin{equation}
    \Pr{D_t(X_t)\ne X_0\mid \calC}
    \le5\err^{\sf G}_\calC(t)+\Pr[X\sim q^\calC_t]{\|\widetilde m_t(X)-m_t(X)\|_2>1/4}\,.
    \label{eq:gaussian_posterior_rounding_error}
  \end{equation}
\end{lemma}

\begin{proof}
  Fix $X_t=x$, let $y_*(x)\in \calC$ maximize the posterior probability $\Pr{X_0=y\mid X_t=x}$, and call this probability $p_*(x)$. If $p_*(x)\ge3/4$, every coordinate of $m_t(x)$ has the sign of $y_*(x)$ and magnitude at least $1/2$. An error at most $1/4$ in the Euclidean norm of the estimated posterior mean preserves all these signs. On these two events, $D_t(x)=y_*(x)$.

  Averaging over $X_t$, we get $\EE[1-p_*(X_t)\mid \calC]=\err^{\sf G}_\calC(t)$. Markov's inequality bounds the probability of $p_*(X_t)<3/4$ by $4\err^{\sf G}_\calC(t)$. Add the probability of an error greater than $1/4$ in the posterior mean estimate to obtain the claim.
\end{proof}

\begin{algorithm}
  \caption{Sampling from the random empirical measure with Gaussian diffusion}
  \label{alg:gaussian_random_measure_sampler}
  \KwInput{$d$, $[\kappa_-,\kappa_+]$, target accuracy $\epsilon$, failure probability $\delta$, and a score oracle satisfying Eq.~\eqref{eq:gaussian_score_accuracy}.}
  \KwOutput{A sample from the approximate output law $\widehat q^\calC$.}
  Set $\Lambda$ as in Eq.~\eqref{eq:gaussian_sampler_parameter} and $w$ as in Eq.~\eqref{eq:gaussian_window_width}\;
  Run Algorithm~\ref{alg:gaussian_window_search} to obtain $\widehat t$\;
  Set $t_+\leftarrow\widehat t+4w$, $t_-\leftarrow\widehat t-4w$, and $\Delta\leftarrow8w$\;
  Set $N\leftarrow\lceil C(d\Delta+d\Delta^2)/\epsilon^2\rceil$ and $r_j\leftarrow t_-+j\Delta/N$ for $0\le j\le N$\;
  Draw $Y_N\sim\nu_{t_+}$ independently of Algorithm~\ref{alg:gaussian_window_search}\;
  \For{$j=N,N-1,\ldots,1$}{
    Query $\widehat s_{r_j}(Y_j)$ and compute $\widetilde s_{r_j}(Y_j)$ using Eqs.~\eqref{eq:gaussian_clipped_mean}--\eqref{eq:gaussian_clipped_score_and_statistic}\;
    Draw $G_j\sim N(0,I_d)$ independently and update $Y_{j-1}$ by Eq.~\eqref{eq:gaussian_reverse_step} with $h=\Delta/N$\;
  }
  Query $\widehat s_{t_-}(Y_0)$ and compute $\widetilde m_{t_-}(Y_0)$ using Eq.~\eqref{eq:gaussian_clipped_mean}\;
  \Return{$\sgn(\widetilde m_{t_-}(Y_0))$}\;
\end{algorithm}

\begin{proof}[Proof of Theorem~\ref{thm:gaussian_random_measure_sampler}]
  We analyze Algorithm~\ref{alg:gaussian_random_measure_sampler}, allowing for errors in the search, the reverse steps, and the final rounding.

  We first apply Proposition~\ref{prop:gaussian_critical_window} at the deterministic times $\tcritg+w$ and $\tcritg-w$. Markov's inequality, with exponent 12 in that proposition, removes at most a $\delta/2$ fraction of subsets and gives that for the remaining fraction of $\calC$,
  \begin{align}
    \TV(q^\calC_t,\nu_t)&\le\epsilon/32,&&t\ge\tcritg+w,\notag\\
    \err^{\sf G}_\calC(t)&\le\epsilon/160,&&t\le\tcritg-w\,.
    \label{eq:gaussian_good_endpoints}
  \end{align}
  Monotonicity extends the two bounds from the fixed two times to the two full ranges. The constant $\epsilon/160$ is chosen so that $5\err^{\sf G}_\calC(t)\le\epsilon/32$ in Lemma~\ref{lem:gaussian_posterior_rounding} below. Proposition~\ref{prop:gaussian_window_search} further removes at most another $\delta/4$ fraction of $\calC$ where the window search might fail. Fix a $\calC$ satisfying all these properties for the remainder of the proof.

  With conditional probability at least $1-\epsilon/8$, the critical window search returns $\widehat t\in[\tcritg-3w,\tcritg+2w]$. Now, condition on any such value $\widehat t$. Then $t_+\ge\tcritg+w$, $t_-\le\tcritg-w$, and $t_+-t_-=\Delta=8w$. The choices of $c_0,d_0$ ensure that the endpoints lie in $I_0$, even if the search fails. All samples used after the search are fresh.

  First initialize the reverse steps from $q^\calC_{t_+}$. Lemma~\ref{lem:gaussian_posterior_variance} gives $V(t_+)-V(t_-)\le d$, and Lemma~\ref{lem:gaussian_mean_error} bounds the total score term by $C\Delta e^{-2B\Lambda}$. Proposition~\ref{prop:gaussian_short_reverse_interval} therefore gives
  \begin{equation}
    \KL{q^\calC_{t_-}}{\widehat q^\calC_{t_-}}
    \le C\left(\frac{d\Delta^2+d\Delta}{N}+\Delta e^{-2B\Lambda}\right)
    \le\epsilon^2/128\,.
  \end{equation}
  The last inequality follows by choosing the constant in $N$ and then $B$ sufficiently large; the same choice ensures $\Delta/N\le h_0$. Pinsker's inequality bounds $\TV(q^\calC_{t_-},\widehat q^\calC_{t_-})$ by $\epsilon/16$. The actual algorithm starts from $\nu_{t_+}$ rather than $q^\calC_{t_+}$. Since both initial distributions are propagated by the same approximate reverse kernels, we have
  \begin{equation}
    \TV\!\left(\widehat q^\calC_{t_-},\law(Y_0\mid \calC)\right) \le \TV(q^\calC_{t_+},\nu_{t_+}) \le \epsilon/32\,.
  \end{equation}
  Combining this with the preceding endpoint bound, with triangle inequality, gives
  \begin{equation}
    \TV\!\left(q^\calC_{t_-},\law(Y_0\mid \calC)\right) \le \epsilon/16+\epsilon/32=3\epsilon/32\,.
  \end{equation}

  If the final observation had law $q^\calC_{t_-}$, Lemma~\ref{lem:gaussian_mean_error} would give
  \begin{equation}
    \Pr[q^\calC_{t_-}]{\|\widetilde m_{t_-}-m_{t_-}\|_2>1/4}
    \le Ce^{-2B\Lambda}\le\epsilon/32\,.
  \end{equation}
  By Lemma~\ref{lem:gaussian_posterior_rounding} and Eq.~\eqref{eq:gaussian_good_endpoints}, its rounded output law would then be within $\epsilon/16$ of $q^\calC$ in total variation. Rounding the actual endpoint adds at most $\TV\!\left(q^\calC_{t_-},\law(Y_0\mid \calC)\right)$ of total variation error. Thus, conditional on a successful window search, the total error is at most $3\epsilon/32+\epsilon/16=5\epsilon/32$. Averaging over the window search adds at most its failure probability $\epsilon/8$, using $\TV\le1$ on the failure event, so the output law is within $5\epsilon/32+\epsilon/8=9\epsilon/32\le\epsilon$ of $q^\calC$ in total variation. The algorithm fails with probability at most $\delta/2+\delta/4<\delta$ over the choice of $\calC$, completing the proof of correctness.

  The search uses $O(\Lambda/w)=O(\sqrt{d\Lambda})$ queries. The reverse steps use $N=O\left(\frac{dw+dw^2}{\epsilon^2}+1\right)=O\left(\frac{\sqrt{d\Lambda}+\Lambda}{\epsilon^2}+1\right)$ queries, and rounding uses one more query, completing the proof.
\end{proof}

\section{\texorpdfstring{$\widetilde\Omega(d)$}{soft-Omega(d)} lower bound for masked diffusion}
\label{sec:masked-lower}

We now consider the same random empirical measure under masked diffusion and exhibit an approximate score oracle for which any sampler requires $\widetilde{\Omega}(d)$ queries, in contrast to uniform and Gaussian diffusion.
The intuition for this lower bound, which we state formally in Theorem~\ref{thm:masked_random_measure_lower} below, is as follows. If the codebook defining the empirical measure has size $|\calC|=2^k$, and a total of $m$ coordinates have been revealed, a fixed partial assignment agrees with about $2^{k-m}$ points of $\calC$. When $m$ is well below $k$, these points have nearly balanced signs in the remaining coordinates. When $m$ is well above $k$, a partial assignment chosen without knowledge of $\calC$ is unlikely to agree with any point at all. We will use these two facts to construct an accurate masked diffusion oracle that usually returns uniform marginals in both ranges. Only queries at values of $m$ close to the unknown value $k$ need reveal information about the target.

Let $\mathcal K_d^{\sf M}\triangleq\{k\in\ZZ:\kappa_-\le k\log 2/d\le\kappa_+\}$, and consider the prior under which $k\sim\Unif(\mathcal K_d^{\sf M})$ and, given $k$, $\calC$ is a uniformly random $2^k$-element subset of $\calX$. For a partial assignment $(I,x_I)$, write $\match{I}{x_I}\triangleq\{y\in \calC:y_I=x_I\}$ for the set of codewords matching it. Then $(I,x_I)$ is consistent with $q^\calC$ if and only if $\match{I}{x_I}\neq\varnothing$, in which case
\begin{equation}
  q^\calC_{i\mid I}(a\mid x_I)
  =\frac{|\{y\in \match{I}{x_I}:y_i=a\}|}{|\match{I}{x_I}|}\,,\qquad i\notin I\,.
  \label{eq:masked_conditional_marginal}
\end{equation}
We measure the accuracy of an oracle $\calO$ by its level-$m$ score error $\epsilon_{\sf mask}(\calO;q^\calC,m)$ from Definition~\ref{def:masked-score-error}, in the form given by Lemma~\ref{lem:masked_score_error_kl}. Recall that this constrains $\calO$ only on consistent partial assignments; on inconsistent partial assignments, the oracle can return arbitrary values.

\begin{theorem}[Masked diffusion lower bound]\label{thm:masked_random_measure_lower}
  Fix $\rho\in(0,1/4)$. There are constants $c,C,d_0>0$, depending only on $[\kappa_-,\kappa_+]$ and $\rho$, such that the following holds for $d\ge d_0$. Given $\eta\in(0,1/10)$, let
  \begin{equation}
    \ell=\left\lceil\log_2\frac{Cd}{\eta\rho^2}\right\rceil\,.
    \label{eq:masked_buffer_width}
  \end{equation}
  There is a family of oracles $\{\calO_{k,\calC}\}_{k\in \mathcal K_d^{\sf M},|\calC|=2^k}$ satisfying
  \begin{equation}
    \max_{0\le m<d}\epsilon_{\sf mask}(\calO_{k,\calC};q^\calC,m)\le\eta,\qquad k\in \mathcal K_d^{\sf M},|\calC|=2^k\,,
    \label{eq:masked_accuracy}
  \end{equation}
  such that every adaptive randomized algorithm $\calA$ making at most
  \begin{equation}
    Q\le c\rho^2\frac d\ell
  \end{equation}
  queries, and which does not have knowledge of $k,\mathcal{C}$, has output law $\widehat{q}_{\calA}^{\calO_{k,\calC}}$ satisfying
  \begin{equation}
    \Pr[\substack{k\sim\Unif(\mathcal K_d^{\sf M})\\\calC\sim\Unif\bigl(\binom{\{\pm 1\}^d}{2^k}\bigr)}]{\TV(\widehat{q}_{\calA}^{\calO_{k,\calC}},q^\calC)\ge1-\rho}\ge1-\rho\,.
    \label{eq:masked_lower_conclusion}
  \end{equation}
  In particular, for constant $\rho$ and any polynomially small $\eta=d^{-\Theta(1)}$, sampling with at most constant error for at least a constant fraction of the codebooks requires $\Omega(d/\log d)$ queries.
\end{theorem}

\noindent In Section~\ref{sec:masked_oracle_construction} we describe the oracle construction. In Section~\ref{sec:proof_masked_lower}, we prove the lower bound.

\subsection{Adversarial oracle construction}
\label{sec:masked_oracle_construction}

Define the oracle $\calO_0$ to return $\Unif(\{\pm 1\})$ in every unrevealed coordinate. We first check its average error when few coordinates are revealed.

\begin{lemma}\label{lem:masked_uniform_answers}
  For every $k\in\mathcal K_d^{\sf M}$ and $0\le m<d$,
  \begin{equation}
    \EE_\calC\epsilon_{\sf mask}(\calO_0;q^\calC,m)\le2^{m-k}\,.
    \label{eq:masked_uniform_answer_error}
  \end{equation}
\end{lemma}

\begin{proof}
  Fix $I$ with $|I|=m$ and $i\notin I$. For a fixed assignment $x_I$, define $N^\calC_I(x_I)=|\match{I}{x_I}|$. Conditional on $N^\calC_I(x_I)>0$, these $N^\calC_I(x_I)$ points form a uniformly drawn subset of the subcube $\{y:y_I=x_I\}$. Exactly half of this subcube has $y_i=+1$. Define $P^\calC_{i\mid I}(x_I)\triangleq |\{y\in \match{I}{x_I}:y_i=+1\}|/N^\calC_{I}(x_I)$ as the fraction of points in $\match{I}{x_I}$ with $y_i=+1$, so when $N^\calC_I(x_I)>0$,
  \begin{equation}
    \EE_\calC[P^\calC_{i\mid I}(x_I)\mid N^\calC_I(x_I)]=\frac12,\qquad
    \EE_\calC\Bigl[\Bigl(P^\calC_{i\mid I}(x_I)-\frac12\Bigr)^2\mid N^\calC_I(x_I)\Bigr]\le\frac1{4N^\calC_I(x_I)}\,.
    \label{eq:masked_conditional_fraction}
  \end{equation}
  The variance bound follows since we are sampling $\match{I}{x_I}$ without replacement. The inequality $\log u\le u-1$ also gives
  \begin{equation}
    \KL{\Ber(p)}{\Ber(1/2)}\le4(p-1/2)^2\,.
  \end{equation}
  Thus for any $N^\calC_I(x_I)>0$,
  \begin{equation}
    \EE_\calC\bigl[N^\calC_I(x_I)\,\KL{\Ber(P^\calC_{i\mid I}(x_I))}{\Ber(1/2)}\mid N^\calC_I(x_I)\bigr]\le4N^\calC_I(x_I)\,\EE_\calC\Bigl[\Bigl(P^\calC_{i\mid I}(x_I)-\frac12\Bigr)^2\mid N^\calC_I(x_I)\Bigr]\le1\,.
  \end{equation}
  When $N^\calC_I(x_I)=0$, define $N^\calC_I(x_I)\,\KL{\Ber(P^\calC_{i\mid I}(x_I))}{\Ber(1/2)}=0$.

  Now fix $\calC$. A sample $X_0\sim q^\calC$ satisfies $X_{0,I}=x_I$ with probability $N^\calC_I(x_I)/2^k$. Hence the expected KL error is
  \begin{equation}
    \EE_{X_0\sim q^\calC}\KL{q^\calC_{i\mid I}(\cdot\mid X_{0,I})}{\Ber(1/2)}
    =\frac1{2^k}\sum_{x_I\in\{\pm1\}^I}N^\calC_I(x_I)\,\KL{\Ber(P^\calC_{i\mid I}(x_I))}{\Ber(1/2)}\,.
  \end{equation}
  Now average over $\calC$. Each of the $2^m$ summands has expectation at most one by the preceding calculation, so the result is at most $2^{m-k}$. Averaging over $I$ and $i$ proves the lemma.
\end{proof}

\noindent Call $\calC$ \emph{good} if the answers provided by $\mathcal{O}_0$ satisfy the desired error bound at every level $m\le k-\ell$:
\begin{equation}
  \epsilon_{\sf mask}(\calO_0;q^\calC,m)\le\eta
  \quad\text{for every }0\le m\le k-\ell\,.
  \label{eq:masked_good_subset}
\end{equation}
Lemma~\ref{lem:masked_uniform_answers}, Markov's inequality, and a union bound over the levels $0\le m\le k-\ell$ imply
\begin{equation}
  \Pr[\calC]{\calC\text{ is not good}}\le\frac1\eta\sum_{m=0}^{k-\ell}2^{m-k}\le\frac{2^{1-\ell}}{\eta}\,.
  \label{eq:masked_bad_subset_probability}
\end{equation}

\begin{definition}[Adversarial masked diffusion oracle]\label{def:masked_lower_oracle}
  Define the oracle $\mathcal{O}_{k,\calC}$ as follows:
  \begin{itemize}
    \item For a good $\calC$, it returns uniform marginals at all levels $m\le k-\ell$. At levels $m > k - \ell$, it returns the exact conditional marginals on consistent assignments and uniform marginals on inconsistent ones. 
    \item For $\calC$ that is not good, at all levels $m$, it returns the exact conditional marginals on every consistent assignment and uniform marginals on inconsistent ones.
  \end{itemize} 
\end{definition}

\begin{proposition}[Validity of the oracle]\label{prop:masked_oracle_valid}
  For every $k\in\mathcal K_d^{\sf M}$ and $|\calC|=2^k$, the oracle $\calO_{k,\calC}$ satisfies $\max_{0\le m<d}\epsilon_{\sf mask}(\calO_{k,\calC};q^\calC,m)\le\eta$, i.e., Eq.~\eqref{eq:masked_accuracy}.
\end{proposition}

\begin{proof}
  By Lemma~\ref{lem:masked_score_error_kl}, $\epsilon_{\sf mask}(\calO_{k,\calC};q^\calC,m)$ depends only on the answers of $\calO_{k,\calC}$ on consistent assignments with $m$ revealed coordinates. These answers are exact, so the error at level $m$ vanishes, except when $\calC$ is good and $m\le k-\ell$, in which case $\calO_{k,\calC}$ agrees with $\calO_0$ and the error is at most $\eta$ by Eq.~\eqref{eq:masked_good_subset}.
\end{proof}

\subsection{Distribution indistinguishability: the lower bound proof}
\label{sec:proof_masked_lower}

Let $R$ denote all the internal randomness of an algorithm $\calA$. The following compares the algorithm's interaction with the constructed oracle to its interaction with $\calO_0$. For the latter, every answer is uniform, so the entire query sequence is fully determined by the internal randomness of the algorithm $R$.

\begin{proposition}[Coupling argument]\label{prop:masked_same_answers}
  Couple an algorithm's run against the oracle $\calO_{k,\calC}$ with its run against the oracle $\calO_0$, using the same internal randomness $R$. For an algorithm making at most $Q$ queries, let $p_{\rm diff}(k,\calC)$ be the conditional probability over $R$ given $k,\calC$ that the runs receive different oracle answers at some point. Then
  \begin{equation}
    \EE_{k,\calC}\,p_{\rm diff}(k,\calC)
    \le\frac{(2\ell+1)Q}{|\mathcal K_d^{\sf M}|}
      +\frac{2^{1-\ell}}\eta+Q2^{-\ell}\,.
    \label{eq:masked_average_transcript_difference}
  \end{equation}
\end{proposition}

\begin{proof}
  Let $\mathcal D$ denote the event that the two coupled runs receive different oracle answers at some point. By definition,
  \begin{equation}
    \EE_{k,\calC}\,p_{\rm diff}(k,\calC)
    =\EE_{k,\calC}\Pr[R]{\mathcal D\mid k,\calC}
    =\Pr[k,\calC,R]{\mathcal D}
    =\EE_R\Pr[k,\calC]{\mathcal D\mid R}\,.
    \label{eq:masked_diff_probability_identity}
  \end{equation}
  Now condition on any fixed realization of the internal randomness $R=r$. We will directly bound $\Pr[k,\calC]{\mathcal D\mid R=r}$ uniformly over $r$.

  It is enough to bound the probability of a first different answer, because the queries agree up to that point.

  Run the algorithm against oracle $\calO_0$. This fixes queries $(I_1,x_{I_1}),\ldots,(I_Q,x_{I_Q})$, padding with irrelevant queries if the algorithm stops early. The sequence is independent of both $k$ and $\calC$. Consider the randomness of $k$. For any query $(I_j,x_{I_j})$, $m_j\triangleq|I_j|$ is within $\ell$ of at most $2\ell+1$ possible integers $k$. A union bound shows that there exists a query with $|m_j-k|\le\ell$ with probability at most $(2\ell+1)Q/|\mathcal K_d^{\sf M}|$ over $k$.

  Fix a $k$ outside all these intervals $\{[m_j-\ell,m_j+\ell]\}_{j=1}^Q$. We also remove the event that $\calC$ is not good, whose probability is bounded by Eq.~\eqref{eq:masked_bad_subset_probability}.

  \paragraph{High-noise regime ($m<k-\ell$).} At any query $(I,x_I)$ with $m<k-\ell$, Definition~\ref{def:masked_lower_oracle} returns uniform marginals when $\calC$ is good, regardless of whether the assignment is consistent. Its answer agrees with $\calO_0$.

  \paragraph{Low-noise regime ($m>k+\ell$).} Fix one query $(I,x_I)$ with $m>k+\ell$. The subcube $\{y:y_I=x_I\}$ specified by this assignment has $2^{d-m}$ points, so
  \begin{equation}
    \Pr[\calC]{\match{I}{x_I}\ne\varnothing}
    \le\EE_\calC|\match{I}{x_I}|=2^{k-m}\le2^{-\ell}\,.
  \end{equation}
  This probability is under the original uniform choice of $\calC$ (we do not condition on $\calC$ being good). A union bound over the at most $Q$ queries with $m_j>k+\ell$ therefore shows that the probability that any such queried assignment is consistent is at most $Q2^{-\ell}$.

  Thus, for the fixed realization of $R=r$, if no query size $m_j$ lies in $[k-\ell,k+\ell]$, if $\calC$ is good, and if every queried assignment with $m_j>k+\ell$ is inconsistent with $\calC$, then every query in the run against $\calO_{k,\calC}$ receives exactly the same uniform answer as in the run against $\calO_0$. Indeed, before a first different answer the two runs have identical queries and answers, and hence make the same next query. By induction over the queries, no first different answer can occur under these three conditions.

  Consequently, uniformly over $R=r$,
  \begin{equation}
    \Pr[k,\calC]{\mathcal D\mid R=r}
    \le
    \frac{(2\ell+1)Q}{|\mathcal K_d^{\sf M}|}
    +\frac{2^{1-\ell}}\eta
    +Q2^{-\ell}\,.
    \label{eq:masked_diff_bound_fixed_randomness}
  \end{equation}
  Averaging Eq.~\eqref{eq:masked_diff_bound_fixed_randomness} over $R$ and using Eq.~\eqref{eq:masked_diff_probability_identity} gives Eq.~\eqref{eq:masked_average_transcript_difference}.
\end{proof}

\noindent We can now complete the proof of the main result of this section:

\begin{proof}[Proof of Theorem~\ref{thm:masked_random_measure_lower}]
  By Proposition~\ref{prop:masked_oracle_valid}, the oracle $\calO_{k,\calC}$ satisfies the accuracy requirement for every instance. Fix any algorithm with query budget $Q$, and let $\nu$ be its output law when every oracle answer is uniform. This law is independent of $k,\calC$.

  For each fixed instance $(k,\calC)$, couple the two runs of the algorithm, one against $\calO_{k,\calC}$ and the other against $\calO_0$, by using the same internal randomness $R$, as in Proposition~\ref{prop:masked_same_answers}. Let $Y_{k,\calC}(R)$ and $Y_0(R)$ denote the corresponding outputs. Their marginal laws over $R$ are respectively $\widehat{q}_{\calA}^{\calO_{k,\calC}}$ and $\nu$. Hence, by the coupling inequality,
  \begin{align}
    \TV(\widehat{q}_{\calA}^{\calO_{k,\calC}},\nu)
    &\le \Pr[R]{Y_{k,\calC}(R)\ne Y_0(R)} \\
    &\le \Pr[R]{\text{the two runs receive different oracle answers}} \\
    &=p_{\rm diff}(k,\calC)\,.
    \label{eq:masked_output_coupling}
  \end{align}
  The second inequality holds because, for every fixed realization of $R$, identical oracle answers yield the same final output.

  Since $q^\calC$ is supported on $\calC$, it follows that
  \begin{equation}
    \TV(\widehat{q}_{\calA}^{\calO_{k,\calC}},q^\calC)
    \ge1-\widehat{q}_{\calA}^{\calO_{k,\calC}}(\calC)
    \ge1-\nu(\calC)-\TV(\widehat{q}_{\calA}^{\calO_{k,\calC}},\nu)
    \ge1-\nu(\calC)-p_{\rm diff}(k,\calC)\,.
    \label{eq:masked_output_separation}
  \end{equation}
  A fixed point belongs to a uniform $2^k$-element subset with probability $2^{k-d}$. Averaging according to any output law $\nu$ independent of $\calC$ thus gives
  \begin{equation}
    \EE_\calC\nu(\calC)=\EE_{X\sim\nu}\EE_\calC\bone{X\in \calC}=2^{k-d}\le e^{-(\log2-\kappa_+)d}\,.
  \end{equation}

  We have $|\mathcal K_d^{\sf M}|\ge c_1d$ for some constant $c_1>0$ depending only on $[\kappa_-,\kappa_+]$, for all sufficiently large $d$. Choose the constant $c$ in the query budget small enough that the first term of Eq.~\eqref{eq:masked_average_transcript_difference} is at most $\rho^2/4$. Also $Q\le d$ for this choice. Taking $C$ in Eq.~\eqref{eq:masked_buffer_width} sufficiently large bounds the second and third terms in Eq.~\eqref{eq:masked_average_transcript_difference} by $\rho^2/4$ in total. For sufficiently large $d$, the expected mass $\EE_\calC\nu(\calC)$ is at most $\rho^2/4$. Hence,
  \begin{equation}
    \EE_{k,\calC}[\nu(\calC)+p_{\rm diff}(k,\calC)]\le\rho^2\,.
  \end{equation}
  Markov's inequality shows that $\Pr[k,\calC]{\nu(\calC)+p_{\rm diff}(k,\calC)\le \rho}\ge 1-\rho$. Eq.~\eqref{eq:masked_output_separation} proves the theorem. Finally, $\ell=O(\log d)$ for polynomially small $\eta$ and constant $\rho$, giving $\Omega(d/\log d)$ queries.
\end{proof}

\noindent Note that implicit in the argument above was that the critical window for masked diffusion has width which is only an $O(\log d / d)$ (rather than $O(\sqrt{\log d/d})$ as in Gaussian and uniform diffusion) fraction of the total range of noise levels. We make this precise with the following analogue of Propositions~\ref{prop:uniform_critical_window} and~\ref{prop:gaussian_critical_window}: above the window, the revealed coordinates of a random codeword are close to uniform, while below it, the original point can be recovered with high probability. In analogy with $\err^{\sf U}_\calC(t)$ and $\err^{\sf G}_\calC(t)$, define
\begin{equation}
  \err^{\sf M}_\calC(m) \triangleq \inf_{\widehat{z}(\cdot)} \Pr[z\sim q^\calC,\, I\sim\Unif\bigl(\binom{[d]}{m}\bigr)]{\widehat{z}(I,z_I) \neq z}\,,
\end{equation}
the smallest error of any decoder that observes $m$ uniformly random coordinates of a sample from $q^\calC$.

\begin{proposition}[Critical window]\label{prop:masked_critical_window}
  Let $|\calC|=2^k$ and $\ell\ge 1$, and fix $0\le m<d$. If $m\le k-\ell$, then for every $I\subseteq[d]$ with $|I|=m$ and $X_0\sim q^\calC$ (with $\law(X_{0,I})$ taken conditionally on $\calC$),
  \begin{equation}
    \EE_\calC\,\TV\bigl(\law(X_{0,I}),\Unif(\{\pm1\}^I)\bigr)\le 2^{-\ell/2}\,.\label{eq:masked_above_window}
  \end{equation}
  If $m\ge k+\ell$, then
  \begin{equation}
    \EE_\calC\,\err^{\sf M}_\calC(m)\le 2^{-\ell}\,.\label{eq:masked_below_window}
  \end{equation}
\end{proposition}

\begin{proof}
  Fix $I$ with $|I|=m$ and write $N(x_I)\triangleq|\match{I}{x_I}|$, so that $\law(X_{0,I})=N(\cdot)/2^k$. Since $\calC$ is a uniformly random $2^k$-element subset of $\calX$, each $N(x_I)$ is hypergeometric with mean $2^{k-m}$ and variance at most $2^{k-m}$. Thus
  \begin{align}
    \EE_\calC\,\TV\bigl(\law(X_{0,I}),\Unif(\{\pm1\}^I)\bigr) &= \frac{1}{2}\sum_{x_I\in\{\pm1\}^I}\EE_\calC\bigl|N(x_I)2^{-k} - 2^{-m}\bigr| \\
    &\le \frac{1}{2}\cdot 2^{m}\cdot 2^{-k}\sqrt{2^{k-m}} = \frac{1}{2}\cdot 2^{-(k-m)/2}\le 2^{-\ell/2}\qquad\text{for }m\le k-\ell\,,
  \end{align}
  which gives Eq.~\eqref{eq:masked_above_window}. For Eq.~\eqref{eq:masked_below_window}, consider the decoder which outputs the unique element of $\match{I}{z_I}$ whenever this set is a singleton, and something arbitrary otherwise. It is incorrect only if some other codeword agrees with $z$ on $I$. Conditionally on $z$, the remaining $2^k - 1$ codewords form a uniformly random subset of $\calX\setminus\{z\}$, and each agrees with $z$ on $I$ with probability at most $2^{-m}$, so by a union bound $\EE_\calC\,\err^{\sf M}_\calC(m)\le 2^{k-m}\le 2^{-\ell}$ for $m\ge k+\ell$.
\end{proof}

\section{Outlook}

In this work, we compared the parallel sampling capabilities of three leading paradigms for diffusion language modeling. We first showed that the query complexity of uniform and Gaussian diffusion can adapt to the intrinsic complexity of the data distribution as quantified by its dual total correlation, matching what was previously known for masked diffusion~\cite{chen2025optimal,lavenant2025error}.

On top of establishing parity in this regard across all three paradigms, we then exhibited a simple family of distributions over which uniform and Gaussian diffusion provably outperform any sampling algorithm based on masked diffusion. In contrast with prior heuristic reasoning which suggested that masked diffusion is less amenable to parallelism because it cannot revise the tokens it commits to during sampling, we showed that the origin of this separation comes from a different mechanism. Instead, it is due to the fact that the critical window of noise levels for masked diffusion, over which the posterior transitions from being uninformative to sharply concentrating on a small subset of the support, is asymptotically narrower (by a factor of $\widetilde{\Theta}(\sqrt{d})$) than the analogous windows for uniform and Gaussian diffusion. To our knowledge, this is the first connection between critical windows and discretization error for diffusion sampling.

That said, this work only touches upon one of many important performance axes for diffusion language modeling and should not be interpreted reductively as advocating for one paradigm over another. Indeed, there are various other advantages of masked diffusion, for instance its any-order generation capability, which remain poorly understood in theory but incredibly relevant in practice~\cite{ye2025beyond,pmlr-v267-kim25ah}. On the computational side, prior work of Ghio et al.~\cite{ghio2024sampling} and Bhatt et al.~\cite{bhatt2026generating} compared the algorithmic complexity of denoising under different corruption processes in the context of spin glass-like distributions. Lastly, in the case of Gaussian diffusion, we have only considered the most basic form of embedding, leaving open the exploration of richer encoding maps. Rigorously understanding all of these aspects, and more generally the interplay between corruption process, parallelism, complexity of denoising, and flexibility of inference remains an important open direction, and we leave further exploration of this to future work.

\section*{Acknowledgments}

We thank Fan Chen, Sinho Chewi, Kevin Cong, Khashayar Gatmiry, Holden Lee, Yiwen Kou, Jerry Li, Jianfeng Lu, Raghu Meka, Adil Salim, and Yimeng Wang for many illuminating discussions on discretization of discrete and continuous diffusion models in recent years. We also thank Martin Wainwright and Yuting Wei for coordinating with us on the discussion of our concurrent works in forthcoming versions of our papers.

\subsection*{Statement on AI usage}

Most of the key ideas of this work were due to the authors from spring and early summer of this year, with assistance from frontier models in a few crucial parts. Most notably, the telescoping argument for the DTC-adaptive rates was suggested by AI. The original proof of the reverse DPI for Gaussian diffusion that was generated by AI was extremely difficult to verify, but further human-AI interaction resulted in the final form of the argument. In an earlier version, we also overlooked the need for the extra Bregman projection step in the uniform diffusion case, which was subsequently suggested by AI. While we had an entirely human-generated argument for our main separation results, AI assisted with simplifying some parts of the argument. It suggested a simpler test statistic for locating the critical window than the one we initially considered. We previously had an alternative argument for sampling below the critical noise level by proving a bound on the relaxed Lipschitz constant of the posterior and then appealing to existing discretization analyses. We believed that this step could be simplified but did not attempt to optimize this part as it did not affect the final result. AI however made the simple but helpful observation that below the critical window, a one-step decoding sufficed. A coding agent also implemented the numerical simulation in Figure~\ref{fig:window-sim}, with no further guidance beyond the initial prompt seeding it with the experiment of~\cite{xun2026query}.

In fact, we could use it to prove a considerably more general version of the $\widetilde{O}(\sqrt{d})$ upper bound for Gaussian diffusion than the one that appears here, generalizing from uniform distributions over random subsets of the hypercube to random empirical distributions for any background measure over the unit sphere satisfying concentration of convex Lipschitz functions. We chose not to include this result as Theorem~\ref{thm:informal-sqrtd} already captures the bulk of the intuition.

Frontier models were also used to create all of the figures in this paper, and to catch minor typos and errors in notation and setting of parameters. Beyond these aspects, everything in this paper was human written.

\bibliographystyle{alpha}
\bibliography{refs}

\appendix

\section{\texorpdfstring{$\widetilde\Omega(\sqrt d)$}{soft-Omega(sqrt(d))} lower bound for uniform diffusion}
\label{app:uniform-query-lower}

The upper bound in Section~\ref{sec:uniform-upper} searches for the critical time $\tcritu$ defined in Eq.~\eqref{eq:critical_times}. We show that this search cannot be avoided with approximate scores. The hard distribution family contains $\widetilde\Omega(\sqrt d)$ possible values of $\tcritu$, separated by more than the width $w=\widetilde\Theta(1/\sqrt d)$ of the critical window defined in Eq.~\eqref{eq:uniform_lower_parameters} below, and a query far from the realized value usually receives the score of the uniform distribution.

This is the continuous-time counterpart of the masked lower bound in Section~\ref{sec:masked-lower}. There, the hidden transition occurs when the number of revealed coordinates passes the hidden integer $k$. Here it occurs when the query time passes $\tcritu$. Above $\tcritu$, the noised random empirical measure is close to uniform. Below $\tcritu$, a genuine sample is close to a point of $\calC$, but a fixed query chosen without knowing $\calC$ is unlikely to be close to any point of $\calC$.

For every integer $k$ satisfying $\kappa_-\le k\log2/d\le\kappa_+$, let $\tcritu(k)$ be the unique solution of
\begin{equation}
  d\,\uMI(\tcritu(k))=k\log2=\log M\,.
  \label{eq:uniform_lower_critical_time}
\end{equation}
The derivative
\begin{equation}
  \uMI'(t)=-\frac{e^{-t}}2\log\frac{1+e^{-t}}{1-e^{-t}}
  \label{eq:uniform_information_derivative}
\end{equation}
is bounded above and below in magnitude on $[\tau_-,\tau_+]$.

Given $A,C_\Lambda>0$, $\eta\in(0,1/10)$, and $\rho\in(0,1/4)$, define
\begin{equation}
  \Lambda=\log\frac{C_\Lambda d}{\eta\rho^2}\,,\qquad
  w=A\sqrt{\frac{\Lambda}{d}}\,.
  \label{eq:uniform_lower_parameters}
\end{equation}

Fix a sufficiently large constant $D>0$. Starting with the smallest admissible integer $k$, retain every $\lceil Ddw\rceil$-th one, and denote the resulting set by
\begin{equation}
  \mathcal K_d^{\sf U}\triangleq\left\{\left\lceil\frac{\kappa_-d}{\log2}\right\rceil+j\lceil Ddw\rceil:j=0,1,\ldots,\left\lfloor\frac{\kappa_+d/\log2-\lceil\kappa_-d/\log2\rceil}{\lceil Ddw\rceil}\right\rfloor\right\}\,.
  \label{eq:uniform_lower_rate_set}
\end{equation}
If $k<k'$ are consecutive elements of $\mathcal K_d^{\sf U}$, the mean value theorem applied to Eq.~\eqref{eq:uniform_lower_critical_time} gives, for some $\xi\in(\tcritu(k'),\tcritu(k))$,
\begin{equation}
  \tcritu(k)-\tcritu(k')
  =\frac{(k'-k)\log2/d}{-\uMI'(\xi)}\,.
  \label{eq:uniform_lower_critical_time_increment}
\end{equation}
Consequently, after choosing $D$ large enough,
\begin{equation}
  \tcritu(k)-\tcritu(k')\ge 10w,
  \qquad
  |\mathcal K_d^{\sf U}|
  \ge\frac{cd}{\lceil Ddw\rceil}
  \ge\frac{c_{\sf U}}w\,,
  \label{eq:uniform_lower_critical_time_spacing}
\end{equation}
for all sufficiently large $d$. In particular,
\begin{equation}
  |\mathcal K_d^{\sf U}|\ge\frac{c_{\sf U}}A\sqrt{\frac d\Lambda}\,.
  \label{eq:uniform_lower_rate_count}
\end{equation}
We choose $k$ uniformly from $\mathcal K_d^{\sf U}$ and, conditional on $k$, choose $\calC$ uniformly among the $2^k$-element subsets of $\calX$. Once $k$ is fixed, we write $\tcritu$ for $\tcritu(k)$.

Throughout, $s^\calC_t$ denotes the exact score of $q^\calC_t$, that is, $s^\calC_t(x)[i,a]=q^\calC_t(x^{i\leftarrow a})/q^\calC_t(x)$. We measure the accuracy of a score oracle $\calO$ by $\epsilon_{\sf unif}(\calO;q^\calC,t)$ from Definition~\ref{def:uniform-score-error}.

\begin{theorem}[Uniform diffusion lower bound]\label{thm:uniform_query_lower}
  Fix $\rho\in(0,1/4)$. There are constants $A,C_\Lambda,c,c_0,d_0>0$, depending only on $[\kappa_-,\kappa_+]$ and $\rho$, such that the following holds for $d\ge d_0$. Given $\eta\in(0,1/10)$, take $\Lambda,w$, and $\mathcal K_d^{\sf U}$ as defined above, and suppose $\Lambda\le c_0d$. There is a family of score oracles $\{\calO_{k,\calC}\}_{k\in\mathcal K_d^{\sf U},|\calC|=2^k}$, each returning normalized score vectors with strictly positive entries, satisfying
  \begin{equation}
    \sup_{t>0}\epsilon_{\sf unif}(\calO_{k,\calC};q^\calC,t)\le\eta,
    \qquad k\in\mathcal K_d^{\sf U},\quad |\calC|=2^k\,,
    \label{eq:uniform_lower_oracle_accuracy}
  \end{equation}
  such that every adaptive randomized algorithm $\calA$ that is not given $k$ or $\calC$ and makes at most
  \begin{equation}
    Q\le c\rho^2\sqrt{\frac d\Lambda}
    \label{eq:uniform_lower_query_budget}
  \end{equation}
  score oracle queries has output law $\widehat{q}_{\calA}^{\calO_{k,\calC}}$ satisfying
  \begin{equation}
    \Pr[\substack{k\sim\Unif(\mathcal K_d^{\sf U})\\\calC\sim\Unif\bigl(\binom{\calX}{2^k}\bigr)}]
    {\TV(\widehat{q}_{\calA}^{\calO_{k,\calC}},q^\calC)\ge1-\rho}
    \ge1-\rho\,.
    \label{eq:uniform_lower_bound}
  \end{equation}
  In particular, for constant $\rho$ and any polynomially small $\eta=d^{-\Theta(1)}$, sampling with at most constant error for at least a constant fraction of the instances requires $\Omega(\sqrt{d/\log d})$ score oracle queries.
\end{theorem}

\noindent For the rest of the appendix, $\Lambda$ and $w$ have the values in Eq.~\eqref{eq:uniform_lower_parameters}. We may increase $A,C_\Lambda,d_0$ and decrease $c_0$ later when necessary so that $\log d\le\Lambda\le c_0d$ and $w\le\min\{\tau_-/8,1/8\}$. Fix a constant $\gamma>0$ sufficiently small in terms of $[\tau_-,\tau_+]$.

\par\vspace{\baselineskip}\noindent In Section~\ref{sec:uniform_lower_oracle_construction}, we construct the oracle family and prove its accuracy. In Section~\ref{sec:proof_uniform_lower}, we prove the query lower bound.

\subsection{Adversarial oracle construction}
\label{sec:uniform_lower_oracle_construction}

Let $\mathbf1\in\RR^{[d]\times\{-1,1\}}$ be the score vector defined by $\mathbf1[i,a]=1$ for every $i\in[d]$ and $a\in\{-1,1\}$, and define the reference oracle by $\calO_0(t,x)\triangleq\mathbf1$ for every query $(t,x)$. This is the exact score oracle of $\Unif(\calX)$. Proposition~\ref{prop:uniform_critical_window} already shows that $q^\calC_t$ is close to $\Unif(\calX)$ above $\tcritu$; the following lemma records the corresponding statement for the score error.

\begin{lemma}\label{lem:uniform_lower_above_window}
  For every $K>0$, there are constants $A_K,c_K,d_K>0$, depending only on $[\kappa_-,\kappa_+]$ and $K$, such that, whenever $A\ge A_K$, $d\ge d_K$, $\log d\le\Lambda\le c_Kd$, and $t\ge\tcritu(k)+4w$,
  \begin{equation}
    \EE_\calC\,\TV(q^\calC_t,\Unif(\calX))\le e^{-K\Lambda},
    \qquad
    \EE_\calC\,\epsilon_{\sf unif}(\calO_0;q^\calC,t)\le e^{-K\Lambda}\,.
    \label{eq:uniform_lower_above_window}
  \end{equation}
  Both inequalities hold uniformly over $k\in\calK^{\sf U}_d$ and $t\ge\tcritu(k)+4w$.
\end{lemma}

\begin{proof}
  Since $t\ge\tcritu+4w$, Proposition~\ref{prop:uniform_critical_window} with exponent $K+2$, whose constants at that exponent we take as the $A_K,c_K,d_K$ of the present lemma, gives
  \begin{equation}
    \EE_\calC\,\TV(q^\calC_t,\Unif(\calX))
    \le e^{-(K+2)\Lambda}
    \le e^{-K\Lambda}\,.
  \end{equation}

  Eq.~\eqref{eq:uniform_score_bounds} places the exact scores in a fixed compact subinterval of $(0,\infty)$, on which $\psi(v,1)\le C|v-1|$. For each $\calC$ and each coordinate $i$, we then have
  \begin{align}
    \EE_{X\sim q^\calC_t}\psi(s^\calC_t(X)[i,-X_i],1)
    &\le C\sum_{x\in\calX}\left|q^\calC_t(x^{i\leftarrow-x_i})-q^\calC_t(x)\right|\notag\\
    &\le4C\,\TV(q^\calC_t,\Unif(\calX))\,.
    \label{eq:uniform_lower_score_from_tv}
  \end{align}
  The second inequality follows by inserting $2^{-d}$ and using the fact that a coordinate flip permutes $\calX$. Summing Eq.~\eqref{eq:uniform_lower_score_from_tv} over $i$ and averaging over $\calC$ gives
  \begin{equation}
    \EE_\calC\,\epsilon_{\sf unif}(\calO_0;q^\calC,t)
    \le4Cd\,e^{-(K+2)\Lambda}
    \le e^{-K\Lambda}\,,
  \end{equation}
  where the last inequality holds after increasing $d_K$.
\end{proof}

\noindent Well below $\tcritu$, a sample $X_t$ from $q^\calC_t$ is usually close to the point $X_0$ of $\calC$ from which it was generated. For $t>0$, define
\begin{equation}
  R_t=\left\lceil d(\beta_t+\gamma w)\right\rceil+1,
  \qquad
  \mathcal N_{\calC,t}=\left\{x\in\calX:\min_{y\in \calC}d_H(x,y)\le R_t\right\}\,.
  \label{eq:uniform_lower_neighborhood}
\end{equation}
The oracle will return the exact score on the neighborhood $\mathcal N_{\calC,t}$ and the all-one score outside it. The next lemma separates the two facts needed later: a genuine sample from $q^\calC_t$ usually lies in this neighborhood, whereas a fixed query usually does not.

\begin{lemma}\label{lem:uniform_lower_below_window}
  There are constants $c,C,c_1,d_1>0$, depending only on $[\kappa_-,\kappa_+]$, such that, whenever $d\ge d_1$, $\log d\le dw^2\le c_1d$, and $0<t\le\tcritu-4w$, the following statements hold uniformly over the allowed values of $k$ and $t$.
  \begin{enumerate}[label=(\roman*)]
    \item For every $\calC$,
    \begin{equation}
      q^\calC_t(\mathcal N_{\calC,t}^c)\le e^{-2\gamma^2dw^2}\,.
      \label{eq:uniform_lower_target_neighborhood}
    \end{equation}
    If the score output of the oracle $\calO$ equals $s^\calC_t$ on $\mathcal N_{\calC,t}$ and the all-one score on $\mathcal N_{\calC,t}^c$, then
    \begin{equation}
      \epsilon_{\sf unif}(\calO;q^\calC,t)
      \le Cd(1+\log d)e^{-cdw^2}\,.
      \label{eq:uniform_lower_neighborhood_score_error}
    \end{equation}
    \item For every fixed $x\in\calX$,
    \begin{equation}
      \Pr[\calC]{x\in\mathcal N_{\calC,t}}\le e^{-cdw}\,.
      \label{eq:uniform_lower_fixed_query_probability}
    \end{equation}
  \end{enumerate}
\end{lemma}

\begin{proof}
  Fix $k$ and write $\tcritu=\tcritu(k)$. Since $\beta'_r=e^{-r}/2$ and $0\le t<\tcritu\le\tau_+$,
  \begin{equation}
    \beta_{\tcritu}-\beta_t
    =\int_t^{\tcritu}\frac{e^{-r}}2\,\d r
    \ge\frac{e^{-\tau_+}}2(\tcritu-t)
    \ge2e^{-\tau_+}w\,.
    \label{eq:uniform_lower_noise_gap}
  \end{equation}
  Choose $\gamma<e^{-\tau_+}/2$. Since $dw^2\ge\log d$, i.e., $w\ge\sqrt{\log d/d}$, after increasing $d_1$ we have
  \begin{equation}
    \frac{R_t}{d}\le\beta_t+\gamma w+\frac2d
    \le\beta_{\tcritu}-cw
    <\frac12\,.
    \label{eq:uniform_lower_radius_gap}
  \end{equation}

  We start by showing Eq.~\eqref{eq:uniform_lower_target_neighborhood}. Fix $\calC$, draw $Y\sim q^\calC$, and generate $X\sim (K_t^{\sf U})^{\otimes d}(\cdot\mid Y)$. Then $d_H(X,Y)\sim\Bin(d,\beta_t)$ and $\min_{z\in \calC}d_H(X,z)\le d_H(X,Y)$. By Hoeffding's inequality, we have
  \begin{equation}
    q^\calC_t(\mathcal N_{\calC,t}^c)
    \le\Pr{\Bin(d,\beta_t)\ge R_t}
    \le e^{-2\gamma^2dw^2}\,,
  \end{equation}
  proving Eq.~\eqref{eq:uniform_lower_target_neighborhood}.

  We next bound the score error outside $\mathcal N_{\calC,t}$. For every $x\in\calX$ and $i\in[d]$, the binary likelihood ratio in Lemma~\ref{lem:uniform-likelihood-ratio} gives
  \begin{equation}
    \frac{\beta_t}{1-\beta_t}
    \le s^\calC_t(x)[i,-x_i]
    \le\frac{1-\beta_t}{\beta_t}\,.
    \label{eq:uniform_lower_score_range}
  \end{equation}
  Hence
  \begin{equation}
    q^\calC_t(x)\psi(s^\calC_t(x)[i,-x_i],1)
    \le q^\calC_t(x)+\log\frac{1-\beta_t}{\beta_t}\,q^\calC_t(x^{i\leftarrow-x_i})\,.
    \label{eq:uniform_lower_pointwise_score_error}
  \end{equation}
  If $x\notin\mathcal N_{\calC,t}$, then every one-coordinate neighbor of $x$ has distance at least $R_t$ from $\calC$. Summing Eq.~\eqref{eq:uniform_lower_pointwise_score_error} over $x\notin\mathcal N_{\calC,t}$ and $i\in[d]$ gives
  \begin{equation}
    \epsilon_{\sf unif}(\calO;q^\calC,t)
    \le d\left(1+\log\frac{1-\beta_t}{\beta_t}\right)
    \Pr{\Bin(d,\beta_t)\ge R_t}\,.
    \label{eq:uniform_lower_error_from_tail}
  \end{equation}

  If $\beta_t\ge d^{-2}$, then the logarithmic factor in Eq.~\eqref{eq:uniform_lower_error_from_tail} is at most $C(1+\log d)$, and Hoeffding's inequality proves Eq.~\eqref{eq:uniform_lower_neighborhood_score_error}. If $0<\beta_t<d^{-2}$, set $m=\lceil\gamma dw\rceil$. By the Chernoff bound,
  \begin{equation}
    \Pr{\Bin(d,\beta_t)\ge R_t}
    \le\left(\frac{ed\beta_t}{m}\right)^m\,.
    \label{eq:uniform_lower_small_noise_tail}
  \end{equation}
  For $m\ge2$, the function $u\mapsto(1+\log(1/u))u^m$ is increasing on $(0,d^{-2}]$. Therefore,
  \begin{equation}
    \left(1+\log\frac{1-\beta_t}{\beta_t}\right)
    \Pr{\Bin(d,\beta_t)\ge R_t}
    \le C(1+\log d)\left(\frac{e}{dm}\right)^m
    \le C(1+\log d)e^{-cdw^2}\,.
    \label{eq:uniform_lower_weighted_tail}
  \end{equation}
  The last inequality uses $m\ge\gamma dw$, $w\le1$, and hence $m\log(dm/e)\ge cdw\log d\ge cdw^2$ for sufficiently large $d$. Eqs.~\eqref{eq:uniform_lower_error_from_tail}--\eqref{eq:uniform_lower_weighted_tail} prove Eq.~\eqref{eq:uniform_lower_neighborhood_score_error} uniformly over $t$.

  Now fix $x$ before drawing $\calC$. Since $\calC$ is a uniform $2^k$-element subset of $\calX$,
  \begin{align}
    \Pr[\calC]{x\in\mathcal N_{\calC,t}}
    &\le\EE_\calC\left|\calC\cap\{y:d_H(x,y)\le R_t\}\right|\notag\\
    &=\frac{2^k}{2^d}\sum_{j=0}^{R_t}\binom dj
    \le\exp\left(k\log2-d\log2+d h_2(R_t/d)\right)\,.
    \label{eq:uniform_lower_hamming_ball}
  \end{align}
  By Eq.~\eqref{eq:uniform_lower_critical_time},
  \begin{equation}
    k\log2=d\,\uMI(\tcritu)
    =d\bigl(\log2-h_2(\beta_{\tcritu})\bigr)\,.
    \label{eq:uniform_lower_rate_at_critical_time}
  \end{equation}
  Moreover, $h'_2(u)=\log((1-u)/u)$ is bounded below by a positive constant for $u\le\beta_{\tau_+}<1/2$. Eq.~\eqref{eq:uniform_lower_radius_gap} consequently gives
  \begin{equation}
    h_2(\beta_{\tcritu})-h_2(R_t/d)
    =\int_{R_t/d}^{\beta_{\tcritu}}\log\frac{1-u}{u}\,\d u
    \ge cw\,.
    \label{eq:uniform_lower_entropy_gap}
  \end{equation}
  Substituting Eqs.~\eqref{eq:uniform_lower_rate_at_critical_time} and~\eqref{eq:uniform_lower_entropy_gap} into Eq.~\eqref{eq:uniform_lower_hamming_ball} proves Eq.~\eqref{eq:uniform_lower_fixed_query_probability}.
\end{proof}

\noindent The oracle now follows the three regimes suggested by the two lemmas above. It gives the exact score within distance $4w$ of $\tcritu$ (the window). Below this window it gives the exact score only on $\mathcal N_{\calC,t}$ and all-one score otherwise. Above the window it gives the all-one score whenever that answer has error at most $\eta$, and otherwise gives the exact score.

\begin{definition}[Adversarial uniform diffusion oracle]\label{def:uniform_lower_oracle}
  For each fixed instance $(k,\calC)$, define $\calO_{k,\calC}$ by
  \begin{equation}
    \calO_{k,\calC}(t,x)\triangleq
    \begin{cases}
      s^\calC_t(x),&|t-\tcritu|<4w,\\
      s^\calC_t(x),&t\le\tcritu-4w\text{ and }x\in\mathcal N_{\calC,t},\\
      \mathbf1,&t\le\tcritu-4w\text{ and }x\notin\mathcal N_{\calC,t},\\
      \mathbf1,&t\ge\tcritu+4w\text{ and }\epsilon_{\sf unif}(\calO_0;q^\calC,t)\le\eta,\\
      s^\calC_t(x),&t\ge\tcritu+4w\text{ and }\epsilon_{\sf unif}(\calO_0;q^\calC,t)>\eta.
    \end{cases}\,.
  \end{equation}
\end{definition}

\begin{proposition}[Validity and indistinguishability of the oracle]\label{prop:uniform_lower_oracle}
  The constants in Theorem~\ref{thm:uniform_query_lower} can be chosen so that the oracles in Definition~\ref{def:uniform_lower_oracle} satisfy Eq.~\eqref{eq:uniform_lower_oracle_accuracy} for every allowed instance. For these choices, whenever $|t-\tcritu|\ge4w$ and $x\in\calX$ is fixed independently of $\calC$,
  \begin{equation}
    \Pr[\calC]{\calO_{k,\calC}(t,x)\ne\calO_0(t,x)}\le e^{-4\Lambda}\,.
    \label{eq:uniform_lower_oracle_difference}
  \end{equation}
\end{proposition}

\begin{proof}
  We first choose $A$ large enough that Lemma~\ref{lem:uniform_lower_above_window} holds with $K=6$ and that, by Lemma~\ref{lem:uniform_lower_below_window} (whose hypothesis $\log d\le dw^2=A^2\Lambda\le c_1d$ holds once $c_0\le c_1/A^2$), uniformly over $t\le\tcritu(k)-4w$,
  \begin{equation}
    \epsilon_{\sf unif}(\calO_{k,\calC};q^\calC,t)
    \le Cd(1+\log d)e^{-cdw^2}
    =Cd(1+\log d)e^{-cA^2\Lambda}\le\eta\,,
    \label{eq:uniform_lower_accuracy_choice}
  \end{equation}
  once $A$ and then $d_0$ are sufficiently large. We then further decrease $c_0$ and increase $d_0$ so that, by Lemma~\ref{lem:uniform_lower_below_window} again, uniformly over $t\le\tcritu(k)-4w$,
  \begin{equation}
    \EE_\calC\bone{x\in\mathcal N_{\calC,t}}
    \le e^{-cdw}\le e^{-4\Lambda}\,.
    \label{eq:uniform_lower_fixed_query_choice}
  \end{equation}

  Fix an instance and a time. In the interval $|t-\tcritu|<4w$, Definition~\ref{def:uniform_lower_oracle} returns the exact score. If $t\le\tcritu-4w$, Eq.~\eqref{eq:uniform_lower_accuracy_choice} gives error at most $\eta$. If $t\ge\tcritu+4w$, the oracle returns the all-one score only when its error is at most $\eta$ and returns the exact score otherwise. This proves Eq.~\eqref{eq:uniform_lower_oracle_accuracy}.

  It remains to prove Eq.~\eqref{eq:uniform_lower_oracle_difference}. Below the window, the oracle can differ from $\calO_0$ only when $x\in\mathcal N_{\calC,t}$, so Eq.~\eqref{eq:uniform_lower_fixed_query_choice} gives the desired bound. Above the window, the two answers can differ only when $\epsilon_{\sf unif}(\calO_0;q^\calC,t)>\eta$. Markov's inequality and Lemma~\ref{lem:uniform_lower_above_window} with $K=6$ give
  \begin{equation}
    \Pr[\calC]{\epsilon_{\sf unif}(\calO_0;q^\calC,t)>\eta}
    \le\frac{e^{-6\Lambda}}\eta
      \le e^{-4\Lambda}\,,
  \end{equation}
  where the last step uses $\eta\ge e^{-2\Lambda}=(\eta\rho^2/(C_\Lambda d))^2$, which holds since $\eta,\rho<1$ and we may take $C_\Lambda\ge1$. This finishes the proof of the proposition.
\end{proof}

\subsection{Distribution indistinguishability: the lower bound proof}
\label{sec:proof_uniform_lower}

Let $R$ denote all the internal randomness of an adaptive algorithm $\calA$. The following compares the algorithm's interaction with the constructed oracle to its interaction with $\calO_0$. For the latter, every answer is the all-one score, so the entire query sequence is determined by $R$ and is independent of $k$ and $\calC$.

\begin{proposition}[Coupling argument]\label{prop:uniform_lower_same_answers}
  Couple an algorithm's run against $\calO_{k,\calC}$ with its run against $\calO_0$, using the same internal randomness $R$. For an algorithm making at most $Q$ queries, let $p_{\rm diff}(k,\calC)$ be the conditional probability over $R$ given $k,\calC$ that the runs receive different oracle answers at some point. Then
  \begin{equation}
    \EE_{k,\calC}p_{\rm diff}(k,\calC)
    \le Q\left(\frac1{|\mathcal K_d^{\sf U}|}+e^{-4\Lambda}\right)\,.
    \label{eq:uniform_lower_transcript_difference}
  \end{equation}
\end{proposition}

\begin{proof}
  Let $\mathcal D$ denote the event that the two coupled runs receive different oracle answers at some point. By definition,
  \begin{equation}
    \EE_{k,\calC}p_{\rm diff}(k,\calC)
    =\EE_{k,\calC}\Pr[R]{\mathcal D\mid k,\calC}
    =\Pr[k,\calC,R]{\mathcal D}
    =\EE_R\Pr[k,\calC]{\mathcal D\mid R}\,.
    \label{eq:uniform_lower_difference_identity}
  \end{equation}
  Now condition on any fixed realization of the internal randomness $R=r$. We will bound $\Pr[k,\calC]{\mathcal D\mid R=r}$ uniformly over $r$.

  It is enough to bound the probability of a first different answer, because the queries agree up to that point.

  Run the algorithm against $\calO_0$, padding with irrelevant queries if it stops before making $Q$ queries. This produces a deterministic sequence $(t_j,x_j)_{j=1}^Q$, independent of $k$ and $\calC$. Consider the randomness of $k$. By Eq.~\eqref{eq:uniform_lower_critical_time_spacing}, the possible values of $\tcritu(k)$ are separated by at least $10w$. Therefore, for each query time $t_j$, the interval $|t_j-\tcritu(k)|<4w$ contains the critical time of at most one $k\in\mathcal K_d^{\sf U}$, and hence
  \begin{equation}
    \Pr[k]{|t_j-\tcritu(k)|<4w}\le\frac1{|\mathcal K_d^{\sf U}|}\,.
    \label{eq:uniform_lower_query_hits_window}
  \end{equation}
  A union bound shows that at least one query time lies in the critical window with probability at most $Q/|\mathcal K_d^{\sf U}|$ over $k$.

  Fix a $k$ for which every query time lies outside the critical window, i.e., $|t_j-\tcritu(k)|\ge4w$ for every $j$. Since each $x_j$ was fixed before $\calC$ was drawn, Proposition~\ref{prop:uniform_lower_oracle} gives, for every $j$,
  \begin{equation}
    \Pr[\calC]{\calO_{k,\calC}(t_j,x_j)\ne\calO_0(t_j,x_j)}\le e^{-4\Lambda}\,.
  \end{equation}
  A union bound over the $Q$ queries shows that some oracle answer differs outside the window with probability at most $Qe^{-4\Lambda}$.

  \par\vspace{\baselineskip}\noindent Consequently, for the fixed realization $R=r$, the probability that a query enters the critical window or that an oracle answer differs outside the window is at most the right-hand side of Eq.~\eqref{eq:uniform_lower_transcript_difference}. If neither event occurs, the two runs have identical queries and answers and hence make the same next query at every step, so $\mathcal D$ does not occur. Averaging over $R$ in Eq.~\eqref{eq:uniform_lower_difference_identity} proves the proposition.
\end{proof}

\begin{proof}[Proof of Theorem~\ref{thm:uniform_query_lower}]
  Proposition~\ref{prop:uniform_lower_oracle} gives an oracle satisfying Eq.~\eqref{eq:uniform_lower_oracle_accuracy} for every instance. Fix an adaptive randomized algorithm $\calA$, and let $\nu$ be its output law when every answer comes from $\calO_0$. Since $\calO_0$ gives the same answer on every instance, $\nu$ is independent of $k$ and $\calC$.

  For each fixed instance $(k,\calC)$, couple the run against $\calO_{k,\calC}$ with the run against $\calO_0$ using the same internal randomness $R$. Let $Y_{k,\calC}(R)$ and $Y_0(R)$ denote their outputs. Their laws over $R$ are $\widehat{q}_{\calA}^{\calO_{k,\calC}}$ and $\nu$, respectively, and hence
  \begin{align}
    \TV(\widehat{q}_{\calA}^{\calO_{k,\calC}},\nu)
    &\le\Pr[R]{Y_{k,\calC}(R)\ne Y_0(R)}\notag\\
    &\le p_{\rm diff}(k,\calC)\,.
    \label{eq:uniform_lower_output_coupling}
  \end{align}
  Since $q^\calC$ is supported on $\calC$,
  \begin{equation}
    \TV(\widehat{q}_{\calA}^{\calO_{k,\calC}},q^\calC)
    \ge1-\widehat{q}_{\calA}^{\calO_{k,\calC}}(\calC)
    \ge1-\nu(\calC)-p_{\rm diff}(k,\calC)\,.
    \label{eq:uniform_lower_output_separation}
  \end{equation}

  Fix $k$. A fixed point of $\calX$ belongs to a uniformly chosen $2^k$-element subset with probability $2^{k-d}$. Since $\nu$ is independent of $\calC$,
  \begin{equation}
    \EE_{\calC\mid k}\nu(\calC)
    =\EE_{X\sim\nu}\Pr[\calC\mid k]{X\in \calC}
    =2^{k-d}
    \le e^{-(\log2-\kappa_+)d}\,.
    \label{eq:uniform_lower_null_output_mass}
  \end{equation}

  We now choose the constant $c$ in Eq.~\eqref{eq:uniform_lower_query_budget}. By Eqs.~\eqref{eq:uniform_lower_rate_count} and~\eqref{eq:uniform_lower_query_budget},
  \begin{equation}
    \frac{Q}{|\mathcal K_d^{\sf U}|}
    \le\frac{cA}{c_{\sf U}}\rho^2\,.
  \end{equation}
  Choose $c\le c_{\sf U}/(4A)$, so this term is at most $\rho^2/4$. Since $\Lambda\ge\log d$,
  \begin{equation}
    Qe^{-4\Lambda}
    \le c\rho^2\sqrt{\frac d\Lambda}\,e^{-4\Lambda}
    \le c\rho^2e^{-7\Lambda/2}
    \le\frac{\rho^2}{4}
  \end{equation}
  after increasing $d_0$. Increase $d_0$ once more so that the right-hand side of Eq.~\eqref{eq:uniform_lower_null_output_mass} is at most $\rho^2/2$. Proposition~\ref{prop:uniform_lower_same_answers} then gives
  \begin{equation}
    \EE_{k,\calC}\bigl[\nu(\calC)+p_{\rm diff}(k,\calC)\bigr]\le\rho^2\,.
  \end{equation}
  Markov's inequality yields
  \begin{equation}
    \Pr[k,\calC]{\nu(\calC)+p_{\rm diff}(k,\calC)\le\rho}\ge1-\rho\,.
  \end{equation}
  Eq.~\eqref{eq:uniform_lower_output_separation} proves Eq.~\eqref{eq:uniform_lower_bound}. Finally, $\Lambda=O(\log d)$ for polynomially small $\eta$ and constant $\rho$, giving the stated $\Omega(\sqrt{d/\log d})$ lower bound.
\end{proof}

\section{Deferred proofs}
\label{app:deferred-proofs}

\subsection{Proofs from Section~\ref{sec:preliminaries}}

\subsubsection{Proof of Lemma~\ref{lem:uniform_score_to_marginals}}

\begin{proof}
  Fix $y\in\calX$ and $i\in[d]$, and denote
  \begin{equation*}
    v \triangleq s_t(y)[i,\cdot]\,, \qquad V \triangleq \sum_{b\in\Sigma} v_b\,.
  \end{equation*}
  Note that $v_{y_i} = 1$.
  Let
  \begin{equation*}
    g_a \triangleq \Pr{X^i_s = a,\, (X_t)_{-i} = y_{-i}}\,, \qquad G \triangleq \sum_{a\in\Sigma} g_a = \Pr{(X_t)_{-i} = y_{-i}}\,.
  \end{equation*}
  Conditionally on $X_s$, the coordinates of $X_t$ are independent, and $(X_t)_i\sim K_h^{\sf U}(\cdot\mid X^i_s)$. As a result, for every $b\in\Sigma$,
  \begin{equation}
    q_t(\sub{y}{i}{b}) = \sum_{a\in\Sigma} g_a K_h^{\sf U}(b\mid a) = \alpha_h g_b + \beta_h G\,.
    \label{eq:score-one-coordinate-step}
  \end{equation}
  Summing Eq.~\eqref{eq:score-one-coordinate-step} over $b\in\Sigma$ and using $\alpha_h + S\beta_h = 1$ gives $\sum_{b\in\Sigma} q_t(\sub{y}{i}{b}) = G$, so dividing by $q_t(y)$ identifies the normalizing constant as
  \begin{equation*}
    V = G/q_t(y)\,.
  \end{equation*}

  \noindent\textbf{The Gibbs sampling marginal.} Dividing Eq.~\eqref{eq:score-one-coordinate-step} by $q_t(y)$ and then by $V = G/q_t(y)$,
  \begin{equation}
    \frac{v_b}{V} = \alpha_h\cdot\frac{g_b}{G} + \beta_h\,.
    \label{eq:score-gibbs-identity}
  \end{equation}
  As $g_b/G = \Pr{X^i_s = b\mid (X_t)_{-i} = y_{-i}}$, solving Eq.~\eqref{eq:score-gibbs-identity} for $g_b/G$ gives the second claim.

  \noindent\textbf{The posterior marginal.} Write $\pi \triangleq \gibbs_h[v]$ for the distribution just computed. As $(X_t)_i$ is conditionally independent of $(X_t)_{-i}$ given $X^i_s$, by Bayes' rule we have
  \begin{equation*}
    \Pr{X^i_s = a\mid X_t = y} = \frac{\pi_a K_h^{\sf U}(y_i\mid a)}{\sum_{a'\in\Sigma} \pi_{a'} K_h^{\sf U}(y_i\mid a')}\,.
  \end{equation*}
  The denominator equals $\alpha_h\pi_{y_i} + \beta_h$, which by Eq.~\eqref{eq:score-gibbs-identity} is $v_{y_i}/V = 1/V$. Substituting this and $\alpha_h V \pi_a = v_a - \beta_h V$, both read off from Eq.~\eqref{eq:score-gibbs-identity}, we obtain the first claim,
  \begin{equation*}
    \Pr{X^i_s = a\mid X_t = y} = V\pi_a\bigl(\alpha_h\bone{a = y_i} + \beta_h\bigr) = \frac{\alpha_h \bone{a=y_i} + \beta_h}{\alpha_h}\bigl(v_a - \beta_h V\bigr) = \post_{h,y_i}[v](a)\,. \qedhere
  \end{equation*}
\end{proof}

\subsubsection{Proof of Lemma~\ref{lem:masked_score_error_kl}}

\begin{proof}
  Fix $I$ with $|I|=m$, $X_0$, and $i\notin I$, and write $Y = X_0^{(I)}$; the partial assignment $(I,X_{0,I})$ is consistent. By Lemma~\ref{lem:masked_score_interpretation}, the definition of $\widehat{q}^{\calO}_{i\mid I}$, and the homogeneity $\psi(\lambda c,\lambda s)=\lambda\psi(c,s)$ for $\lambda>0$,
  \begin{align}
    &\frac{1-\alpha_t}{\alpha_t}\sum_{a\in [S]} \psi(s_t(Y)[i,a],\widehat{s}_t(Y)[i,a]) \notag\\
    &= \sum_{a\in [S]} \left(\widehat{q}^{\calO}_{i\mid I}(a\mid X_{0,I})-q_{i\mid I}(a\mid X_{0,I})+ q_{i\mid I}(a\mid X_{0,I})\log\frac{q_{i\mid I}(a\mid X_{0,I})}{\widehat{q}^{\calO}_{i\mid I}(a\mid X_{0,I})}\right) \notag\\
    &= \KL{q_{i\mid I}(\cdot\mid X_{0,I})}{\widehat{q}^{\calO}_{i\mid I}(\cdot\mid X_{0,I})}\,,
  \end{align}
  since $\sum_{a\in [S]} q_{i\mid I}(a\mid X_{0,I}) = \sum_{a\in [S]} \widehat{q}^{\calO}_{i\mid I}(a\mid X_{0,I}) = 1$. Averaging over $i\notin I$ and taking the expectation over $I$ and $X_0$ gives Eq.~\eqref{eq:mask_error_kl}.
\end{proof}

\subsection{Proofs from Section~\ref{sec:uniform-dtc}}

\subsubsection{Proof of Proposition~\ref{prop:robust-marginals}}

\begin{proof}
    For the first part, write $v_a \triangleq \unifscore{t}{y}{i,a} = q_t(\yia)/q_t(y)$. Every entry of $K_t^{\sf U}(\cdot\mid z)$ is at least $\beta_t > 0$, so $q_t$ has full support and $v\in\mathbb{R}^S_{>0}$; moreover $v_{y_i} = 1$ since $\sub{y}{i}{y_i} = y$. For the constraint $v_a/\sum_{a'\in\Sigma}v_{a'}\ge\beta_t$,
    \begin{equation}
        \frac{v_a}{\sum_{a'\in\Sigma} v_{a'}} = \frac{q_t(\yia)}{\sum_{a'\in\Sigma}q_t(\sub{y}{i}{a'})} = \Pr{(X_t)_i = a \mid (X_t)_{-i} = y_{-i}}\,.
    \end{equation}
    That this is lower bounded by $\beta_t$ follows by Proposition~\ref{prop:floor}.

    For the second part, write $w_a \triangleq v_a/\sum_{a'\in\Sigma} v_{a'}$, so that $p \triangleq \gibbs_h[v]$ is given by $p(a) = (w_a - \beta_h)/\alpha_h$. It is immediate that $\sum_a p(a) = 1$. Moreover, $w_a \ge \beta_t$ by definition, and $\beta_t - \beta_h = \frac{e^{-h} - e^{-t}}{S} = \alpha_h\beta_s$, so $p(a) \ge (\beta_t - \beta_h)/\alpha_h = \beta_s$. Finally, the identity $\post_{h,b}[v](a)=\bigl(\sum_{a'}v_{a'}\bigr)K_h^{\sf U}(b\mid a)\,p(a)$ follows from the definitions, and summing it over $a$ gives $\bigl(\sum_{a'}v_{a'}\bigr)(\alpha_hp(b)+\beta_h)=v_b=1$.
\end{proof}

\subsubsection{Proof of Lemma~\ref{lem:forward_decay}}

\begin{proof}
  For each coordinate $i\in[d]$, let $B_i\sim\Ber(e^{-T})$ be independent, and let $U=(U_1,\dots,U_d)\sim \Unif(\calX)$ be independent of $X_0\sim q$ and $B$. Define the coupled random vector $\widetilde X$ by
  \begin{equation}
    \widetilde X_i \triangleq
    \begin{cases}
      (X_0)_i,& B_i=1,\\
      U_i,& B_i=0.
    \end{cases}\,.
  \end{equation}
  Then $\widetilde X\sim q_T$ (this is exactly the coordinatewise kernel $(K_T^{\sf U})^{\otimes d}$ with $\alpha_T=e^{-T}$ and $\beta_T=(1-e^{-T})/S$).

  Conditioned on $B=b\in\{0,1\}^d$, the coordinates with $b_i=0$ are i.i.d.\ uniform, and the coordinates with $b_i=1$ have joint law equal to the corresponding marginal of $q$. Let $q^{(b)}$ denote the conditional law of $\widetilde X$ given $B=b$. By convexity of KL in its first argument,
  \begin{equation}
    \KL{q_T}{\Unif(\calX)}
    = \KL{\EE_B[q^{(B)}]}{\Unif(\calX)}
    \le \EE_B\,\KL{q^{(B)}}{\Unif(\calX)}\,.
  \end{equation}
  For any fixed $b$, write $I(b)\triangleq\{i:b_i=1\}$ and $k\triangleq |I(b)|$. Since $q^{(b)}$ is uniform on coordinates outside $I(b)$,
  \begin{equation}
    \KL{q^{(b)}}{\Unif(\calX)}
    = \KL{\law(X_{0,I(b)})}{\Unif(\Sigma^{I(b)})}
    \le k\log S\,.
  \end{equation}
  Taking expectation over $B$ gives
  \begin{equation}
    \KL{q_T}{\Unif(\calX)} \le \EE_B|I(B)|\log S = d e^{-T}\log S\,,
  \end{equation}
  as claimed.
\end{proof}

\subsubsection{Proof of Lemma~\ref{lem:shorttime}}

\begin{proof}
  Fix $0<t\le 1$ and use the same mask coupling as in the proof of Lemma~\ref{lem:forward_decay}, with $B_i\sim\Ber(e^{-t})$ and $\widetilde X\sim q_t$ defined by
  \begin{equation}
    \widetilde X_i =
    \begin{cases}
      (X_0)_i,& B_i=1,\\
      U_i,& B_i=0,
    \end{cases}\,,
  \end{equation}
  where $U\sim\Unif(\Sigma^d)$ is independent.

  Using $H(\widetilde X)\le H(\widetilde X,B)=H(B)+H(\widetilde X\mid B)$, we obtain
  \begin{equation}
    H(q_t)-H(q)=H(\widetilde X)-H(X_0)
    \le H(B)+\bigl(H(\widetilde X\mid B)-H(X_0)\bigr)\,.
  \end{equation}
  Conditioned on $B=b$, the refreshed coordinates contribute $(d-|I(b)|)\log S$ nats of entropy and are independent of everything else, while the kept coordinates are a function of $X_0$. Hence
  \begin{equation}
    H(\widetilde X\mid B=b)
    = H(X_{0,I(b)}) + (d-|I(b)|)\log S
    \le H(X_0) + (d-|I(b)|)\log S\,.
  \end{equation}
  Averaging over $B$ yields
  \begin{equation}
    H(\widetilde X\mid B)-H(X_0)
    \le \EE_B[d-|I(B)|]\log S
    = d(1-e^{-t})\log S
    \le dt\log S\,.
  \end{equation}

  It remains to bound $H(B)$. Since $B$ has i.i.d.\ coordinates with $\Pr{B_i=0}=1-e^{-t}$, we have
  \begin{equation}
    H(B)=d\,h_2(1-e^{-t})\,.
  \end{equation}
  For $0<\rho\le 1$, the standard bound $h_2(\rho)\le \rho\log(e/\rho)$ gives
  \begin{equation}
    H(B)\le d(1-e^{-t})\log\Bigl(\frac{e}{1-e^{-t}}\Bigr)\,.
  \end{equation}
  For $0<t\le 1$ we have $1-e^{-t}\le t$, and also $\rho\mapsto \rho\log(e/\rho)$ is increasing on $(0,1]$, so
  \begin{equation}
    H(B)\le dt\log(e/t)\,.
  \end{equation}
  Combining the two parts,
  \begin{equation}
    H(q_t)-H(q)\le dt\log(e/t)+dt\log S = dt\log(eS/t)\,,
  \end{equation}
  as desired.
\end{proof}

\subsection{Proofs from Section~\ref{sec:gaussian-dtc}}

\subsubsection{Proof of Lemma~\ref{lem:gaussian_kl_fi_evolution}}

\begin{proof}
  Apply Lemma~\ref{lem:heatflow} with $\mu_u=f_{\alpha,u}$, $\nu_u=f_{\beta,u}$, and $\ell_u=\log(f_{\alpha,u}/f_{\beta,u})$; the first part is then the first identity of that lemma. For the second part, since $\nabla\log f_{\lambda,u}(z)=(\mathsf{m}_{\lambda,u}(z)-z)/u$ for every $\lambda\in\Delta(\Sigma)$, a direct calculation yields
  \begin{equation}
    \nabla \ell_u = \nabla \log f_{\alpha,u} - \nabla \log f_{\beta,u} = \frac{1}{u}(\mathsf{m}_{\alpha,u}(z) - \mathsf{m}_{\beta,u}(z))\,,
  \end{equation}
  \begin{equation}
    \nabla^2\log \nu_u =\nabla^2 \log f_{\beta,u} = -\frac{1}{u}\Id_S + \frac{1}{u^2} \mathsf{C}_{\beta,u}(z)\,,
  \end{equation}
  \begin{equation}
    \nabla^2 \ell_u = \nabla^2 \log f_{\alpha,u} - \nabla^2 \log f_{\beta,u} = \frac{1}{u^2}(\mathsf{C}_{\alpha,u}(z) - \mathsf{C}_{\beta,u}(z))\,.
  \end{equation}
  By Lemma~\ref{lem:heatflow} and the fact that $\mathsf{C}_{\lambda,u}(z) \succeq 0$ for all $\lambda$,
  \begin{align}
    \partial_u \FI{f_{\alpha,u}}{f_{\beta,u}} &\ge -\frac{1}{u^4}\EE_{f_{\alpha,u}}\norm{\mathsf{C}_{\alpha,u}(z) - \mathsf{C}_{\beta,u}(z)}^2_F - \frac{2}{u^3}\EE_{f_{\alpha,u}} \norm{\mathsf{m}_{\alpha,u}(z) - \mathsf{m}_{\beta,u}(z)}^2_2 \\
    &\ge -\Bigl(\frac{2}{u} + \frac{3}{u^2}\Bigr) \FI{f_{\alpha,u}}{f_{\beta,u}}\,,
  \end{align}
  where in the last step we used Proposition~\ref{prop:linalg} below.
\end{proof}

\noindent The above proof used the following elementary fact:

\begin{proposition}\label{prop:linalg}
  Given $p\in \Delta(\Sigma)$, define $\mathsf{C}(p) = \mathrm{diag}(p) - p^{\otimes 2}$. Then for all $p,p'\in\Delta(\Sigma)$,
  \begin{equation}
    \norm{\mathsf{C}(p) - \mathsf{C}(p')}^2_F \le 3\norm{p - p'}_2^2\,.
  \end{equation}
\end{proposition}

\begin{proof}
  Let $s \triangleq p + p'$ and $\delta \triangleq p - p'$. Then $\mathsf{C}(p) - \mathsf{C}(p') = \mathrm{diag}(\delta) - \frac{1}{2}(s\delta^\top + \delta s^\top)$, so
  \begin{equation}
    \norm{\mathsf{C}(p) - \mathsf{C}(p')}^2_F = \norm{\delta}_2^2 - 2\sum_a s_a\delta_a^2 + \frac{1}{2}\norm{s}_2^2\norm{\delta}_2^2 + \frac{1}{2}\langle s,\delta\rangle^2 \le 3\norm{\delta}_2^2\,,
  \end{equation}
  as $\norm{s}_2^2 \le 4$ and $\langle s,\delta\rangle^2 \le 2\sum_a s_a\delta_a^2$ by Cauchy--Schwarz.
\end{proof}

\subsubsection{Proof of Lemma~\ref{lem:categorical-mixture-perturbation}}

\begin{proof}
Write $\varphi_a$ for the density of $N(e_a,\rho \Id_S)$ and $h\triangleq \alpha-\beta$. We have
\begin{equation}
  \KL{f_{\alpha,\rho}}{f_{\beta,\rho}} \le \chisq{f_{\alpha,\rho}}{f_{\beta,\rho}}
  \leq\int\frac{(\sum_ah_a\varphi_a)^2}{\sum_a\beta_a\varphi_a}\,.
\end{equation}
For fixed $h$, the right-hand side is convex in $\beta$, so it is maximized at some vertex $\beta=e_b$, in which case the denominator is $\varphi_b$. A direct calculation gives
\begin{equation}
  \int\frac{\varphi_a(z)\varphi_c(z)}{\varphi_b(z)}\,\d z
  =\exp\!\left(\frac{(e_a-e_b)^\top(e_c-e_b)}{\rho}\right)\,.
  \label{eq:categorical-gaussian-overlap}
\end{equation}
Using $\sum_ah_a=0$ and expanding the quadratic form, with $x=e^{1/\rho}$, we conclude that
\begin{equation}
  \int\frac{(\sum_ah_a\varphi_a)^2}{\varphi_b}
  =(x-1)h_b^2+(x^2-x)\sum_{a\neq b}h_a^2
  \leq(x^2-1)\|h\|_2^2\,,
\end{equation}
as claimed.
\end{proof}

\subsection{Proofs from Section~\ref{sec:uniform-upper}}

Throughout this subsection we use the notation of Section~\ref{sec:uniform-upper}: $\Lambda$ and $w=A\sqrt{\Lambda/d}$ are as in Eqs.~\eqref{eq:uniform_sampler_parameter} and~\eqref{eq:uniform_window_width}, $I_0=[\tau_-/2,\tau_++1]$ is the interval of Eq.~\eqref{eq:known_ts}, $L_t=q^\calC_t/2^{-d}$ is the likelihood ratio with respect to the uniform distribution, $\ell_t=\beta_t/(1-\beta_t)$, $\widetilde s_t$ is the clipped approximate score of Eq.~\eqref{eq:uniform_clipped_score}, $\widehat T_t$ is the statistic of Eq.~\eqref{eq:uniform_empirical_statistic}, and $F_t$ is the squared clipped score error of Eq.~\eqref{eq:uniform_mse}.

\subsubsection{Proof of Lemma~\ref{lem:uniform_output_radius}}

\begin{proof}
  Define the truncated likelihood ratios
  \begin{equation}
    \widetilde{L}^c_z(x) = 2^d (K_t^{\sf U})^{\otimes d}(x \mid z) \cdot \bone{d_H(x,z)> r} \qquad \text{and} \qquad \widetilde{L}_z(x) = 2^d (K_t^{\sf U})^{\otimes d}(x \mid z) \cdot \bone{d_H(x,z) \le r}\,.
  \end{equation}
  We can write the likelihood ratio as
  \begin{equation}
    L_t(x) = \frac{q^\calC_t(x)}{2^{-d}} = \EE_{z\sim \Unif(\calC)}\,\widetilde{L}^c_z(x) + \EE_{z\sim \Unif(\calC)}\,\widetilde{L}_z(x)\,.
  \end{equation}
  Also define $\mu(x) = \EE_{z\sim \Unif(\calX)}\,\widetilde{L}^c_z(x)$. Note that $\mu(x) \le 1$ because the expectation of the (un-truncated) likelihood ratio is $1$. Then
  \begin{align}
    2\EE_\calC\TV(q^\calC_t, \Unif(\calX)) &\le \underbrace{\EE_\calC\EE_{x\sim\Unif(\calX)} \,\Bigl|\EE_{z\sim \Unif(\calC)}\,\widetilde{L}^c_z(x)- \mu(x)\Bigr|}_{\circled{1}} + \underbrace{\EE_{x\sim\Unif(\calX)}[1 - \mu(x)]}_{\circled{2}} \\
    &\qquad \qquad + \underbrace{\EE_\calC\EE_{x\sim \Unif(\calX)} \EE_{z\sim \Unif(\calC)}\,\widetilde{L}_z(x)}_{\circled{3}}\,.
  \end{align}

  We have
  \begin{align}
    \circled{1}^2 &\le \EE_\calC\EE_{x\sim\Unif(\calX)} \,\Bigl(\EE_{z\sim \Unif(\calC)} \,\widetilde{L}^c_z(x) - \mu(x)\Bigr)^2 \\
    &\le \frac{1}{M}\EE_{x, z\sim\Unif(\calX)} \Bigl(\widetilde{L}^c_z(x) - \mu(x)\Bigr)^2 \\
    &\le \frac{1}{M}\EE_{x, z\sim\Unif(\calX)} \,\widetilde{L}^c_z(x)^2 \\
    &\le \frac{1}{M}(1 + e^{-t})^{d}\cdot \tanh(t/2)^{(r+1)}\,,
  \end{align}
  where in the second step we used that the variance of the empirical average over samples drawn without replacement is at most the variance of the empirical average over samples drawn with replacement, and in the last step we used that $\widetilde{L}^c_z(x)^2\le(1+e^{-t})^d\tanh(t/2)^{r+1}\cdot2^d(K_t^{\sf U})^{\otimes d}(x\mid z)$ by Lemma~\ref{lem:uniform-likelihood-ratio}, together with $\EE_{x,z\sim\Unif(\calX)}2^d(K_t^{\sf U})^{\otimes d}(x\mid z)=1$.
  Finally,
  \begin{equation}
    \circled{2} = \circled{3} = \EE_{x,z\sim\Unif(\calX)} \widetilde{L}_z(x) = \Pr[z\sim\Unif(\calX), x\sim (K_t^{\sf U})^{\otimes d}(\cdot \mid z)]{d_H(x,z) \le r}\,.
  \end{equation}
  The claimed bound follows upon observing that the above quantity is $\Pr{\Bin(d, \beta_t) \le r}$.
\end{proof}

\subsubsection{Proof of Lemma~\ref{lem:uniform_recovery_radius}}

\begin{proof}
  We will bound the Bayes-optimal recovery error by the error of the nearest-neighbor decoder, i.e., the one which given $x$ outputs anything in $\arg\min_{z'\in \calC} d_H(x,z')$. This decoder succeeds when the true $z$ satisfies $d_H(x,z) \le r$ and for all $z'\in \calC\backslash\{z\}$, $d_H(x,z') > r$.
  
  The probability that $d_H(x,z)>r$ is exactly the second term on the right-hand side of the claimed bound, because the distribution of $z$ given by sampling random $\calC$ and then taking $z\sim q^\calC$ is uniform over $\calX$, and conditioned on $z$ the quantity $d_H(x,z)$ is distributed as $\Bin(d,\beta_t)$.

  The probability that some $z'\in\calC\setminus\{z\}$ has $d_H(x,z')\le r$ is at most the expected number of such $z'$. Conditioned on any $x, z$, the remaining elements of $\calC$ are a random $(M-1)$-sized subset of the cube, so the expected number of such $z'$ is at most $\frac{M-1}{2^d - 1}\cdot |B_r(x)|$, where $B_r(x)$ is the closed Hamming ball of radius $r$ around $x$. The size of $B_r(x)$ is independent of $x$ and equal to $2^d\Pr{\Bin(d,1/2) \le r}$, so the claim follows.
\end{proof}

\subsubsection{Proof of Lemma~\ref{lem:uniform_score_truncation}}

\begin{proof}
  For the first part of Eq.~\eqref{eq:uniform_clipped_error}, the definitions of $T_t$ and $\widehat T_t$ (Eq.~\eqref{eq:uniform_empirical_statistic}) and the Cauchy--Schwarz inequality give
  \begin{equation}
    |\widehat T_t(x)-T_t(x)|
    \le\frac{1-e^{-2t}}{2e^{-t}}\sum_{i=1}^d\bigl|\widetilde s_t(x)[i,-x_i]-s_t(x)[i,-x_i]\bigr|
    \le\frac{1-e^{-2t}}{2e^{-t}}\sqrt{dF_t(x)}\,,
  \end{equation}
  and the prefactor $\frac{1-e^{-2t}}{2e^{-t}}$ is bounded on $I_0$ by a constant depending only on $[\kappa_-,\kappa_+]$.

  For the second part, fix any entry $c\triangleq s_t(x)[i,-x_i]$ of the true score. The function $u\mapsto\psi(c,u)$ decreases on $(0,c]$ and increases on $[c,\infty)$, since its derivative is $1-c/u$. By Eq.~\eqref{eq:uniform_score_bounds}, $c\in[\ell_t,\ell_t^{-1}]$, so clipping the estimated entry $\widehat s_t(x)[i,-x_i]$ to this interval cannot increase $\psi(c,\cdot)$. On the compact time interval $I_0$, all the clipping intervals $[\ell_t,\ell_t^{-1}]$ lie in a single interval $[m_0,M_0]\subset(0,\infty)$. For $c,u\in[m_0,M_0]$, the second derivative of $u\mapsto \psi(c,u)$, which is $c/u^2$, is at least $m_0/M_0^2$. Taylor's theorem at $u=c$ gives
  \begin{equation}
    (u-c)^2\le \frac{2M_0^2}{m_0}\psi(c,u)\,.
  \end{equation}
  Applying this inequality to each clipped entry $u=\widetilde s_t(x)[i,-x_i]$, summing over $i\in[d]$, and averaging under $q^\calC_t$ gives the second part of Eq.~\eqref{eq:uniform_clipped_error} with $C'=2M_0^2/m_0$, which depends only on $I_0$ and hence only on $[\kappa_-,\kappa_+]$.
\end{proof}

\subsubsection{Proof of Lemma~\ref{lem:hamming_sphere_estimates}}

\begin{proof}
  Consider a binomial random variable with $d$ trials and success probability $v=r/d$. Its most likely value is $r$, so the probability at $r$ lies between $1/(d+1)$ and $1$. Since this probability is $\binom dr e^{-dh_2(v)}$, Eq.~\eqref{eq:uniform_sphere_size} follows.

  For a different success probability $p$, multiplying the lower bound in Eq.~\eqref{eq:uniform_sphere_size} by $p^r(1-p)^{d-r}$ gives the first inequality in Eq.~\eqref{eq:uniform_binomial_point_mass}. The second follows from Eq.~\eqref{eq:uniform_bernoulli_kl}, which holds for all $v,p\in(0,1)$ by applying $\log u\le u-1$ to the two terms of the KL divergence:
  \begin{align}
    \KL{\Ber(v)}{\Ber(p)}&\le\frac{v^2}{p}+\frac{(1-v)^2}{1-p}-1=\frac{(v-p)^2}{p(1-p)}\,.\qedhere
  \end{align}
\end{proof}

\subsubsection{Proof of Lemma~\ref{lem:uniform_density_lower_bound}}

\begin{proof}
  Choose $r=\left\lceil d\max\{\beta_t,\beta_{\tcritu}\}+D\Lambda\right\rceil$, where $D$ is a sufficiently large constant; decreasing $c_{A,K,L}$ if necessary, $1\le r\le d-1$. Write $v=r/d$. Let $N_r$ be the number of elements of $\calC$ at Hamming distance exactly $r$ from $x$. Its expectation is
  \begin{equation}
    \mu_r=\EE_\calC N_r=M2^{-d}\binom dr\,.
  \end{equation}
  By Eq.~\eqref{eq:uniform_sphere_size} of Lemma~\ref{lem:hamming_sphere_estimates} and $\log M=d\,\uMI(\tcritu)=d(\log2-h_2(\beta_{\tcritu}))$ (Eq.~\eqref{eq:critical_times}),
  \begin{equation}
  \log\mu_r\ge d\bigl(h_2(v)-h_2(\beta_{\tcritu})\bigr)-\log(d+1)\ge cD\Lambda-\log(d+1)\,.
  \label{eq:uniform_sphere_mean}
  \end{equation}
  The last step uses that $v-\beta_{\tcritu}\ge D\Lambda/d$ and that $h'_2$ is lower bounded away from zero in the relevant interval. Since $\Lambda\ge\log d$, choosing $D$ large enough gives $\mu_r\ge4e^{K\Lambda}$.

  Define the Hamming sphere $\mathcal S_r(x)=\{z\in\calX\mid d_H(x,z)=r\}$. For $z\in\mathcal S_r(x)$, denote the indicator that $z$ is included in $\calC$ by 
  $\mathbf{I}_z=\bone{z\in \calC}$. Then $N_r=\sum_{z\in\mathcal S_r(x)}\mathbf{I}_z$. For distinct $z,z'\in\mathcal S_r(x)$, the events $\{z\in \calC\}$ and $\{z' \in \calC\}$ are negatively correlated, so
  $\Var[\calC]{N_r}\le\sum_{z\in\mathcal S_r(x)}\Var[\calC]{\mathbf{I}_z}\le
  \mu_r$. By Chebyshev's inequality,
  \begin{equation}
  \Pr[\calC]{N_r<\mu_r/2}\le4/\mu_r\le e^{-K\Lambda}\,.
  \label{eq:uniform_sphere_occupancy}
  \end{equation}

  Henceforth condition on the event $N_r\ge\mu_r/2$. Each point $z\in\mathcal S_r(x)\cap \calC$ has the same transition probability $(K_t^{\sf U})^{\otimes d}(x \mid z)$. Their contribution alone yields
  \begin{align}
  L_t(x)=\frac{q^\calC_t(x)}{2^{-d}}
  &\ge\frac{N_r}{M}2^d \beta_t^r(1-\beta_t)^{d-r}\\
  &\ge\frac12\binom dr \beta_t^r(1-\beta_t)^{d-r}
  =\frac12\Pr{\Bin(d,\beta_t)=r}\,.
  \label{eq:uniform_sphere_contribution}
  \end{align}
  The choice of $r$ gives $0\le v-\beta_t\le|\beta_t-\beta_{\tcritu}|+D\Lambda/d+1/d$. Since $|t-\tcritu|\le Lw$ and $\beta'_t$ is bounded, $|\beta_t-\beta_{\tcritu}|\le Cw$. Thus
  \begin{equation}
  d(v-\beta_t)^2\le C\left(dw^2+\frac{D^2\Lambda^2}{d}+\frac1d\right)\le C'\Lambda\,.
  \end{equation}
  Eq.~\eqref{eq:uniform_binomial_point_mass} now lower bounds the right-hand side of Eq.~\eqref{eq:uniform_sphere_contribution} by $e^{-C\Lambda}$, the factor $\frac{1}{2(d+1)}$ being absorbed using $\Lambda\ge\log d$. Together with Eq.~\eqref{eq:uniform_sphere_occupancy}, this proves the lemma.
\end{proof}

\subsubsection{Proof of Proposition~\ref{prop:uniform_window_search}}

\begin{proof}
  Recall Algorithm~\ref{alg:uniform_window_search}: the grid $\mathcal T$ covers $[\tau_--3w,\tau_++3w]$ with mesh $w/20$, each grid point $t$ is tested with $R=2\lceil C\Lambda\rceil+1$ independent samples $X^{(j)}\sim\Unif(\calX)$, and the search stops at the first grid point at which the median of $\widehat T_t(X^{(j)})$ over these samples falls below $-(a+b)dw/2$, where $a<b$ are the constants of Proposition~\ref{prop:uniform_score_separation}.

  The search only needs two kinds of decisions to be reliable: it must continue at every grid point above $\tcritu+2w$, and it must stop by the time it has passed through $[\tcritu-3w,\tcritu-2w]$. Proposition~\ref{prop:uniform_score_separation} controls a single query in exactly these two regions.

  For a relevant grid point $t$ ($t>\tcritu+2w$ or $\tcritu-3w<t<\tcritu-2w$), let $b_{\calC,t}$ be the probability, conditional on $\calC$, that one trial lies on the wrong side of the threshold $-(a+b)dw/2$. Proposition~\ref{prop:uniform_score_separation} with exponent $12$ gives $\EE_\calC b_{\calC,t}\le e^{-12\Lambda}$ and hence $\Pr[\calC]{b_{\calC,t}>1/16}\le16e^{-12\Lambda}$. The grid has $O(1/w)\lesssim e^\Lambda$ points for sufficiently large $d$. A union bound shows that, outside an event over $\calC$ of probability at most $16e^{-11\Lambda}\le\delta/4$, every relevant trial has conditional error probability at most $1/16$.

  Fix such an $\calC$. For each grid point the $R$ trials are independent. The probability that at least half are wrong is at most $2^R(1/16)^{R/2}\le2^{-R}$. Taking a union bound over the grid and choosing the constant $C$ in the expression of $R$ sufficiently large makes the probability that there is a grid point where at least half of the $R$ trials are wrong at most $\epsilon/8$. We may draw the randomness for all grid points in advance, so that stopping early does not change this simultaneous guarantee. On the event that more than half of the $R$ trials are correct at all relevant grid points, the search never stops above $\tcritu+2w$, and the grid mesh ensures that it tests a point in $[\tcritu-3w,\tcritu-2w]$, where it must stop, if it has not already done so. This proves Eq.~\eqref{eq:uniform_window_search_success}. There are $O(1/w)$ grid points and $O(\Lambda)$ queries at each, resulting in the query bound.
\end{proof}

\subsubsection{Proof of Lemma~\ref{lem:uniform_reverse_step}}

\begin{proof}
  Write $u=e^{-s}$ and $v=e^{-h}$, so that $e^{-t}=uv$. Substitution of the two clipping endpoints into Eq.~\eqref{eq:uniform_reverse_flip_map} gives
  \begin{equation}
    p_h(\ell_t)=\frac{(1-v)(1-u)}{2(1+uv)}\,,\qquad
    p_h(\ell_t^{-1})=\frac{(1-v)(1+u)}{2(1-uv)}\,.
  \end{equation}
  Because $s,t\in I_0$, the factors other than $1-v$ are bounded above and below by positive constants. Also $1-v\asymp h$ for $h\le h_0$. Since $p_h$ is increasing, every true or clipped approximate score gives a probability in $[ch,Ch]$. Choose $h_0$ small enough that the upper bound is at most $1/2$.

  Let $r$ and $\widetilde r$ be the true and clipped approximate score entries for one coordinate, and write $p=p_h(r)$ and $\widehat p=p_h(\widetilde r)$. The slope of $p_h$ is at most $Ch$, so $|p-\widehat p|\le Ch|r-\widetilde r|$. Eq.~\eqref{eq:uniform_bernoulli_kl} now gives
  \begin{equation}
    \KL{\Ber(p)}{\Ber(\widehat p)}\le\frac{(p-\widehat p)^2}{\widehat p(1-\widehat p)}\le Ch(r-\widetilde r)^2\,.
  \end{equation}
  Summing over the coordinates and applying Eq.~\eqref{eq:uniform_clipped_error} proves Eq.~\eqref{eq:uniform_reverse_step_kl}.
\end{proof}

\subsection{Proofs from Section~\ref{sec:gaussian-upper}}

Throughout this subsection we use the notation of Section~\ref{sec:gaussian-upper}: $\Lambda$ is as in Eq.~\eqref{eq:gaussian_sampler_parameter} and $w=A\sqrt{\Lambda/d}$ as in Proposition~\ref{prop:gaussian_critical_window}; $I_0$ is the interval of Eq.~\eqref{eq:known_ts}; $\nu_t$ is the reference law of Eq.~\eqref{eq:gaussian_reference_law} and $L_t=q^\calC_t/\nu_t$ the density ratio of Eq.~\eqref{eq:gaussian_density_ratio}; $\lambda_t=e^{-t}/\sigma_t$, $p_u$, and $\phi_d$ are as in Eqs.~\eqref{eq:gaussian_normalized_channel} and~\eqref{eq:gaussian_reference_density}; $c_I,C_I$ are the constants of Eq.~\eqref{eq:gaussian_information_derivative}; $A_z,B_z$ are the log averages of Eq.~\eqref{eq:gaussian_log_averages}, and $\pi_r(\cdot\mid z)$ and $\mathcal I_z$ are the cube posterior and posterior information of Eqs.~\eqref{eq:gaussian_cube_posterior} and~\eqref{eq:gaussian_posterior_information}; $m_t$, $m^0_t$, and $T_t$ are as in Eqs.~\eqref{eq:gaussian_reference_mean} and~\eqref{eq:gaussian_test_quantity}, and $\widetilde m_t$, $\widetilde s_t$, $\widehat T_t$ are their clipped versions from Eqs.~\eqref{eq:gaussian_clipped_mean} and~\eqref{eq:gaussian_clipped_score_and_statistic}; finally $\Sigma_t$ and $V$ are the posterior covariance and its expected trace from Eq.~\eqref{eq:gaussian_posterior_variance}.

\subsubsection{Proof of Lemma~\ref{lem:gaussian_reference_comparison}}

\begin{proof}
  Rather than comparing the mixture laws $p_u$ and $p_v$ directly, we lift them to the joint laws of the latent variable and observation $(U,Z)$. Let $P_r$ denote the joint law
  \begin{equation}
    U\sim\Unif(\calX),\qquad Z=rU+G\,,
  \end{equation}
  for $r>0$. Conditional on $U=y$, the laws of $Z$ under $P_u$ and $P_v$ are Gaussians with means $uy$ and $vy$ and common covariance $I_d$. Hence
  \begin{equation}
    \frac{\d P_u}{\d P_v}(U,Z)=\frac{p_u(Z\mid U)}{p_v(Z\mid U)}=\frac{\phi_d(Z-uU)}{\phi_d(Z-vU)}\,.
  \end{equation}
  Since $\|y\|_2^2=d$ for every $y\in\calX$, the squared distance between these two means is $d(u-v)^2$. Therefore, the second moment of the density ratio under $P_v$ is
  \begin{equation}
    \EE_{P_v}\left[\left(\frac{\d P_u}{\d P_v}\right)^2\right]=\EE_{Y\sim\Unif(\calX)}\int\frac{\phi_d(z-uY)^2}{\phi_d(z-vY)}\,\d z=e^{d(u-v)^2}\,.
  \end{equation}
  Thus, by change of measure and Cauchy--Schwarz,
  \begin{equation}
    p_u(A)=\EE_{P_v}\left[\frac{\d P_u}{\d P_v}(U,Z)\bone{Z\in A}\right]\le\left(\EE_{P_v}\left[\left(\frac{\d P_u}{\d P_v}\right)^2\right]\right)^{1/2}P_v(Z\in A)^{1/2}=e^{d(u-v)^2/2}\sqrt{p_v(A)}\,,
  \end{equation}
  proving Eq.~\eqref{eq:gaussian_reference_comparison}.

  Since $\calC$ has the same independent distribution under both laws, adjoining it to the two joint laws does not change the likelihood ratio or its second moment. The same argument therefore proves Eq.~\eqref{eq:gaussian_reference_comparison_random_subset}.
\end{proof}

\subsubsection{Proof of Lemma~\ref{lem:gaussian_reference_observation}}

\begin{proof}
  The coordinates of $Z$ are independent mixtures of Gaussians with bounded means. Thus Eq.~\eqref{eq:gaussian_reference_norm} holds except with probability $e^{-cd}$. Since
  \begin{equation}
    B''_Z(r)=\sum_{i=1}^dZ_i^2\sech^2(rZ_i)\,,
  \end{equation}
  this also gives the upper bound for every $r$. For the lower bound, count the coordinates for which $1\le|Z_i|\le2$. Each coordinate has a probability bounded away from zero of satisfying this condition, uniformly in $t$. A Chernoff bound shows that at least $cd$ coordinates satisfy it except with probability $e^{-cd}$. Each such coordinate contributes a fixed positive amount to $B''_Z(r)$ throughout the prescribed interval.

  By Eq.~\eqref{eq:gaussian_posterior_information}, $\mathcal I_Z(\lambda_t)$ is a sum of $d$ independent one-coordinate $\lambda_tZ_i\tanh(\lambda_tZ_i)-\log\cosh(\lambda_tZ_i)$ terms. Each lies in $[0,\log2]$, and its mean is $\gMI(t)$. Hoeffding's inequality gives Eq.~\eqref{eq:gaussian_posterior_information_concentration}. Choose $C_K$ so that its failure probability is at most $e^{-(K+2)\Lambda}$, and then choose $c_K$ small enough to absorb the two probabilities $e^{-cd}$.
\end{proof}

\subsubsection{Proof of Lemma~\ref{lem:gaussian_mean_error}}

\begin{proof}
  Every coordinate of $m_t(x)$ lies in $[-1,1]$. Tweedie's formula and the fact that clipping cannot increase Euclidean distance to this cube therefore give
  \begin{equation}
    \|\widetilde m_t(x)-m_t(x)\|_2\le e^t\sigma_t^2\|\widehat s_t(x)-s_t(x)\|_2\,.
  \end{equation}
  Multiply this inequality by $e^{-t}\sigma_t^{-2}$ to obtain the corresponding bound for $\widetilde s_t-s_t$. Squaring, averaging under $q^\calC_t$, and using the definition of $\epsilon_{\sf gauss}$ proves Eq.~\eqref{eq:gaussian_clipping_error}. Eq.~\eqref{eq:gaussian_overlap_error} follows by Cauchy--Schwarz.
\end{proof}

\subsubsection{Proof of Lemma~\ref{lem:gaussian_density_near_transition}}

\begin{proof}
  For $u=\lambda_t$ and $z=x/\sigma_t$, Eq.~\eqref{eq:gaussian_overlap_derivative} writes the density ratio as
  \begin{equation}
    L_t(x)=\frac{q^\calC_t(x)}{\nu_t(x)}
    =\frac1M\sum_{y\in \calC}e^{u\langle z,y\rangle-B_z(u)}\,.
    \label{eq:gaussian_density_empirical_average}
  \end{equation}
  Define the summand
  \begin{equation}
    L_{u,z}(y)=e^{u\langle z,y\rangle-B_z(u)}
    =\frac{\phi_d(z-uy)}{p_u(z)}
    =\frac{\pi_u(y\mid z)}{2^{-d}}\,.
    \label{eq:gaussian_density_summand}
  \end{equation}
  Thus $L_{u,z}(y)$ is the posterior-to-prior ratio of $y$ under the whole-cube prior, and
  \begin{equation}
    \EE_{Y\sim\Unif(\calX)}L_{u,z}(Y)=1\,.
  \end{equation}
  Eq.~\eqref{eq:gaussian_density_empirical_average} is therefore the empirical average of these likelihood ratios over the $M$ points of $\calC$, sampled without replacement from the cube. We will show that enough moderate likelihood ratios remain in this average to prevent it from being too small.

  \par\vspace{\baselineskip}\noindent Fix $t$ in the stated range, and write $u=\lambda_t$ and $Z=X/\sigma_t$. Since $|t-\tcritg|\le Lw$ and $\tcritg\in[\tau_-,\tau_+]$, we have $t\in I_0$ as soon as $Lw\le\tau_-/2$, which holds after decreasing $c_{A,L,K}$; we may therefore apply Lemma~\ref{lem:gaussian_reference_observation} with exponent $K+3$. Since $|t-\tcritg|\le Lw$, Eqs.~\eqref{eq:gaussian_information_derivative} and~\eqref{eq:gaussian_posterior_information_concentration} imply
  \begin{equation}
    \mathcal I_Z(u) \le d\,\gMI(t)+C_{K+3}\sqrt{d\Lambda} \le \log M+(C_ILA+C_{K+3})\sqrt{d\Lambda}=\log M+C_0\sqrt{d\Lambda}\,,
    \label{eq:gaussian_information_near_window}
  \end{equation}
  where $C_0$ depends only on $A,L,K$ and $[\kappa_-,\kappa_+]$. Fix an observation $z$ for which this bound and the uniform bounds $cd\le B''_z(r)\le Cd$ hold. Until the last step, all probabilities are conditional on this $z$.

  \par\vspace{\baselineskip}\noindent Now we migrate to a slightly larger noise level. Set
  \begin{equation}
    v=u-D_0\sqrt{\Lambda/d}\,,
  \end{equation}
  where $D_0$ will be chosen sufficiently large. Since $\mathcal I'_z(r)=rB''_z(r)\ge cd$ throughout the relevant parameter interval, decreasing the signal strength from $u$ to $v$ decreases the posterior information $\mathcal I_z$ by at least $cD_0\sqrt{d\Lambda}$. Thus, choosing $D_0$ sufficiently large compared with $C_0$ gives
  \begin{equation}
    \mathcal I_z(v)\le\log M-4\sqrt{d\Lambda}\,.
    \label{eq:gaussian_information_at_noisier_parameter}
  \end{equation}
  Decreasing $c_{A,L,K}$ if necessary ensures that $v$ remains in the same fixed parameter interval.

  Draw $Y\sim\pi_v(\cdot\mid z)$. The logarithm of the summand $L_{u,z}(Y)$ has mean
  \begin{equation}
    \EE_{Y\sim\pi_v(\cdot\mid z)}\log L_{u,z}(Y)
    =\EE_{Y\sim\pi_v(\cdot\mid z)}\left[u\langle z,Y\rangle-B_z(u)\right]
    =\mathcal I_z(v)-\KL{\pi_v(\cdot\mid z)}{\pi_u(\cdot\mid z)}
    \le\mathcal I_z(v)\,,
    \label{eq:gaussian_log_summand_mean}
  \end{equation}
  and variance
  \begin{equation}
    \Var[\pi_v]{\log L_{u,z}(Y)}
    =u^2B''_z(v)\le Cd\,.
  \end{equation}
  Thus, under the slightly noisier posterior, the log likelihood ratio is typically well below $\log M$.

  Fix a constant $D>0$, to be chosen later, and define
  \begin{equation}
    \mathcal G
    =\left\{y\in\calX:L_{u,z}(y)\le Me^{-D\Lambda}\right\}\,.
    \label{eq:gaussian_retained_points}
  \end{equation}
  Take $c_{A,L,K}$ sufficiently small that $D\Lambda\le\sqrt{d\Lambda}$. By Eqs.~\eqref{eq:gaussian_information_at_noisier_parameter} and~\eqref{eq:gaussian_log_summand_mean}, the cutoff $\log M-D\Lambda$ lies at least $3\sqrt{d\Lambda}$ above the mean of $\log L_{u,z}(Y)$. Chebyshev's inequality therefore gives
  \begin{equation}
    \pi_v(\mathcal G\mid z)\ge1-\frac{C}{\Lambda}\ge\frac12
    \label{eq:gaussian_retained_mass_at_noisier_parameter}
  \end{equation}
  for sufficiently large $d$.

  \par\vspace{\baselineskip}\noindent We next show that $\mathcal G$ still has nonnegligible posterior mass at the original signal strength $u$. Taylor's theorem and the bound $B''_z(r)\le Cd$ imply
  \begin{equation}
    \KL{\pi_v(\cdot\mid z)}{\pi_u(\cdot\mid z)}
    =B_z(u)-B_z(v)-(u-v)B'_z(v)
    \le C d(u-v)^2
    \le CD_0^2\Lambda\,.
    \label{eq:gaussian_posterior_comparison}
  \end{equation}
  On the other hand, the data processing inequality for KL divergence gives
  \begin{align}
    \KL{\pi_v(\cdot\mid z)}{\pi_u(\cdot\mid z)}
    &\ge\KL{\Ber(\pi_v(\mathcal G\mid z))}{\Ber(\pi_u(\mathcal G\mid z))} \\
    &\ge\pi_v(\mathcal G\mid z)\log\frac{1}{\pi_u(\mathcal G\mid z)}-\log2 \\
    &\ge\frac12\log\frac{1}{\pi_u(\mathcal G\mid z)}-\log2\,.
  \end{align}
  Combining this with Eq.~\eqref{eq:gaussian_posterior_comparison} gives
  \begin{equation}
    \pi_u(\mathcal G\mid z)\ge e^{-C_1\Lambda}\,,
    \label{eq:gaussian_retained_mass}
  \end{equation}
  where $C_1$ depends on $D_0$ but not on $D$.

  \par\vspace{\baselineskip}\noindent Now, consider only the points of $\calC$ belonging to $\mathcal G$, and define
  \begin{equation}
    \widetilde L
    =\frac1M\sum_{y\in \calC\cap\mathcal G}L_{u,z}(y)\,.
  \end{equation}
  Its expectation over $\calC$ is
  \begin{equation}
    \EE_\calC\widetilde L
    =\EE_{Y\sim\Unif(\calX)}
    \left[L_{u,z}(Y)\bone{Y\in\mathcal G}\right]
    =\pi_u(\mathcal G\mid z)\,.
  \end{equation}
  Moreover, $L_{u,z}(y)\le Me^{-D\Lambda}$ on $\mathcal G$. The variance bound for sampling without replacement therefore gives
  \begin{align}
    \Var[\calC]{\widetilde L}
    &\le\frac1M\Var[Y\sim\Unif(\calX)]{L_{u,z}(Y)\bone{Y\in\mathcal G}} \\
    &\le\frac1M\EE_{Y\sim\Unif(\calX)}\left[L_{u,z}(Y)^2\bone{Y\in\mathcal G}\right] \\
    &\le e^{-D\Lambda}\EE_{Y\sim\Unif(\calX)}\left[L_{u,z}(Y)\bone{Y\in\mathcal G}\right] \\
    &\le e^{-D\Lambda}\pi_u(\mathcal G\mid z)\,.
  \end{align}
  Hence, by Chebyshev's inequality and Eq.~\eqref{eq:gaussian_retained_mass},
  \begin{equation}
    \Pr[\calC]{\widetilde L<\pi_u(\mathcal G\mid z)/2}
    \le\frac{4e^{-D\Lambda}}{\pi_u(\mathcal G\mid z)}
    \le4e^{-(D-C_1)\Lambda}\,.
  \end{equation}
  Choose $D>C_1+K+4$. Then, except with probability at most $e^{-(K+2)\Lambda}$ over $\calC$,
  \begin{equation}
    L_t(x)\ge\widetilde L\ge \pi_u(\mathcal G\mid z)/2\ge e^{-(C_1+1)\Lambda}\,.
  \end{equation}
  Finally, add the probability that the conditions imposed on $Z$ fail and enlarge the constant in the statement. This proves Eq.~\eqref{eq:gaussian_density_near_transition}.
\end{proof}

\subsubsection{Proof of Proposition~\ref{prop:gaussian_window_search}}

\begin{proof}
  Recall Algorithm~\ref{alg:gaussian_window_search}: the grid $\mathcal T$ covers $[\tau_--3w,\tau_++3w]$ with mesh $w/20$, each grid point $t$ is tested with $R=2\lceil C\Lambda\rceil+1$ independent samples $X^{(j)}\sim\nu_t$, and the search stops at the first grid point at which the median of $\widehat T_t(X^{(j)})$ over these samples falls below $-(a+b)dw/2$, where $a<b$ are the constants of Proposition~\ref{prop:gaussian_score_separation}.

  We first control the choice of $\calC$, and then the randomness of the search conditional on $\calC$.

  At a fixed grid point above $\tcritg+2w$, a wrong decision for one sample means $\widehat T_t<-(a+b)dw/2$. In $[\tcritg-3w,\tcritg-2w]$, it means the reverse inequality. Proposition~\ref{prop:gaussian_score_separation}, with exponent $12$, bounds each of these probabilities by $e^{-12\Lambda}$ after averaging over $\calC$.

  At any one relevant grid point, Markov's inequality shows that the conditional probability of a wrong decision exceeds $1/16$ for at most a $16e^{-12\Lambda}$ fraction of $\calC$. There are $O(1/w)\le e^\Lambda$ grid points. A union bound therefore removes at most a $16e^{-11\Lambda}\le\delta/4$ fraction of $\calC$, and for every remaining $\calC$ the conditional error is at most $1/16$ at every relevant point.

  Fix $\calC$ satisfying the preceding simultaneous guarantee. At one grid point the $R$ trials are independent. The probability that at least half are wrong is at most $2^R(1/16)^{R/2}\le2^{-R}$. Taking $R=2\lceil C\Lambda\rceil+1$ with a sufficiently large constant $C$ makes the union of these events over the grid have probability at most $\epsilon/8$.

  Now condition on all relevant medians being correct. The search cannot stop above $\tcritg+2w$. The grid contains a point in $[\tcritg-3w,\tcritg-2w]$, at which it must stop if it has not already done so. This gives Eq.~\eqref{eq:gaussian_window_search_success}. We may draw all samples in advance, so early stopping does not alter this argument. Finally, there are $O(1/w)$ grid points and $O(\Lambda)$ score oracle queries at each, giving the query complexity bound.
\end{proof}

\subsubsection{Proof of Lemma~\ref{lem:gaussian_posterior_variance}}

\begin{proof}
  Let $Y\sim q^\calC$ and observe
  \begin{equation}
    R_r=rY+B_r,\qquad r\ge0\,,
  \end{equation}
  where $B$ is an independent Brownian motion. For a fixed $y$, the likelihood of the trajectory $(R_v)_{0\le v\le r}$ relative to standard Brownian motion is $\exp\left(\langle y,R_r\rangle-\frac r2\|y\|_2^2\right)$. Hence the conditional probability of $Y=y$ given the observations up to time $r$ is proportional to $q^\calC(y)\exp\left(\langle y,R_r\rangle-\frac r2\|y\|_2^2\right)$. In particular, this conditional distribution depends on the observations only through $R_r$. Define
  \begin{equation}
    m_r\triangleq\EE[Y\mid R_r],\qquad \Gamma_r\triangleq\Cov(Y\mid R_r)\,.
  \end{equation}
  
  Notice that $R_r/\sqrt r$ has the same law as $\sqrt rY+G$, where $G\sim N(0,I_d)$ independently. Hence the posterior at time $r$ is the same as the Gaussian-diffusion posterior at time $t$ when $r=\lambda_t^2=e^{-2t}/\sigma_t^2$.

  Subtract the conditional drift from the observation and set 
  \begin{equation}
    W_r=R_r-\int_0^r m_v\,\d v\,.
  \end{equation}
  In the observation filtration $\mathcal{F}_r=\sigma(R_v:0\le v\le r)$, this is a continuous martingale with quadratic variation $rI_d$, hence a Brownian motion. Applying It\^o's formula to $\Pr{Y=y\mid R_r}$, we have
  \begin{equation}
    \d\Pr{Y=y\mid R_r}
    =\Pr{Y=y\mid R_r}\langle y-m_r,\d W_r\rangle\,.
  \end{equation}

  By definition of $m_r$ and $\Gamma_r$,
  \begin{equation}
    \d m_r
    =\sum_{y\in \calC}y\,\d\Pr{Y=y\mid R_r}
    =\left(\sum_{y\in \calC}\Pr{Y=y\mid R_r}y(y-m_r)^\top\right)\d W_r
    =\Gamma_r\,\d W_r\,.
    \label{eq:gaussian_posterior_mean_martingale}
  \end{equation}

  Since $\|Y\|_2^2=d$,
  \begin{equation}
    \EE\Tr \Gamma_r=d-\EE\|m_r\|_2^2\,.
  \end{equation}
  On the other hand, applying It\^o's formula to $\|m_r\|_2^2$ with Eq.~\eqref{eq:gaussian_posterior_mean_martingale} gives
  \begin{equation}
    \d\|m_r\|_2^2
    =2\langle m_r,\Gamma_r\,\d W_r\rangle
    +\Tr(\Gamma_r^2)\,\d r\,.
  \end{equation}
  Taking expectations therefore yields
  \begin{equation}
    \frac{\D}{\D r}\EE\|m_r\|_2^2
    =\EE\Tr(\Gamma_r^2)\,.
  \end{equation}
  Combining the preceding two identities,
  \begin{equation}
    \frac{\D}{\D r}\EE\Tr \Gamma_r=-\EE\Tr(\Gamma_r^2)\,.
  \end{equation}

  Recall that the Gaussian diffusion posterior corresponds to the preceding observation model with $r=\lambda_t^2=e^{-2t}/\sigma_t^2$. Since
  \begin{equation}
    \frac{\D}{\D t}\lambda_t^2=-2e^{-2t}\sigma_t^{-4}\,,
  \end{equation}
  by the chain rule we get the derivative in Eq.~\eqref{eq:gaussian_posterior_variance_evolution}. Finally, $V(t)=d-\EE\|m_t(X_t)\|_2^2$ lies in $[0,d]$.
\end{proof}

\subsubsection{Proof of Lemma~\ref{lem:gaussian_score_change}}

\begin{proof}
  Fix $\calC$ throughout. There are two quantities to estimate: the norm of the score itself and the norm of its derivative. Both have direct posterior interpretations.
  
  From $X_t=e^{-t}X_0+\sigma_tG$, we have $\Cov(G\mid X_t) = \frac{e^{-2t}}{\sigma_t^2}\Sigma_t(X_t)$. From Tweedie's formula, we have $s_t(X_t)=-\sigma_t^{-1}\EE[G\mid X_t]$. Since $d=\EE\|G\|_2^2=\EE\|\EE[G\mid X_t]\|_2^2+\EE\Tr\Cov(G\mid X_t)$, we obtain
  \begin{equation}
    \EE_{q^\calC_t}\|s_t\|_2^2=d\sigma_t^{-2}-e^{-2t}\sigma_t^{-4}V(t)\le d\sigma_t^{-2}\le Cd\,.
    \label{eq:gaussian_score_second_moment}
  \end{equation}

  Differentiating the posterior mean gives $\nabla m_t(x)=e^{-t}\sigma_t^{-2}\Sigma_t(x)$, and hence, by Tweedie's formula,
  \begin{equation}
    \nabla s_t(x)
    =-\sigma_t^{-2}I_d
    +e^{-2t}\sigma_t^{-4}\Sigma_t(x)\,.
    \label{eq:gaussian_score_hessian}
  \end{equation}
  All coefficients are bounded on $I_0$. Squaring the two terms separately and using Lemma~\ref{lem:gaussian_posterior_variance} yields
  \begin{equation}
    \EE_{q^\calC_t}\|\nabla s_t\|_{\rm F}^2
    \le Cd+C\EE_{q^\calC_t}\Tr(\Sigma_t^2)
    \le C(d+V'(t))\,.
    \label{eq:gaussian_score_derivative_bound}
  \end{equation}
  This is where the increase of posterior variance enters the discretization error.

  The density $q^\calC_t$ solves $\partial_tq^\calC_t=\Delta q^\calC_t+\nabla\cdot(xq^\calC_t)$. Differentiating its logarithm gives
  \begin{equation}
    \partial_ts_t=\Delta s_t+2(\nabla s_t)s_t+s_t+(\nabla s_t)x\,.
  \end{equation}
  We now apply It\^o's formula to $s_{r-u}(Y_u)$, where $\d Y_u=(Y_u+2s_{r-u}(Y_u))\,\d u+\sqrt2\,\d B_u$. The terms containing $\Delta s_t$, $(\nabla s_t)s_t$, and $(\nabla s_t)x$ cancel, leaving
  \begin{equation}
    \d s_{r-u}(Y_u)=-s_{r-u}(Y_u)\,\d u+\sqrt2\,\nabla s_{r-u}(Y_u)\,\d B_u\,.
  \end{equation}
  Solving this linear equation gives
  \begin{equation}
    s_{r-u}(Y_u)-s_r(Y_0)
    =(e^{-u}-1)s_r(Y_0)+\sqrt2\int_0^u e^{-(u-v)}\nabla s_{r-v}(Y_v)\,\d B_v\,.
  \end{equation}
  The stochastic integral has conditional expectation zero given $Y_0$. Eqs.~\eqref{eq:gaussian_score_second_moment}--\eqref{eq:gaussian_score_derivative_bound} and It\^o's isometry therefore imply
  \begin{align}
    \EE\|s_{r-u}(Y_u)-s_r(Y_0)\|_2^2
    &\le Cdu^2+C\int_0^u(d+V'(r-v))\,\d v\notag\\
    &\le Cdu+C(V(r)-V(r-u)),\qquad 0\le u\le h\le1\,.
  \end{align}
  Integrating in $u$ and using the monotonicity of $V$ proves Eq.~\eqref{eq:gaussian_score_change}.
\end{proof}

\section{Details of simulation for Figure~\ref{fig:window-sim}}

Here we describe the numerical simulation that generated Figure~\ref{fig:window-sim}, which was largely inspired by an experiment carried out for Gaussian diffusion by Xun and Price~\cite{xun2026query}. 

For a fixed codebook $\mathcal{C}$, we consider a planted codeword $y^*\in\mathcal{C}$. Starting from $X_0 = y^*$, we apply the forward process associated to one of the three dLLM paradigms to obtain noisy sample $X_t$. The posterior distribution given by conditioning on $X_t$ places some mass on $y^*$, i.e.,
\begin{equation}
    \Pr{X_0 = y^* \mid X_t}\,,
\end{equation}
and panels (a)-(c) plot empirical estimates for the expectation of this quantity, which we call the \emph{recovery probability}, over the randomness of the forward process and the codebook. In the simulation, we then define the width of the critical window as the distance between the noise levels at which this recovery probability goes from $0.2$ to $0.8$.

To ensure that the noise levels are normalized the same way, we parametrize by the mutually information between a uniformly random bit and its corrupted version. For uniform and Gaussian diffusion, this information is given by $\mathsf{I}_{\sf unif}(t)$ and $\mathsf{I}_{\sf gauss}(t)$, and for masked diffusion with $m$ revealed coordinates, it is $\mathsf{I} = (m/d)\log 2$. The theory then predicts that the critical noise level occurs at $\mathsf{I} = \kappa$, where recall that $\kappa = \frac{1}{d}\log M$.

Because the size of the codebook is exponential in the dimension, it is not possible to directly simulate score access to the random empirical measure. Instead, we use the same Poissonization trick used in~\cite{xun2026query}: after planting an arbitrary bitstring $y^*$, we take the remaining ``spurious codewords'' in the random codebook to be sampled from a Poisson point process on the hypercube with intensity $M2^{-d}$ per point. We elaborate on how to perform an empirical estimate of the recovery probability under this approximation for each of the three frameworks.

\paragraph{Masked diffusion.} In this case, for a fixed number of revealed coordinates $m$, the expected recovery probability is simply the expectation of $\frac{1}{N}$, where $N$ is the number of codewords (including $y^*$) which agree with $y^*$ on the revealed subset of coordinates. Under Poissonization, $N-1$ is a draw from $\mathrm{Poi}(M2^{-m})$. Note that the value of $m$ at which the expectation of $\frac{1}{N}$ crosses from $0.2$ to $0.8$ is clearly independent of $d$ in this setup, so under the normalization by $\mathsf{I} = (m/d)\log 2$, the width of the window in the simulation is $\Theta(1/d)$ as the theory predicts.

\paragraph{Uniform diffusion.} In this case, by Lemma~\ref{lem:uniform-likelihood-ratio},
\begin{equation}
    \Pr{X_0 = y^* \mid X_t} = \frac{1}{1 + \sum^d_{j=0} N_j \tanh(t/2)^{j- d_H(X_t,y^*)}}\,,
\end{equation}
where $N_j$ is the number of spurious codewords at Hamming distance $j$ from $X_t$. Note that $d_H(X_t,y^*)$ is distributed as $\mathrm{Bin}(d,\beta_t)$, whereas conditioned on $X_t$, each quantity $N_j$ is independently distributed as $\mathrm{Poi}(M\binom{d}{j}2^{-d})$. This enables one to empirically estimate the recovery probability, where each sample in the empirical estimate only requires $d$ Poisson samples and one binomial sample.

\paragraph{Gaussian diffusion.} In this case,
\begin{equation}
    \Pr{X_0 = y^* \mid X_t} = \frac{1}{1 + Z} \qquad \text{for} \qquad Z = \sum_{y\in\calC\backslash\{y^*\}} \frac{K_t^{\sf G}(X_t\mid y)}{K_t^{\sf G}(X_t\mid y^*)}\,.
\end{equation}
For spurious codeword $y$ which differs from planted $y^*$ on coordinates $T\subseteq[d]$,
\begin{equation}
    \frac{K^{\sf G}_t(X_t\mid y)}{K_t^{\sf G}(X_t\mid y^*)} = \exp\Bigl(-\frac{2e^{-t}}{1 - e^{-2t}}\sum_{i\in T} a_i\Bigr) \qquad a_i = y^*_i X_{t,i}\,.
\end{equation}
Note that each $a_i$ is independently distributed according to $\mathcal{N}(e^{-t}, 1- e^{-2t})$, and conditional on these draws, $T$ is a uniformly random subset of $[d]$. Under Poissonization, the sum $\sum_{i\in T} a_i$ is thus a Poisson process with intensity $M p_a(s)\,\mathrm{d}s$, where $p_a$ is the density of the random variable $\sum^d_{i=1} b_i a_i$ for $b\sim \mathrm{Unif}(\{0,1\}^d)$. This density is estimated using the saddlepoint method for approximating the pdf of a sum of independent random variables, and the Poisson point process is approximated by discretizing the approximated density into bins.

\end{document}